%% file: main.tex
\documentclass[nohyperref]{article}

\usepackage{microtype}
\usepackage{graphicx}
\usepackage{subfigure}
\usepackage{booktabs} %

\usepackage{hyperref}

\usepackage[accepted]{icml2022}

\usepackage{amsmath}
\usepackage{amssymb}
\usepackage{mathtools}
\usepackage{amsthm}
\usepackage{centernot}

\usepackage[capitalize,noabbrev]{cleveref}

\theoremstyle{plain}
\newtheorem{theorem}{Theorem}[section]
\newtheorem{proposition}[theorem]{Proposition}
\newtheorem{lemma}[theorem]{Lemma}
\newtheorem{corollary}[theorem]{Corollary}
\theoremstyle{definition}
\newtheorem{definition}[theorem]{Definition}

\theoremstyle{remark}

\newcount\Comments  %
\newcommand{\kibitz}[2]{\ifnum\Comments=0\textcolor{#1}{#2}\fi}

\newcommand{\cready}[1]{#1}

\usepackage[textsize=tiny]{todonotes}

\input{math_commands.tex}

\usepackage{graphicx}
\usepackage{hyperref}
\usepackage{autonum}
\usepackage{url}
\usepackage{color}
\usepackage{amssymb}
\usepackage{amsmath}
\usepackage{soul}
\usepackage{enumerate}
\usepackage{enumitem}
\usepackage{xcolor}
\usepackage{bbm}
\usepackage{natbib}

\icmltitlerunning{Adversarially Robust Models may not Transfer Better: Domain Transferability from the View of Regularization}

\begin{document}

\twocolumn[
\icmltitle{Adversarially Robust Models may not Transfer Better: Sufficient Conditions for Domain Transferability from the View of Regularization}

\icmlsetsymbol{equal}{*}

\begin{icmlauthorlist}
\icmlauthor{Xiaojun Xu}{equal,uiuc}
\icmlauthor{Jacky Yibo Zhang}{equal,uiuc}
\icmlauthor{Evelyn Ma}{uiuc}
\icmlauthor{Danny Son}{uiuc}
\icmlauthor{Oluwasanmi Koyejo}{uiuc}
\icmlauthor{Bo Li}{uiuc}
\end{icmlauthorlist}

\icmlaffiliation{uiuc}{University of Illinois at Urbana-Champaign}

\icmlcorrespondingauthor{Xiaojun Xu}{xiaojun3@illinois.edu}
\icmlcorrespondingauthor{Jacky Yibo Zhang}{yiboz@illinois.edu}
\icmlcorrespondingauthor{Oluwasanmi Koyejo}{sanmi@illinois.edu}
\icmlcorrespondingauthor{Bo Li}{lbo@illinois.edu}

\icmlkeywords{Machine Learning, ICML}

\vskip 0.3in
]

\printAffiliationsAndNotice{\icmlEqualContribution} %

\begin{abstract}

Machine learning (ML) robustness and domain generalization are fundamentally correlated: they essentially concern data distribution shifts  under  adversarial and natural settings, respectively. On one hand, recent studies show that more robust (adversarially trained) models are more generalizable. On the other hand, there is a lack of theoretical understanding of their fundamental connections. In this paper, we explore the relationship between regularization and domain transferability considering different factors such as norm regularization and data augmentations (DA). We propose a general theoretical framework proving that factors involving the model function class regularization are sufficient conditions for \textit{relative} domain transferability. Our analysis implies that ``robustness" is  neither necessary nor sufficient for transferability; 
rather, regularization is a more fundamental perspective for understanding domain transferability. We then discuss popular DA protocols (including adversarial training) and show when they can be viewed as the function class regularization under certain conditions and therefore improve generalization. We conduct extensive experiments to verify our theoretical findings and show several counterexamples where robustness and generalization are negatively correlated on different datasets. 

\end{abstract}

\input{intro}

\input{analysis}

\input{regularizations}

\input{exp}

\section*{Acknowledgement}
This work is partially supported by
NSF 1910100, NSF 2046795, NSF 1909577, NSF 1934986, NSF CNS 2046726, NIFA award 2020-67021-32799, C3 AI, and the Alfred P. Sloan Foundation.

\clearpage
\bibliography{ref}
\bibliographystyle{icml2022}

\newpage
\appendix
\onecolumn
\input{appendix}

\end{document}

%% file: math_commands.tex
\usepackage{amsmath,amsfonts,bm, amsthm, mathrsfs, mathtools}

\numberwithin{theorem}{section}

\def\conv{{\mathrm{conv}}}
\def\Rad{{\mathrm{Rad}}}

\def\eqref#1{equation~\ref{#1}}
\def\Eqref#1{Equation~\ref{#1}}

\def\1{\bm{1}}

\def\eps{{\epsilon}}

\def\vv{{\bm{v}}}

\def\vx{{\bm{x}}}

\def\vz{{\bm{z}}}

\DeclareMathAlphabet{\mathsfit}{\encodingdefault}{\sfdefault}{m}{sl}
\SetMathAlphabet{\mathsfit}{bold}{\encodingdefault}{\sfdefault}{bx}{n}

\def\gD{{\mathcal{D}}}

\def\gF{{\mathcal{F}}}
\def\gG{{\mathcal{G}}}
\def\gH{{\mathcal{H}}}

\def\gL{{\mathcal{L}}}
\def\gM{{\mathcal{M}}}

\def\gP{{\mathcal{P}}}

\def\gS{{\mathcal{S}}}

\def\gV{{\mathcal{V}}}

\def\gX{{\mathcal{X}}}
\def\gY{{\mathcal{Y}}}
\def\gZ{{\mathcal{Z}}}

\def\sP{{\mathbb{P}}}

\def\sR{{\mathbb{R}}}

\newcommand{\E}{\mathbb{E}}

\newcommand{\R}{\mathbb{R}}

\newcommand{\aobj}[1][\lambda]{\mathrm{Obj}^A_{#1}}

\usepackage{dsfont}
\def\1{\mathds{1}}
\newcommand{\norm}[1]{\left\lVert#1\right\rVert}

\DeclareMathOperator*{\argmin}{arg\,min}

%% file: intro.tex
\section{Introduction}%
Domain generalization (or domain transferability) is the task of training machine learning models with data from one or more \textit{source} domains that can be adapted to a \textit{target} domain, often via low-cost fine-tuning. 
Thus, domain generalization refers to approaches designed to address the \textit{natural data distribution shift} problem~\citep{muandet2013domain,rosenfeld2021online}.
A wide array of approaches have been proposed to address domain transferability, including fine-tuning the last layer of DNNs~\citep{huang2018domain}, invariant feature optimization~\citep{muandet2013domain}, efficient model selection for fine-tuning~\citep{you2019towards}, and optimal transport based domain adaptation~\citep{courty2016optimal}.
Understanding domain generalization has emerged as an important task in the machine learning community.

\begin{figure*}[t]
    \centering
    \includegraphics[width=0.98\linewidth]{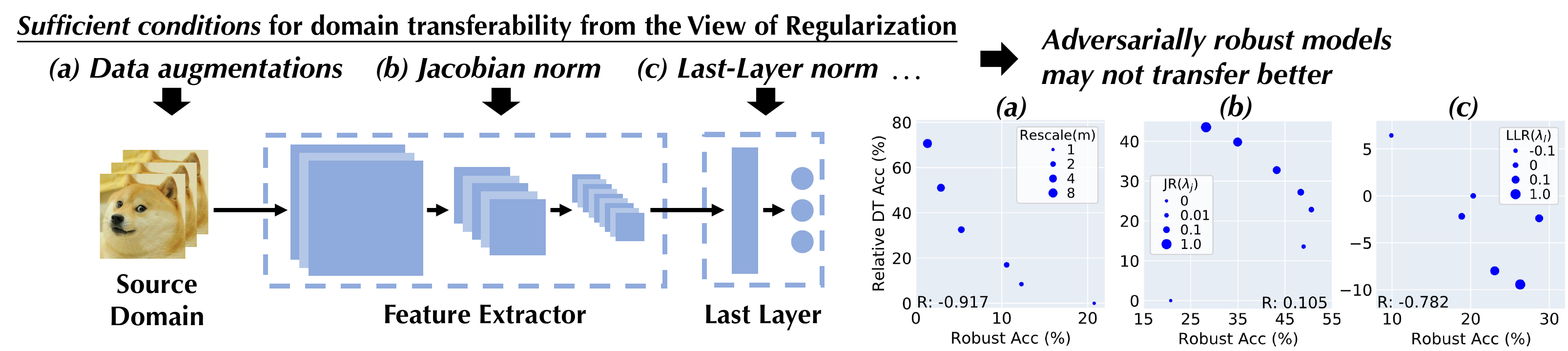}
    \caption{Illustration of robustness and domain transferability in different conditions. \cready{We study different augmentation and regularization techniques that can serve as sufficient conditions for domain transferability. We observe that adversarially robust models do not necessarily achieve a better performance in domain transferability and sometimes they are negatively correlated.} }
    \label{fig:counter-eg}
\end{figure*}

On the other hand, robust machine learning aims to tackle the problem of \textit{adversarial data distribution shift}. 
Both empirical and certified robust learning approaches have been proposed, such as empirical adversarial training~\citep{madry2018towards} and certified defenses based on both deterministic and probabilistic approaches~\citep{cohen2019certified,li2019robustra,li2021tss,li2020sok}.

Recent studies~\citep{salman2020adversarially, utrera2020adversarially} draw a connection between domain transferability and robustness, and suggest that adversarially robust models\cready{ (i.e., models with good accuracy under adversarial attacks)} are more domain transferable. However, a theoretical analysis of their fundamental connections is still lacking, and it is unclear whether robustness is necessary or sufficient. To fill in this gap, this paper aims to answer the following questions: \textit{Is model robustness sufficient or necessary for domain transferability? What are sufficient conditions for domain transferability? }

To answer the first question,  our analysis and experiments show that adversarial \textit{robustness is neither sufficient nor necessary} for domain transferability and they can even be negatively correlated. To answer the second question,  we first observe that domain transferability is  fundamentally a ``relative'' concept, as it by definition involves two domains, i.e., the source/target domain. With the observation, we propose a general theoretical framework that characterizes sufficient conditions for the \textit{relative} domain transferability from the view of function class regularization. The relative domain transferability, loosely speaking, is the performance of the fine-tuned source model on the target domain relative to the performance of the source model on the source domain. We then prove an inequality showing that stronger \textit{regularization} on the feature extractor (during the source model training process) implies a better relative domain transferability. %
We also discuss what data augmentations can be viewed as function class regularization generally. 
Since adversarial training can be viewed as a regularization under some conditions~\citep{roth2020reg, el1997robust, bertsimas2018characterization}, our work implies that the regularization effect of adversarial training is a better and more fundamental explanation for the connection between adversarial training and domain transferability.

To verify our theory, we conduct extensive experiments on ImageNet (CIFAR-10 as target domain) and CIFAR-10 (SVHN as target domain) based on different models. We show that regularizations such as norm regularization and  certain data augmentations can control the relative and absolute domain transferability, while the robustness and domain transferability can be even negatively correlated with the domain transferability, as illustrated in Fig.~\ref{fig:counter-eg}. 

{\textbf{Technical contributions.}} Our theoretical analysis and empirical findings show that, instead of robustness or adversarial training, regularization is a more fundamental perspective to understand domain transferability. Concretely,
\begin{itemize}[leftmargin=0.3cm, topsep=1pt,itemsep=1pt,partopsep=1pt, parsep=1pt]
    \item \cready{ We show that improving adversarial robustness is neither necessary nor sufficient for improving domain transferability without additional conditions, as shown in Section~\ref{subsec:example}. %
    \item We propose a theoretical framework to analyze the sufficient conditions for domain transferability from the view of function class regularization (Section~\ref{sec:up}\&\ref{subsec:generalized-ub}). We prove that shrinking the function class of feature extractors during training monotonically decreases a tight upper bound on the relative domain transferability loss. Therefore, {it is reasonable to expect} that imposing regularization on the feature extractor during training can lead to a better relative domain transferability. %
    \item We provide general analysis on when data augmentations (including adversarial training) can be viewed as regularization. In particular, we verify analysis based on the data augmentations of {Gaussian} noise, rotation, and translation, as discussed in Section~\ref{sec:DA}. 
    \item We conduct extensive experiments on different datasets and model architectures to verify our theoretical claims (Section~\ref{sec:exp}). We also show counterexamples where adversarial robustness is significantly negatively correlated with domain transferability. 
    }
    
\end{itemize}

\cready{Taken together, our results suggest a more nuanced explanation of the phenomenon that ``adversarially trained models transfer better,'' suggesting instead that adversarial training implies training with regularization, which, in turn, implies better transferability. 
As a consequence, although adversarial training implies better adversarial robustness, better adversarial robustness does not necessarily imply better transferability.
    }

\textbf{Related Work. }%
{Domain Transferability} has been analyzed in different settings. %
\citet{muandet2013domain} present a  generalization bound for classification tasks based on the properties of the assumed prior over training environments.
\citet{rosenfeld2021online} model domain transferability/generalization as an online game and show that generalizing beyond the  convex hull  {of training environments} is NP-hard.
Given the complexity of domain transferability analysis, recent empirical studies observe that adversarially trained models transfer better~\citep{salman2020adversarially, utrera2020adversarially}.

{Model robustness} is an important topic given recent diverse adversarial attacks~\citep{goodfellow2014explaining,carlini2017towards}. These attacks may be launched without access to model parameters~\citep{tu2019autozoom} or even with the model predictions alone~\citep{chen2020hopskipjumpattack}. Different approaches have been proposed to improve model robustness against adversarial attacks~\citep{yangli2021trs,ma2018characterizing,xiao2018characterizing}. Adversarial training has been shown to be effective empirically~\citep{madry2018towards,zhang2019theoretically,miyato2018virtual}.
Some studies have shown that robustness is related to other model characteristics, such as transferability and invertibility~\citep{engstrom2019adversarial,liang2020uncovering}.
A recent work \citep{deng2021adversarial} theoretically analyzes how adversarial training helps transfer learning. Although their proof implicitly depends on regularization, the authors only focus on adversarial training for linear models, while we directly focus on regularization for general models (e.g., DNNs).

%% file: analysis.tex
\begin{figure*}[h!]
    \centering
    \includegraphics[width=0.7\linewidth]{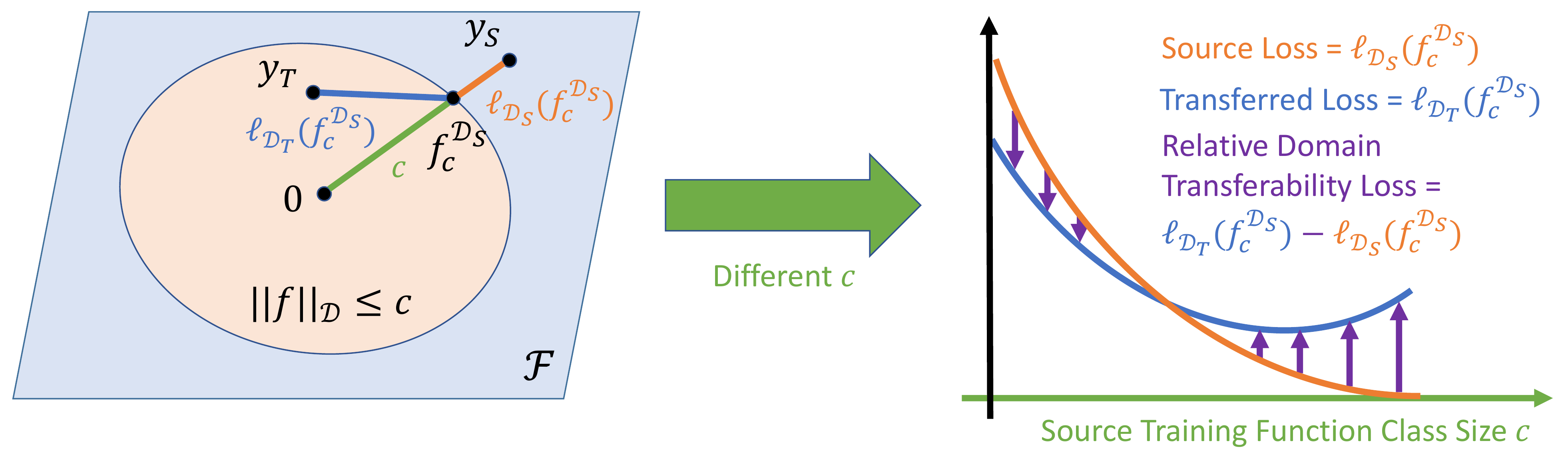}
    \caption{\small
    The left figure illustrates the example in the function space $\gF$ given a regularization parameter $c$. The right figure shows the relations between domain transferability and the $c$.
    In this example, the stronger the regularization effect (smaller $c$) is, the lower the relative domain transferability loss is (\textcolor{violet}{violet} arrow), and the better the relative domain transferability is.  }
    \label{fig:example-illustration}
\end{figure*}

\section{Sufficient Conditions for Domain Transferability}\label{sec:analysis} %

In this section, we theoretically analyze the problem of domain transferability from the view of regularization and discuss some sufficient conditions for  good  transferability. All of the proofs are provided in Section~\ref{sec:proofs} in the appendix.

\textbf{Notations.} We denote the input space as $\gX$; the feature space as $\gZ$ and the output space as $\gY$. Let the fine-tuning function class be $g\in \gG$. Given a feature extractor $f:\gX\to \gZ$ and a fine-tuning function $g:\gZ\to \gY$, the  full model is $g\circ f:\gX\to \gY$. We denote $\gP_{\gX\times \gY}$ as the set of distributions on $\gX\times \gY$. The loss function on $\gY$ is denoted by $\ell:\gY\times \gY \to \sR_+$. The population loss function based on data distribution $\gD\in \gP_{\gX\times \gY}$ and a model $g\circ f$ is defined as
\begin{align}
    \ell_{\gD}(g\circ f):= \E_{(x,y)\sim \gD} [\ell(g\circ f (x),y)].\label{eq:loss-def}
\end{align}
In the following, we first provide an example to show that the robustness can be irrelevant to domain transferability and to illustrate why one might investigate domain transferability from the view of regularization.

\subsection{A Toy Example: Motivation and Intuition} \label{subsec:example}

In this subsection, we construct a simple example where improving adversarial robustness is neither necessary nor sufficient for improving  (relative) domain transferability, yet stronger regularization sufficiently improves relative domain transferability. %
The settings introduced in this subsection are only applied in this subsection.

We consider the case that $\gX= \sR^m$ and $\gY=\sR^d$. Given an input $x\in \gX$, the ground truth target for the source domain is $y_S(x)$ generated by a function $y_S:\sR^m\to\sR^d$. Similarly, we define $y_T$ for the target domain. In this example, for simplicity, we neglect the fine-tuning process but directly consider learning a function $f:\sR^m\to \sR^d$ with a norm $\|\cdot\|$ on $\sR^d.$ We note that the analysis in this subsection holds with any choice of norm on $\sR^d$. 

Given the source and target distributions $\gD_S, \gD_T\in \sP_{\gX\times \gY}$, we consider the case that their marginal distributions on the input space $\gX$ are both $\gD$, while $y_S$ and $y_T$ could be different. 
Moreover, we consider the case that the support of the input data distribution $\gD$ lies on a low-dimensional manifold $\gM\subset \gX=\sR^m$ such that for $\forall x\in \gM$, any Euclidean ball centered at $x$ has non-empty intersection with $\sR^m\backslash \gM$. 
Given the distribution $\gD$, we define a norm for functions $f:\sR^m\to \sR^d$ as
$
    \|f\|_\gD:={\E_{x\sim \gD}[\|f(x)\|]},
$
where we view two functions $f_1,f_2$ as the same if $\|f_1-f_2\|_\gD=0$. 
Therefore, given a model $f$, for the source domain and the target domain we consider the respective loss functions as
\begin{align}
    \ell_{\gD_S}(f)&=\E_{x\sim \gD}[\|f(x)-y_S(x)\|]=\|f-y_S\|_{\gD},\label{eq:ex-1}\\
    \ell_{\gD_T}(f)&=\E_{x\sim \gD}[\|f(x)-y_T(x)\|]=\|f-y_T\|_{\gD}.
\end{align}

\cready{The toy example serves two purposes: (1) supporting the ``neither necessary nor sufficient'' claim; and (2) motivating the perspective of regularization. 
For the first purpose, the main intuition is that we can construct a setting where the domain transferability is only evaluated on a low-dimensional manifold while the adversarial robustness is only evaluated off the manifold. In such cases, a model having better adversarial robustness does not imply it has better domain transferability, and similarly a model having better domain transferability does not imply it has better  adversarial robustness. For the second purpose, as illustrated in Figure~\ref{fig:example-illustration}, regularization is related to the domain transferability in this toy example. This motivates the general study of the relationship between regularization and domain transferability in Section~\ref{sec:up}. 
}

\textbf{Robustness is neither necessary nor sufficient for domain transferability. } We may see the relation between adversarial robustness and domain transferability in this example as follows. Given a source model $f^{\gD_S}:\sR^m\to \sR^d$, we consider the adversarial loss on an input $x\in \gM$,  i.e.,
\begin{align}
    \label{eqn:adv-loss}
    \ell_{adv}(x; f^{\gD_S}):=\max_{\delta:\|\delta\|_2\leq \epsilon} \ell(f^{\gD_S}(x+\delta), y_S(x)),
\end{align}
as an indicator of its robustness on the input $x$ on the source domain. The lower the adversarial loss, the better the robustness. We can see that both the regular loss functions $\ell_{\gD_S}(f^{\gD_S})$ and $\ell_{\gD_T}(f^{\gD_S})$ only evaluate $f^{\gD_S}$ on the low-dimensional manifold $\gM$. Therefore, an adversarial perturbation $\delta\in \sR^m$ could make $x+\delta\notin  \gM$ if the loss value is sufficiently high in $\{x+\delta\mid \|\delta\|_2\leq \epsilon\}\backslash \gM$. As a result, in such cases the adversarial loss $\ell_{adv}(x; f^{\gD_S})$ could be arbitrarily high without affecting either the source domain performance $\ell_{\gD_S}(f^{\gD_S})$ or the target domain performance $\ell_{\gD_T}(f^{\gD_S})$, i.e., without affecting their transferability. This implies that improving adversarial robustness is neither necessary nor sufficient for improving domain transferability.  

The toy example illustrates that robustness can be irrelevant to domain transferability, and then the question one may naturally ask is ``what may have a stronger relevance to domain transferability?'' To provide the intuition that regularization may be the key, we make the following analysis using the same toy example. 

\textbf{Intuition on why regularization matters.} Denoting a function space $\gF=\{f:\sR^m\to \sR^d\mid \|f\|_\gD<\infty\}$, we assume $y_S, y_T\in \gF$ such that we can compare $f, y_S, y_T$ in the same space.  Therefore, given $c>0$ as a regularization parameter, we define the domain transferability problem as:
\begin{align}
    &\text{Learning a source model:}\\
    &\qquad f^{\gD_S}_{c}\in \argmin_{f\in \gF} \ell_{\gD_S}(f), \quad \text{s.t.}\ \ \|f\|_\gD\leq c;\label{eq:ex-2}\\
    &\text{Testing on a target domain:} \qquad \ell_{\gD_T}(f^{\gD_S}_{c}),
\end{align}
where the minimizer is $f^{\gD_S}_{c}:=y_S\min\{1, \tfrac{c}{\|y_S\|_\gD}\}$, the source domain loss is $\ell_{\gD_S}(f)=\|f-y_S\|_{\gD}$, and the  target domain loss is $\ell_{\gD_T}(f)=\|f-y_T\|_{\gD}$. We prove in Proposition~\ref{prop:example} that $f^{\gD_S}_{c}$ is indeed a minimizer of \eqref{eq:ex-2}.%

{

}
 
Considering the relation between (relative) domain transferability and the regularization parameter $c$, we have an interesting finding. 
An illustration of the finding is shown in Figure~\ref{fig:example-illustration}, and a more formal statement is provided in Proposition~\ref{prop:example}. As we can see, the relation between regularization and domain transferability is clear if we consider the domain transferability in a ``relative'' way, i.e., the loss value on the target domain minus the loss value on the source domain. A  formal definition of the relative transferability loss is deferred to Definition~\ref{def:relative-transf} in the next subsection.

\begin{proposition}\label{prop:example}
Given the toy example problem defined in Section~\ref{subsec:example}, $f^{\gD_S}_{c}$ is a minimizer of \eqref{eq:ex-2}. If $ c\geq c'\geq 0$, then the relative domain transferability loss $\ell_{\gD_T}(f^{\gD_S}_{c})-\ell_{\gD_S}(f^{\gD_S}_{c})\geq \ell_{\gD_T}(f^{\gD_S}_{c'})-\ell_{\gD_S}(f^{\gD_S}_{c'})$.
\end{proposition}

As we can see from this toy example, robustness is neither necessary nor sufficient to characterize domain transferability. However, there is a monotone relation between the regularization strength and the relative domain transferability loss. \cready{Although the above proposition is derived specifically for the toy example, similar behavior is also observed in our experiments.}
Naturally, these findings motivate the study of the connections between the regularization of the training process and domain transferability in general, as we consider next.

\subsection{Upper Bound of Relative Domain Transferability}  \label{sec:up}
In this subsection, we consider the general transferability problem with fine-tuning. 
We prove that there is a monotone decreasing relationship between the regularization strength and a tight upper bound on the relative domain transferability loss. %
Given a training algorithm $A$, it takes a data distribution $\gD$ and outputs a feature extractor $f_A^{\gD}\in \gF_A$ chosen from a function class $\gF_A$ as well as a fine-tuning function $g_A^{\gD}\in \gG$. First, we formally define the relative domain transferability loss.
\begin{definition}[Relative Domain Transferability Loss]\label{def:relative-transf}
Given the training algorithm $A$ and a pair of distributions $\gD_S,\gD_T\in \gP_{\gX\times \gY}$, the relative domain transferability loss between $\gD_S,\gD_T$ is defined to be the difference of fine-tuned losses, i.e.,
\begin{align}
    \tau(A;\gD_S,\gD_T):=\inf_{g\in \gG} \ell_{\gD_T}(g\circ f_A^{\gD_S})- \ell_{\gD_S}(g_A^{\gD_S} \circ f_A^{\gD_S}).
\end{align}
\end{definition}
\cready{As we can see, when $\ell_{\gD_S}(g_A^{\gD_S} \circ f_A^{\gD_S})$ is the same, smaller $\tau(A;\gD_S,\gD_T)$ means the better performance on the target domain.}

Another perspective of Definition~\ref{def:relative-transf} is that $\inf_{g\in \gG} \ell_{\gD_T}(g\circ f_A^{\gD_S})= \ell_{\gD_S}(g_A^{\gD_S} \circ f_A^{\gD_S})+\tau(A;\gD_S,\gD_T)$. From this perspective, the transferred loss is the source loss plus an additional term to be upper bounded by a certain distance metric between the source and target distributions -- as is common in the literature of domain adaptation (e.g., \cite{ben2007analysis,zhao2019learning}). The key question of the ``distance metric'' remains unanswered. To this end, we propose the following.
\begin{definition}[$(\gG,\gF)$-pseudometric]\label{def:pseudometric}
Given a fine-tuning function class $\gG$, a feature extractor function class $\gF$ and distributions $ \gD_S,\gD_T\in \gP_{\gX\times \gY}$, the $(\gG, \gF)$-pseudometric between $\gD_S,\gD_T$ is
\begin{align}
    d_{\gG,\gF}(\gD_S,\gD_T):=\sup_{f\in \gF}|\inf_{g\in \gG} \ell_{\gD_S}(g\circ f)-\inf_{g\in \gG} \ell_{\gD_T}(g\circ f)|.
\end{align}
Since the fine-tuning function class is usually simple and fixed, we will use $d_{\gF}$ as an abbreviation when $\gG$ is clear.
\end{definition}
It can be easily verified that $d_{\gG, \gF}$ is a pseudometric that measures the \textit{distance} between two distributions, as shown in the following proposition.
\begin{proposition}\label{prop:pseudometric}
$d_{\gG, \gF}(\cdot,\cdot):\gP_{\gX\times \gY}\times \gP_{\gX\times \gY}\to \sR_+$ satisfies the following properties.
\begin{enumerate}[leftmargin=0.5cm]
    \item (Symmetry) $d_{\gG, \gF}(\gD_S,\gD_T)=d_{\gG, \gF}(\gD_T,\gD_S)$.
    \item (Triangle Inequality) For $\forall \gD'\in \gP_{\gX\times \gY}$, we have  $d_{\gG, \gF}(\gD_S,\gD_T)\leq d_{\gG, \gF}(\gD_S,\gD')+d_{\gG, \gF}(\gD',\gD_T)$.
    \item (Weak Zero Property) For $\forall \gD\in \gP_{\gX\times \gY}$: $d_{\gG, \gF}(\gD,\gD)=0$.
\end{enumerate}
\end{proposition}

The motivation of the $(\gG, \gF)$-pseudometric comes from the following observations. We want to study what factors affect how a source model transfers to the target domain. The obvious factor is the difference between the two domains. But the function class where the model is trained from is also an important factor (e.g., the example in Section~\ref{subsec:example}). 
Note that the proposed $(\gG, \gF)$-pseudometric is both a complexity measure of the model function class and a distance measure of two distributions. Given a certain fixed function class,  the $(\gG, \gF)$-pseudometric can serve as a distance measure related to the Wasserstein distance or the total variance distance.  
In proposition~\ref{prop:wasserstein} in the appendix, we show that, if the loss function class is Lipschitz, then the $(\gG, \gF)$-pseudometric between $\mathcal{D}_S$ and $\mathcal{D}_T$ is upper bounded by the product of the Lipschitz constant and the Wasserstein distance between  $\mathcal{D}_S$ and $\mathcal{D}_T$.
Moreover, in proposition~\ref{prop:d-TV} in the appendix, we show that the total variation distance upper bounds the $(\gG, \gF)$-pseudometric if we are working in the realm of multi-class classification and the loss function is the 0-1 loss.

The major difference of the $(\gG, \gF)$-pseudometric with existing metrics for domain transfer~\citep{ben2010theory, mansour2009domain, Acuna2021fDomainAdversarialLT, zhao2019learning} is that the proposed $(\gG, \gF)$-pseudometric is more general. Concretely, the aforementioned work only considers the distributions on the input space $\mathcal{X}$, while we consider both the input space and the output space, i.e, $\mathcal{X}\times\mathcal{Y}$. This difference enables us to consider the fine-tuning process, which is important and widely applied in practice.

In this section, we consider a fixed fine-tuning function class $\gG$ and feature extractor function class $\gF_A$ given by the training algorithm $A$. Thus, we denote $d_{\gG, \gF}$  as $d_{\gF_A}$ for the remainder of the paper. With the definition of $d_{\gF_A}$, we can derive the following result which provides justification for the regularization perspective.
\begin{theorem}\label{thm-bound}
Given a training algorithm $A$, for $\forall \gD_S,\gD_T\in \gP_{\gX\times \gY}$ we have
\begin{align}
    \tau(A;\gD_S,\gD_T) &\leq d_{\gF_A}(\gD_S,\gD_T), \ \text{ or equivalently,}\\
    \inf_{g\in \gG} \ell_{\gD_T}(g\circ f_A^{\gD_S})&\leq \ell_{\gD_S}(g_A^{\gD_S} \circ f_A^{\gD_S}) + d_{\gF_A}(\gD_S,\gD_T).
\end{align}
\end{theorem}
\textbf{Interpretation:} As we can see, the above theorem provides sufficient conditions for good domain transferability. There is a monotone relation between the regularization strength and $ d_{\gF_A}(\gD_S,\gD_T)$, i.e., the upper bound on the relative domain transferability loss $\tau(A;\gD_S,\gD_T)$. More explicitly, if a training algorithm $A'$ has $\gF_{A'}\subseteq \gF_{A}$, then $d_{\gF_{A'}}(\gD_S,\gD_T)\leq d_{\gF_{A}}(\gD_S,\gD_T)$. Moreover, small $ d_{\gF_{A}}(\gD_S,\gD_T)$ implies good relative domain transferability. From this perspective, we can see that we need both small $ d_{\gF_{A}}(\gD_S,\gD_T)$ and small source loss $\ell_{\gD_S}(g_A^{\gD_S} \circ f_A^{\gD_S})$ to guarantee good absolute domain transferability. Note that there is a possible trade-off, i.e., with $\gF_{A}$ being smaller, $d_{\gF_A}(\gD_S,\gD_T)$ decreases but possibly $\ell_{\gD_S}(g_A^{\gD_S} \circ f_A^{\gD_S})$ increases due to the limited power of $\gF_A$.  On the other hand, there may not be such trade-off if $\gD_S$ and $\gD_T$ are close enough such that $d_{\gF_A}(\gD_S,\gD_T)$ is small. %

To make the upper bound more meaningful, we need to study its tightness. 
\begin{theorem}\label{thm-tightness}
Given any source distribution $\gD_S\in \gP_{\gX\times \sR^d}$, any fine-tuning function class $\gG$ where $\gG$ includes the zero function, we assume the training algorithm $A$ is optimal, i.e.,
$
    \ell_{\gD_S}(g_A^{\gD_S} \circ f_A^{\gD_S})= \inf_{g\in \gG, f\in \gF_{A}}\ell_{\gD_S}(g \circ f).
$
We assume some properties of the loss function $\ell:\sR^d\times \sR^d\to \sR_+$: it is differentiable and strictly convex w.r.t. its first argument; $\ell(y,y)=0$ for any $y\in  \sR^d$; and $\lim_{r\to \infty} \inf_{y:\|y\|_2=r} \ell({\vec{0}}, y)=\infty$, where $\vec{0}$ is the zero vector.
Then, given any distribution $\gD^\gX$ on $\gX$, there exist some distributions $\gD_T\in \gP_{\gX\times \sR^d}$ with its marginal on $\gX$ being $\gD^\gX$ such that
\begin{align}
    \tau(A;\gD_S,\gD_T)&=d_{\gF_A}(\gD_S,\gD_T),  \ \text{ or equivalently,}\\
    \inf_{g\in \gG} \ell_{\gD_T}(g\circ f_A^{\gD_S})&= \ell_{\gD_S}(g_A^{\gD_S} \circ f_A^{\gD_S}) + d_{\gF_A}(\gD_S,\gD_T).
\end{align}
\end{theorem}
\textbf{Interpretation:}
In the above theorem, we show that given any $A, \gD_S$, and the marginal $\gD^\gX$, there exist some conditional distributions of $y| x$ such that by composing it with the given $\gD^\gX$ we have a distribution $\gD_T$ where the equality holds in Theorem~\ref{thm-bound}. The optimality assumption on the training algorithm is mild, as it is common for modern neural networks to achieve considerably low loss. Nonetheless, a generalized version of the theorem is provided as Theorem~\ref{thm-tightness-g} in the appendix which works with \textit{any} training algorithm. Alternative form of the tightness analysis is discussed immediately after the proof of Theorem~\ref{thm-tightness-g}. %

Therefore, we prove that stronger regularization on the feature extractor implies a decreased tight upper bound on the {relative} transferability loss. For a cleaner presentation, the analysis so far does not consider the potential influence from finite samples which for sure affects domain generalization. In the next subsection,  we investigate the proposed theory on relative transferability with finite samples. 

\subsection{Generalization Upper Bound of the Relative Domain Transferability} \label{subsec:generalized-ub}
For a distribution $\gD\in \gP_{\gX\times \gY}$, we denote its empirical distribution with $n$ samples as $\widehat\gD^n$. That being said,
\begin{align}
    \ell_{\widehat\gD^n}(g\circ f)&= \E_{(x,y)\sim \widehat\gD^n}[\ell(g\circ f(x), y)]\\
    &=\tfrac{1}{n}\textstyle\sum_{i=1}^n\ell(g\circ f(x_i), y_i),\label{eq:emp-def}
\end{align}
where $(x_i,y_i)$ are i.i.d. samples from $\gD$. Therefore, given two distributions $ \gD_S,\gD_T\in \gP_{\gX\times \gY}$, the empirical $(\gG, \gF)$-pseudometric between them is $d_{\gG, \gF}(\widehat\gD^n_S,\widehat\gD^n_T)$. 

Note that $d_{\gG, \gF}$ is not only a pseudometric of distributions, but also a complexity measure, and we will first connect it with the Rademacher complexity.
\begin{definition}[Empirical Rademacher Complexity~\citep{bartlett2002rademacher, koltchinskii2001rademacher}]\label{def:rad}
Denote the loss function class induced by $\gG, \gF$ as
\begin{align}
    \gL_{\gG, \gF}:=\{h_{g, f}:\gX\times \gY\to \sR_+ \mid g\in \gG, f\in \gF\}, 
\end{align} 
where $h_{g, f}(x, y):=\ell(g\circ f(x), y)$.

Given an empirical distribution $\widehat\gD^n$ (i.e., $n$ data samples), the Rademacher complexity of it is
\begin{align}
    \Rad_{\widehat\gD^n}(\gL_{\gG, \gF}):=\frac{1}{n}\E_{\bm{\xi}}\left[\sup_{h\in \gL_{\gG, \gF}}\sum_{i=1}^n \xi_i h(x_i, y_i)\right],
\end{align}
where $\bm{\xi}\in \sR^n$ are Rademacher variables, i.e., each $\xi_i$ is i.i.d. uniformly distributed on $\{-1, 1\}$.
\end{definition}
We can see that if there is a $\gF'\subseteq \gF$, then $\Rad_{\widehat\gD^n}(\gL_{\gG, \gF'})\leq \Rad_{\widehat\gD^n}(\gL_{\gG, \gF})$. 
With the above definitions, we have the following lemma connecting the $(\gG,\gF)$-pseudometric to Rademacher complexity.
\begin{lemma}\label{lemma:emp-finite-dist}
Assuming the  loss function $\ell:\gY\times\gY\to [0, c]$, given any distribution $\gD\in \gP_{\gX\times \gY}$ and $\forall \delta>0$, with probability $\geq 1-\delta$ we have
\begin{align}
    d_{\gG, \gF}(\gD, \widehat\gD^n)\leq 2\Rad_{\widehat\gD^n}(\gL_{\gG, \gF})+3c\sqrt{\frac{\ln(4/\delta)}{2n}}.
\end{align}
\end{lemma}
Therefore, denoting again $d_{\gF_A}$ as $d_{\gG, \gF_A}$,  the empirical version of Theorem~\ref{thm-bound} is as follows.
\begin{theorem}\label{thm:emp-bound}
Assuming the  loss function $\ell:\gY\times\gY\to [0, c]$, given $\forall \gD_S, \gD_T\in \gP_{\gX\times \gY}$, for $\forall \delta>0$ with probability $\geq 1-\delta$ we have
\begin{align}
    \tau(A; \widehat\gD^n_S, \gD_T)&\leq d_{\gF_A}(\widehat\gD^n_S, \widehat\gD^n_T) +2\Rad_{\widehat\gD_T^n}(\gL_{\gG, \gF_A})\\
    &+ 4\Rad_{\widehat\gD_S^n}(\gL_{\gG, \gF_A})+9c\sqrt{\frac{\ln(8/\delta)}{2n}}.
\end{align}
\end{theorem}
\textbf{Interpretation:} We can see that a smaller feature extractor function class $\gF_A$ implies both a smaller $d_{\gF_A}$ and the Rademacher complexity. Therefore, the monotone relation between the regularization strength and the upper bound on the relative domain transferability loss also holds for the empirical settings.

The proposed theoretical analysis suggests that regularization may be a fundamental perspective to understand domain transferability. Other than explicit regularization, empirically we find that the transferability is also related to the use of certain data augmentation and adversarial training. Can we explain such phenomena from the view of regularization again? We discuss this question in the next section.

%% file: regularizations.tex
\input{dataaug}

\textbf{Adversarial training.} It is known that adversarial training, a special kind of data augmentation, can be viewed as regularization in some scenarios~\citep{roth2020reg}. We further prove that, under certain conditions, adversarial training reduces the size of the feature extractors function class during training (see Section~\ref{sec:adv-reg} for details). Therefore, our theoretical analysis implies that adversarial training helps domain transferability from its regularization effect.

%% file: dataaug.tex
\section{When Can Data Augmentation be Viewed as Regularization?}
\label{sec:DA}

In this section, we discuss the connections between data augmentation (DA) and regularization. We present the results and their interpretation in this section, while deferring the detailed discussion and comparisons with related work to Section~\ref{sec:da} in the appendix.

\textbf{General settings.} We consider the fine-tuning function $g:\mathbb{R}^d\rightarrow \mathbb{R}$ as a linear layer, which will be concatenated to the feature extractor $f:\mathbb{R}^m\rightarrow \mathbb{R}^d$. Given a model $g\circ f$, we use the squared loss $\ell(g\circ f(x), y)=(g\circ f(x)-y)^2$, and accordingly apply second-order Taylor expansion to the objective function to study the effect of data augmentation.

\textbf{DA categories.} We discuss two categories of DA, the {\textit{feature-level DA}} and the {\textit{data-level DA}}. The feature-level DA~\citep{wong2016understanding, devries2017dataset} requires the transformation to be performed in the learned feature space: given a data sample $x\in \sR^m$ and a feature extractor $f$, the augmented feature is $W_\star f(x)+b_\star$ where $W_\star\in \sR^{d\times d}, b_\star\in \sR^d$ are sampled from a distribution. On the other hand, the data-level DA requires the transformation to be performed in the input space: given a data sample $x$, the augmented sample is $W_\star x + b_\star$ where $W_\star\in \sR^{m\times m}, b_\star\in \sR^m$ are sampled from a distribution.

\textbf{Intuition on sufficient conditions}. For either the feature-level or the data-level DA, the intuitions given by our analysis are similar. Our results (Theorem~\ref{apdx:FLDA}\&\ref{apdx:DLDA}) suggest that the following conditions indicate regularization effects of a data augmentation: 
1) $\E_{W_\star}[W_\star] = \mathbb{I}$; 
2) $\E_{b_\star}[b_\star] = \vec{0}$; 
3) $W_\star$ and $b_\star$ are independent, where $\mathbb{I}$ is the identity matrix and $\vec{0}$ is the zero vector;
4) $W_\star$ is not a constant if it is the feature-level DA;
5) DA is of a small magnitude if it is the data-level DA.

\textbf{Empirical verification.} Combining with Theorem \ref{thm:emp-bound}, it suggests that DA satisfying the conditions above may improve the relative domain transferability. In fact, it matches the empirical observations in Section~\ref{sec:exp}. Concretely, 
\textbf{\textit{1) Gaussian noise}} {\textit{satisfies}} the four conditions, and empirically the Gaussian noise improves domain transferability while robustness decreases a bit (Figure~\ref{fig:result-da}); \textbf{\textit{2) Rotation}}, which rotates input image with a predefined fixed angle with predefined fixed probability, {\textit{violates}} $\E_{W_\star}[W_\star] = \mathbb{I}$, and empirically the rotation barely affects domain transferability (Figure~\ref{fig:result-rotate}); 
\textbf{\textit{3) Translation}}, which moves the input image for a predefined  distance along a pre-selected axis with  fixed probability, {\textit{violates}} $\E_{b_\star}[b_\star] = \vec{0}$, and empirically the translation distance  barely co-relates to the domain transferability (Figure~\ref{fig:result-rotate}).

%% file: exp.tex
\section{Evaluation}\label{sec:exp} %
\subsection{Experimental Setting}
\label{sec:exp-setting}

\textbf{Source model training.} We train our model on two source domains: CIFAR-10 and ImageNet. Unless specified, we will use the training settings as follows\footnote{These settings are inherited from the standard training algorithms for CIFAR-10 (\url{https://github.com/kuangliu/pytorch-cifar}) and ImageNet (\url{https://github.com/pytorch/examples/tree/master/imagenet}).}. For CIFAR-10, we train the model with 200 epochs using the momentum SGD optimizer with momentum 0.9, weight decay 0.0005, an initial learning rate 0.1 which decays by a factor of 10 at the 100-th and 150-th epoch. For ImageNet, we train the model with 90 epochs using the momentum SGD optimizer with momentum 0.9, weight decay 0.0001, an initial learning rate 0.1 which decays by a factor of 10 at the 30-th and 60-th epoch. We use the standard cross-entropy loss denote as $L_{CE}(h_s,x,y)$, where $h_s=g_s\circ f$ is the trained model and $x,y$ are the input and label respectively. For both tasks, we use ResNet-18 as the model 
architecture. We provide results of other model structures in Appendix~\ref{sec:app-wrncnn-result}.

\textbf{Model robustness evaluation.}  To evaluate the model robustness on the source domain, we will show the model accuracy under adversarial attack. We follow the evaluation setting in \cite{ilyas2019adversarial} and perform the PGD attack with 20 steps using $\epsilon=0.25$. This empirical robust accuracy reflects how well the model performs under adversarial attack, which is the adversarial loss as in \eqref{eqn:adv-loss} if we view $\ell(\cdot,\cdot)$ as the 0-1 loss between prediction and ground truth. We also provide robustness evaluation with AutoAttack in Appendix~\ref{sec:app-autoatk}.

\textbf{Domain transferability.} We evaluate the transferability from CIFAR-10 to SVHN and from ImageNet to CIFAR-10. 
For the ImageNet, we focus on CIFAR as the target domain, since it is the domain that is the most positively correlated with robustness as shown in \cite{salman2020adversarially}. We evaluate the fixed-feature transfer where only the last fully-connected layer is fine-tuned following our theoretical framework. We fine-tune the last layer with 40 epochs using SGD  with momentum 0.9, weight decay 0.0005, an initial learning rate 0.01 which decays by a factor of 10 at the 20-th and 30-th epoch. To mitigate the impact of benign accuracy, we evaluate the \textit{relative domain transfer accuracy} (DT Acc) as follows. Let $acc_{src}$ and $acc_{tgt}$ be the accuracy of the fine-tuned model on the source and target domain, and $acc_{src}^{v}$ and $acc_{tgt}^{v}$ be the accuracy of vanilla model (\textit{i.e.}, models trained with standard settings) on source and target domain, then the {relative} DT accuracy is defined as:
$$
    \text{DT Acc} = (acc_{tgt} - acc_{src}) - (acc_{tgt}^v - acc_{src}^v).
$$
\cready{Note that by definition, we can directly use $acc_{tgt} - acc_{src}$ as the relative accuracy. We use a relative score ($acc_{tgt}^v - acc_{src}^v$) so that the positive/negative values reflect the comparison with the vanilla-trained model.}
We also provide the results of absolute DT accuracy in Appendix~\ref{sec:app-abs-result}. 

\begin{figure}[t]
    \centering
    \includegraphics[height=1.3in]{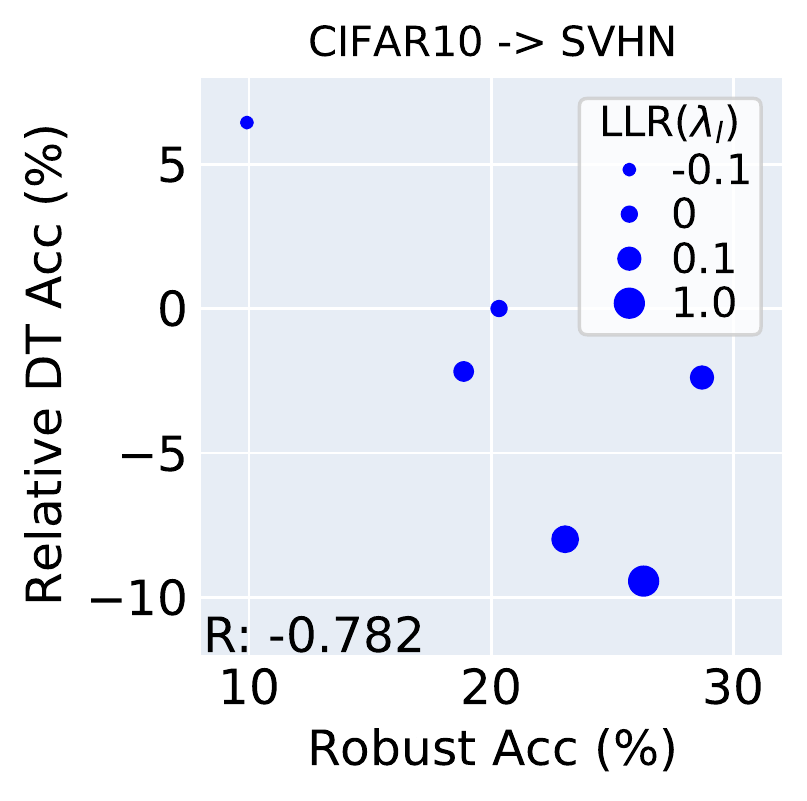}
    \includegraphics[height=1.3in]{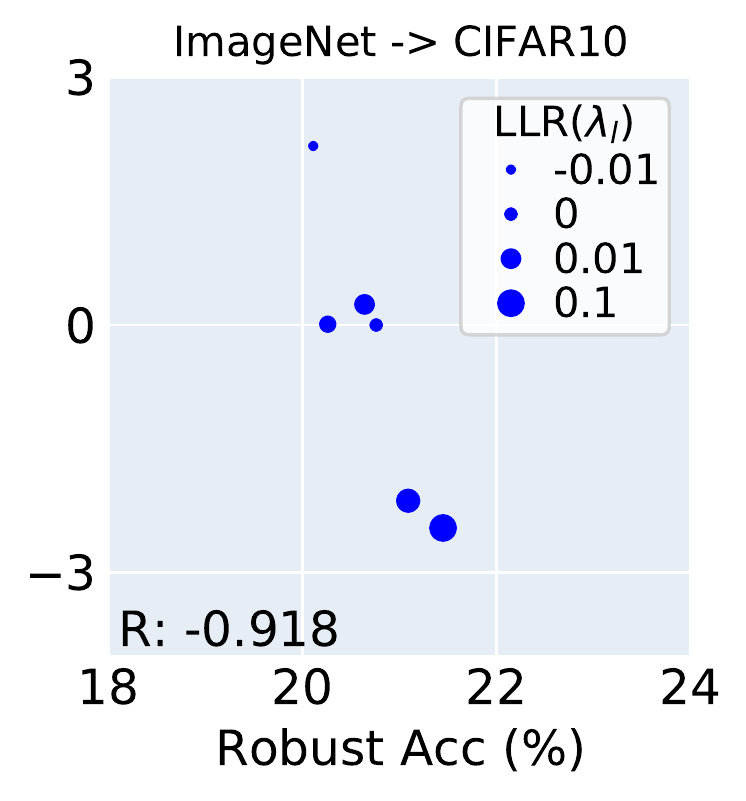}
    \includegraphics[height=1.3in]{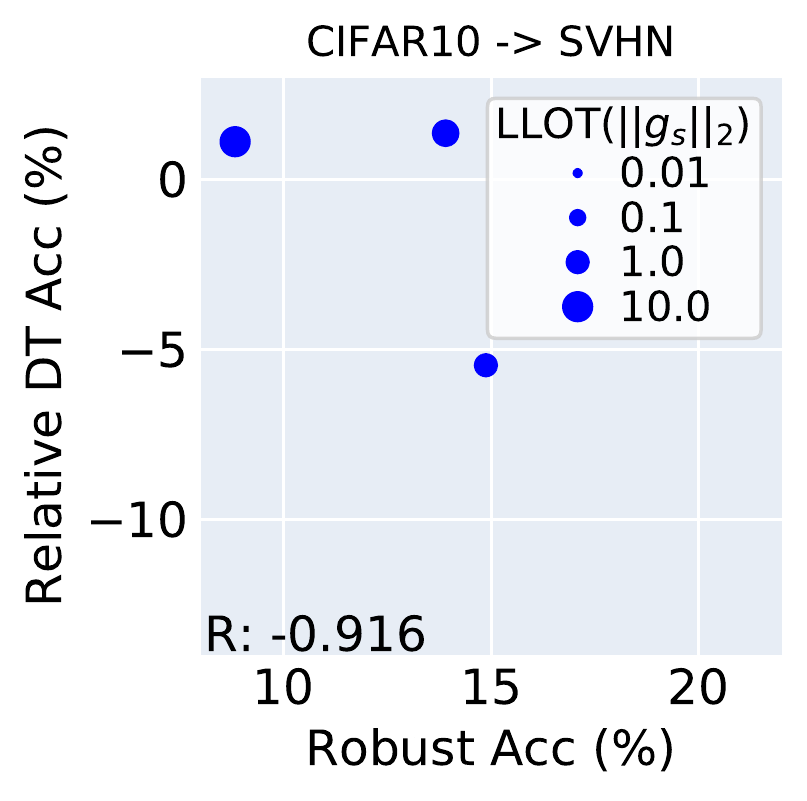}
    \includegraphics[height=1.3in]{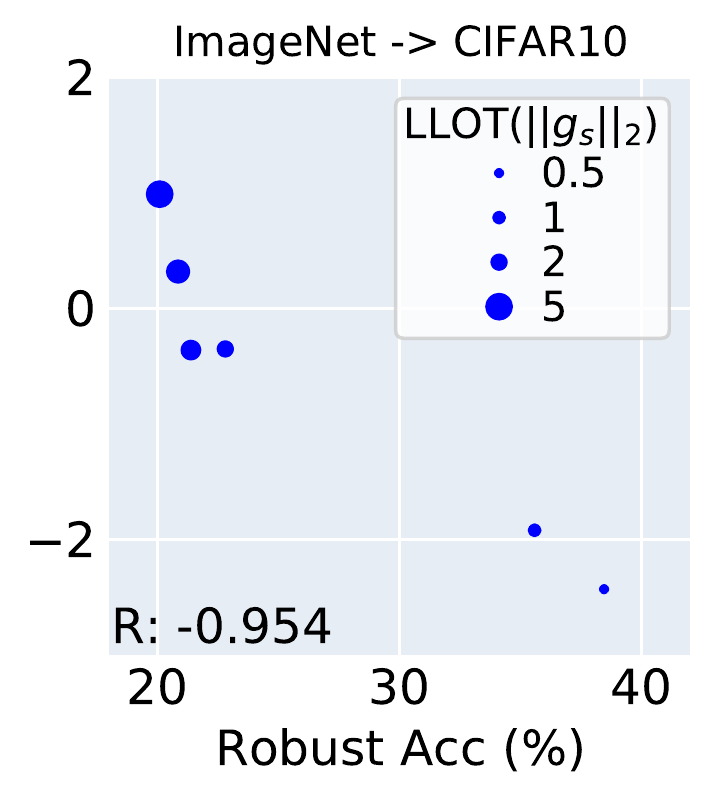}
    \caption{Relationship between robustness and transferability under different norms of last layer, via training with last-layer regularization (LLR) and last-layer orthogonalization (LLOT) %
    }
    \label{fig:result-ll}
\end{figure}
\begin{figure}[t]
    \centering
    \includegraphics[height=1.4in]{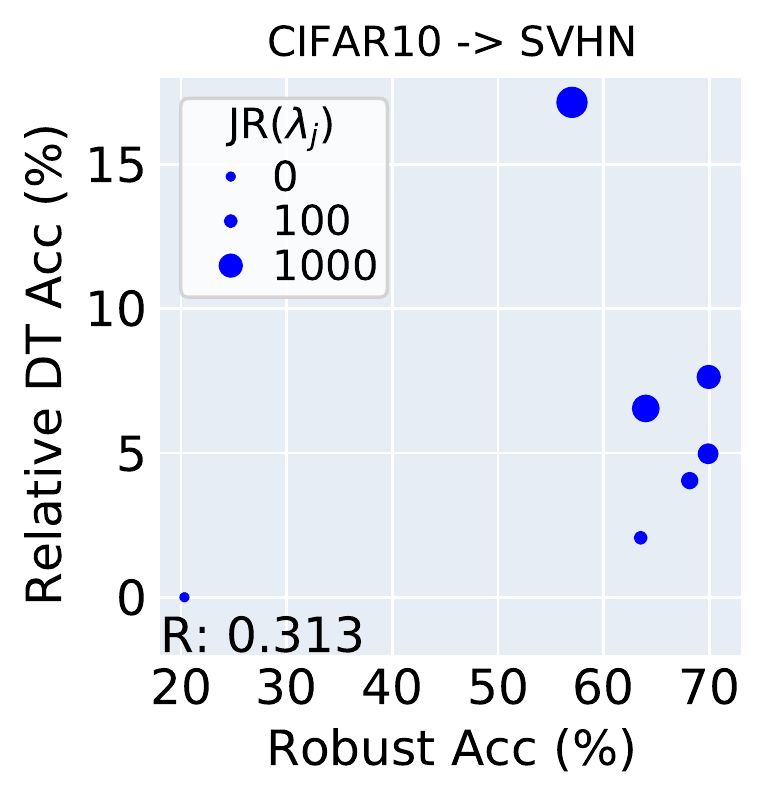}
    \includegraphics[height=1.4in]{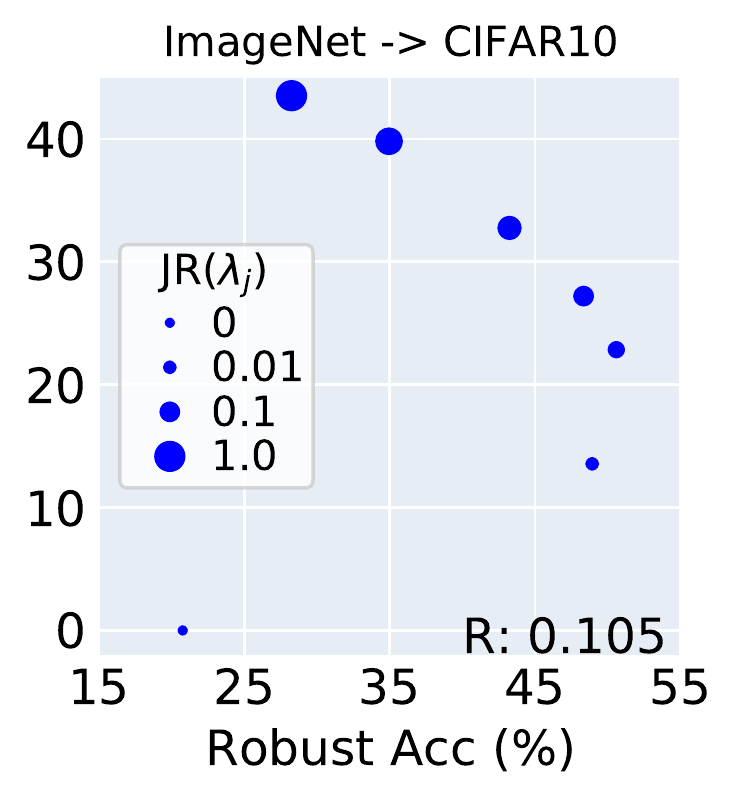}
    \includegraphics[height=1.4in]{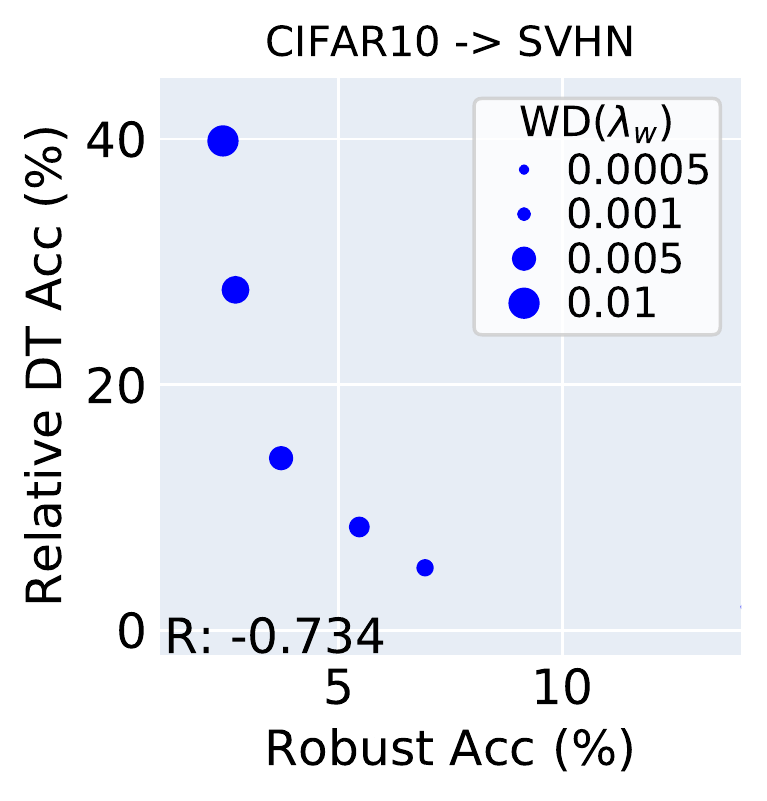}
    \includegraphics[height=1.4in]{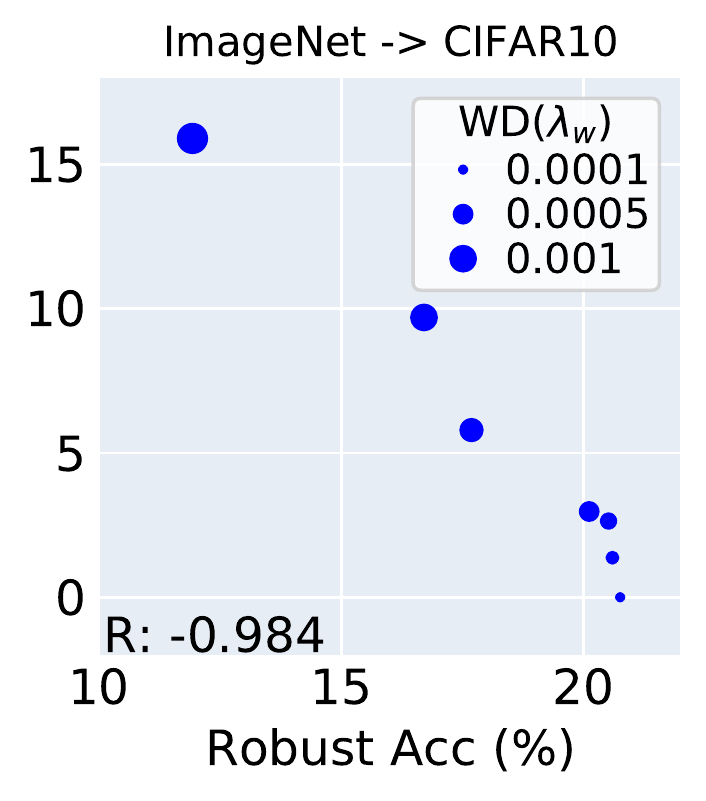}
    \caption{Relationship between robustness and transferability when we regularize the feature extractor with Jacobian Regularization (JR) and weight decay (WD).}
    \label{fig:result-jr}
\end{figure}
\subsection{Relationship between Robustness and Transferability Under Controllable Conditions}
We train the model under different controllable conditions to validate our analysis. In particular, we train the methods by controlling different regularization or data augmentations to evaluate the model robustness and transferability. We emphasize that our goal is to  identify conditions for domain transferability, rather than proposing methods to achieve the state-of-the-art transferable models. Nevertheless, we do show in Appendix~\ref{sec:app-cmp-result} that with basic regularization the model can achieve better absolute transferability than vanilla trained or  adversarially trained models in some cases.

\textbf{Controlling the last-layer norm.}
As shown in our theory, (relative) domain transferability is related to the regularization of feature extractors. Here we regularize the  transferability by controlling the last-layer norm  $g_s$. Intuitively, when we force the norm of $g_s$ to be big during training, the corresponding norm of $f$ will be regularized to be small. We use two approaches to control the last-layer norm:
\begin{itemize}[leftmargin=*, topsep=1pt,itemsep=1pt,partopsep=1pt, parsep=1pt]
    \item Last-layer regularization (LLR): we impose a strong l2-regularizer with parameter $\lambda_l$ specifically on the weight of $g_s$ and therefore our training loss becomes: $L_{LLR}(h_s,x,y) = L_{CE}(h_s,x,y) + \lambda_l \cdot ||g_s||_F$,
    where $||g_s||_F$ is the frobenius norm of the weight matrix of $g_s$.
    \item Last-layer orthogonal training (LLOT): we directly control the l2-norm of $g_s$ with orthogonal training (\cite{huang2020controllable}). The orthogonal training will enforce the weight to become a 1-norm matrix and we multiply a constant to obtain the desired norm $||g_s||_2$.
\end{itemize}

The result of LLR and LLOT are shown in Figure \ref{fig:result-ll}. We observe that when we regularize the norm of the last layer to be large (i.e. smaller $\lambda$ in LLR and larger $||g_s||_2$ in LLOT), the relative domain transferability will increase while the model robustness will decrease (their negative correlation is significant with Pearson's coefficient around $-0.9$). This is because the larger last layer norm will produce a feature extractor $f$ with a smaller norm, which, according to our analysis, leads to a better relative domain transferability. On the other hand, the model $g_s\circ f$ will have a larger norm and therefore becomes less robust under adversarial attacks.

\begin{figure}[t]
    \centering
    \includegraphics[height=1.4in]{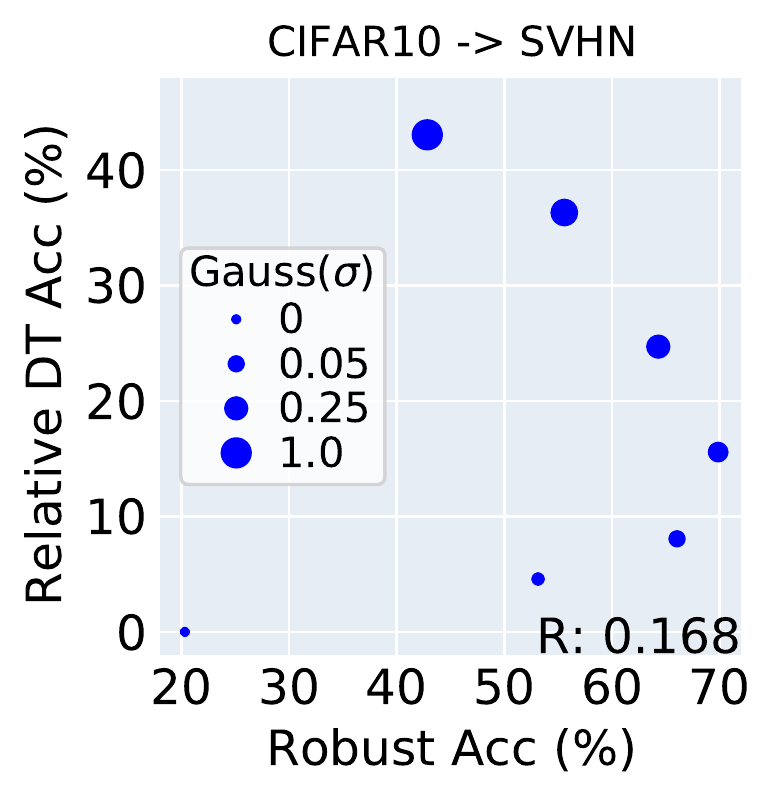}
    \includegraphics[height=1.4in]{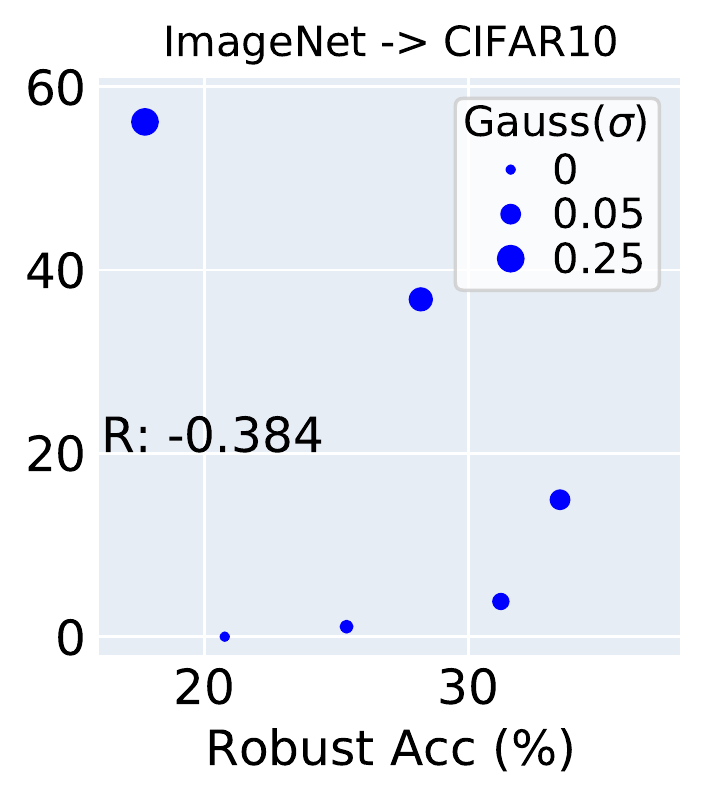}
    \includegraphics[height=1.4in]{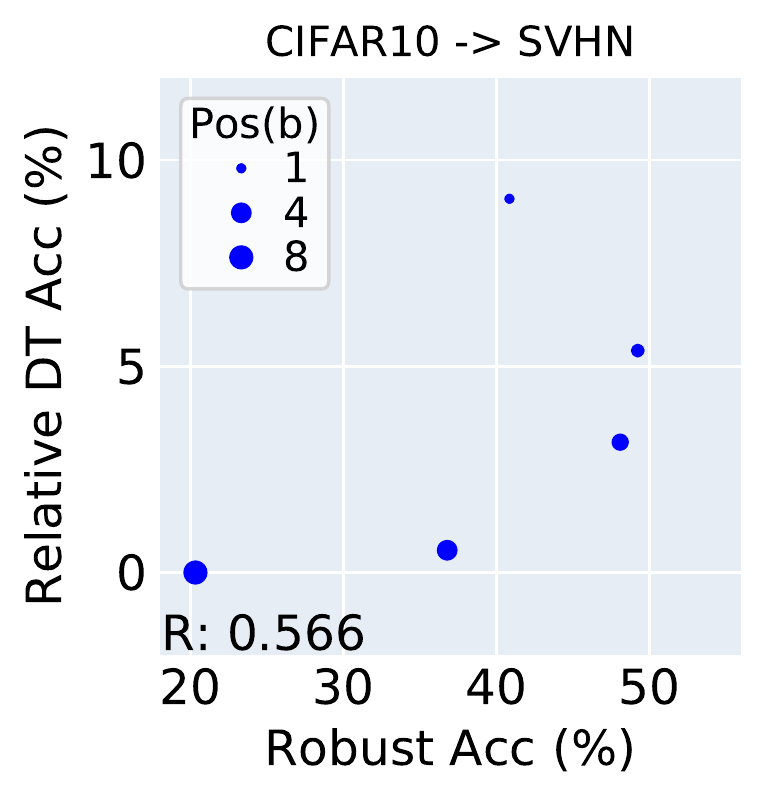}
    \includegraphics[height=1.4in]{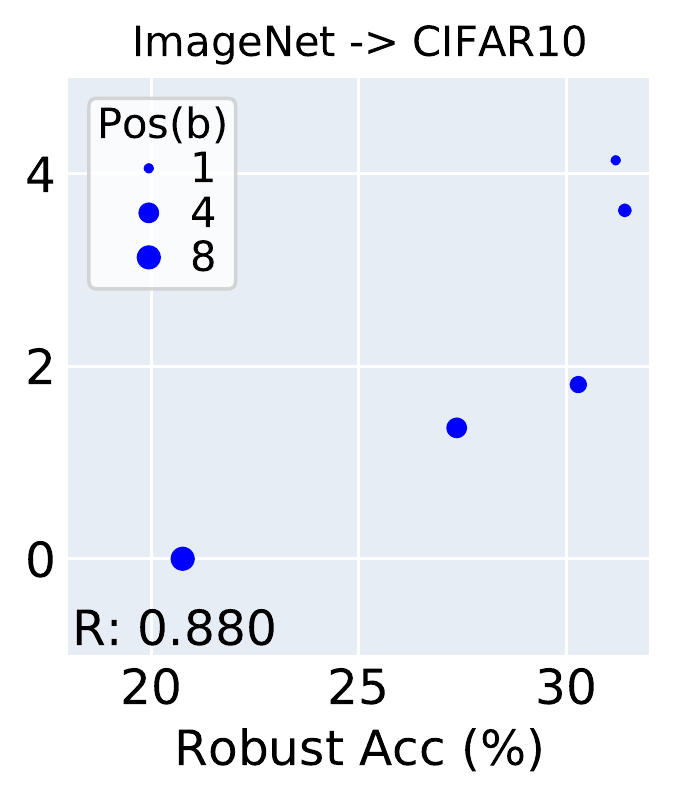}
    \caption{Relationship between robustness and transferability when we use Gaussian noise (\textit{Gauss}) and posterize (\textit{Pos}) as data augmentations.}
    \label{fig:result-da}
\end{figure}

\textbf{Controlling the norm of feature extractor.}
We directly regularize the feature extractor $f$ and check the impact on the (relative) domain transferability. We implement two regularization as follows:
\begin{itemize}[leftmargin=*, topsep=1pt,itemsep=1pt,partopsep=1pt, parsep=1pt]
    \item Jacobian regularization (JR): we follow the approach in \cite{hoffman2019robust} to apply JR on the feature extractor. Given model $h_s=g_s\circ f$, the training loss becomes: $L_{JR}(g_s\circ f,x,y) = L_{CE}(g_s\circ f,x,y) + \lambda_j \cdot ||J(f,x)||_F^2$,
    where $J(f,x)$ denotes the Jacobian matrix of $f$ on $x$ and $||\cdot||_F$ is the frobenius norm.
    \item Weight Decay (WD): we impose  weight decay with factor $\lambda_w$ on the feature extractor $f$ during training. This is equivalent to imposing l2-regularizer with factor $\lambda_w$ on the feature extractor (excluding the last layer).
\end{itemize}

The results under JR and WD are shown in Figure~\ref{fig:result-jr}. We observe that with larger regularization on the feature extractor, the model shows higher relative domain transferability, which matches our analysis. Meanwhile, the robustness decreases significantly with a  large regularizer. This is because a large regularization will harm the model performance on the source domain and lead to low model robustness.

\textbf{Noise-dependent data augmentation.} As shown in Section~\ref{sec:DA}, certain data augmentation can be viewed as a type of regularization during training and thus affects the (relative) domain transferability.
Here we consider both noise dependent and independent data augmentations. For the noise-dependent case, We include two augmentations:
\begin{itemize}[leftmargin=*, topsep=1pt,itemsep=1pt,partopsep=1pt, parsep=1pt]
    \item Gaussian Noise data augmentation (\textit{Gauss}): we add zero-mean Gaussian noise with variance $\sigma^2$ to the input image.
    \item Posterize (\textit{Pos}): we truncate each channel of one pixel value into $b$ bits (originally they are 8 bits).
\end{itemize}
\begin{figure}[t]
    \centering
    \includegraphics[height=1.4in]{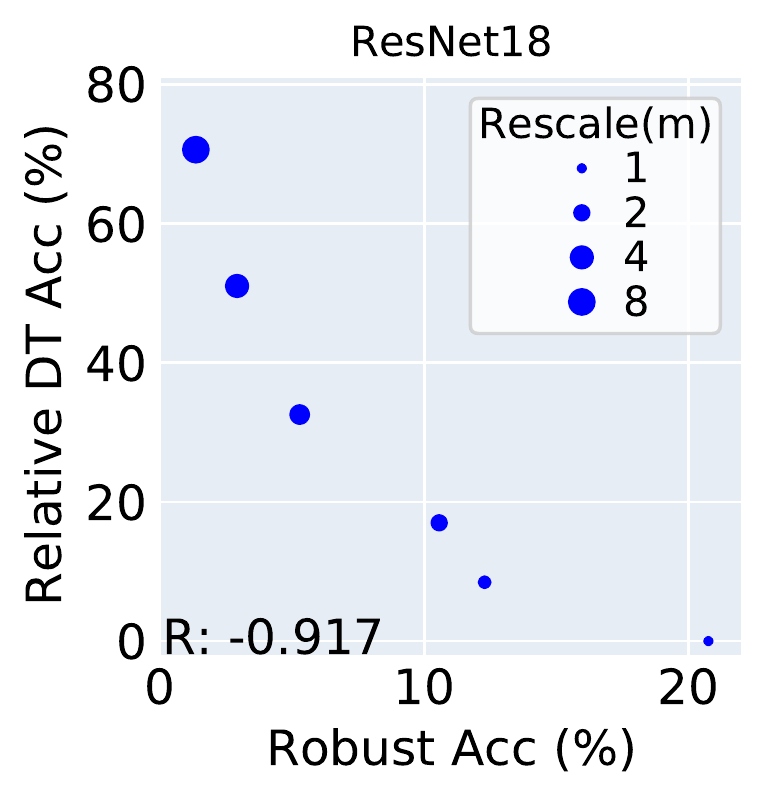}
    \includegraphics[height=1.4in]{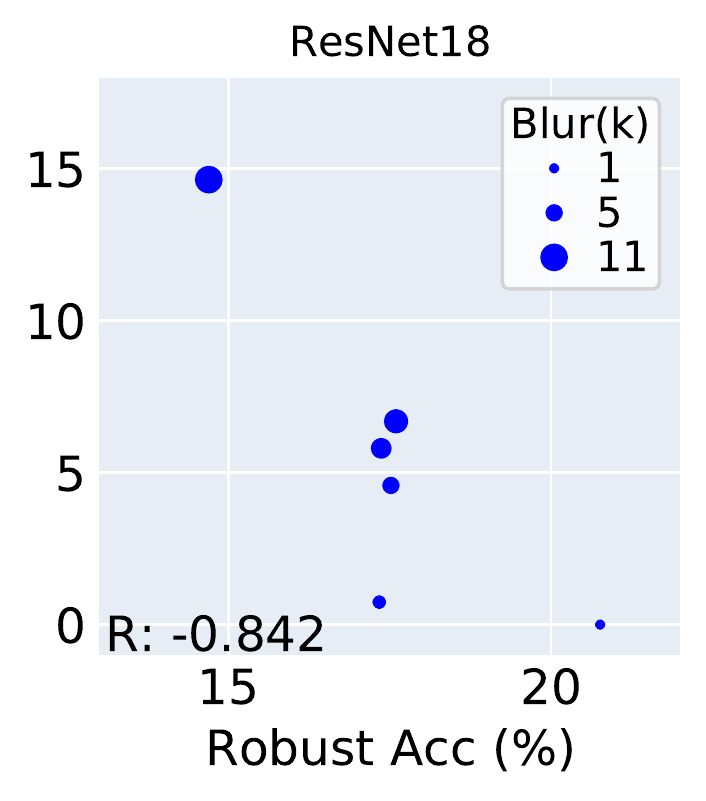}
    \includegraphics[height=1.4in]{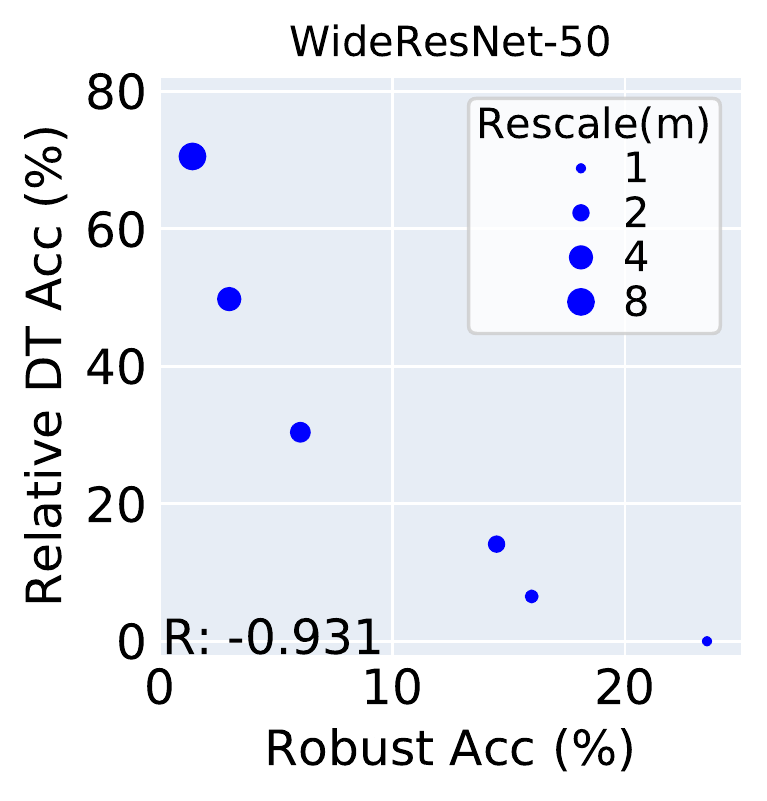}
    \includegraphics[height=1.4in]{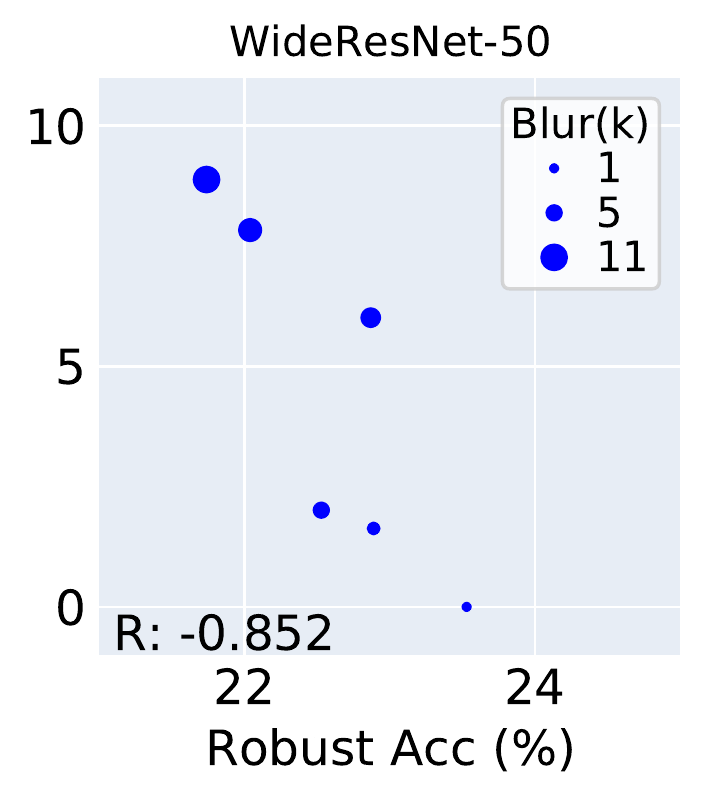}
    \caption{Relationship between robustness and transferability on ImageNet when we use rescale and blur as data augmentations.}
    \label{fig:result-resolution}
\end{figure}

The results of \textit{Gauss} and \textit{Pos} are shown in Figure~\ref{fig:result-da}. We observe that the relative domain transferability of the trained models improves with greater data augmentation, matching our theory. The robustness also benefits from a small data augmentation but decreases when it becomes large.

\textbf{Resolution-related (noise-independent) data augmentation.}
Specifically, for ImageNet to CIFAR-10 transferability, we consider two resolution-related data augmentations. The intuition is that when the target domain has a lower resolution than the source domain (ImageNet is $224\times 224$ while CIFAR-10 is $32\times 32$), the data augmentations that down-sample the inputs during the training on the source domain will help  transferability. We consider the below resolution-related augmentations:
\begin{itemize}[leftmargin=*, topsep=1pt,itemsep=1pt,partopsep=1pt, parsep=1pt]
    \item Rescale: we rescale the input to be $m$ times smaller (\textit{i.e.},  shape ImageNet as $(224/m)\times(224/m)$) and then rescale them back to the original size.
    \item Blur: we apply Gaussian blurring with kernel size $k$ on the input. The Gaussian kernel is created with a standard deviation randomly sampled from $[0.1,2.0]$.
\end{itemize}
The corresponding results are shown in Figure~\ref{fig:result-resolution}. The experiments are evaluated only for ImageNet to CIFAR-10, and we include the results of both ResNet18 (the default model) and WideResNet50. We can see that the data augmentations help with relative domain transferability to the target domain, although the robustness on the source domain decreases since these augmentations do not relate to robustness operations

\subsection{Other Data Augmentations}
\label{sec:badda}

\cready{
In addition, we study rotation and translation, the two data augmentations that violate the sufficient condition for regularization as we discussed in Section~\ref{sec:DA}. The result is shown in Figure~\ref{fig:result-rotate}. We observe that these augmentations do not have an obvious impact on domain transferability, which is consistent with our theoretical analysis.
}
\begin{figure}[tb]
    \centering
    \includegraphics[height=1.4in]{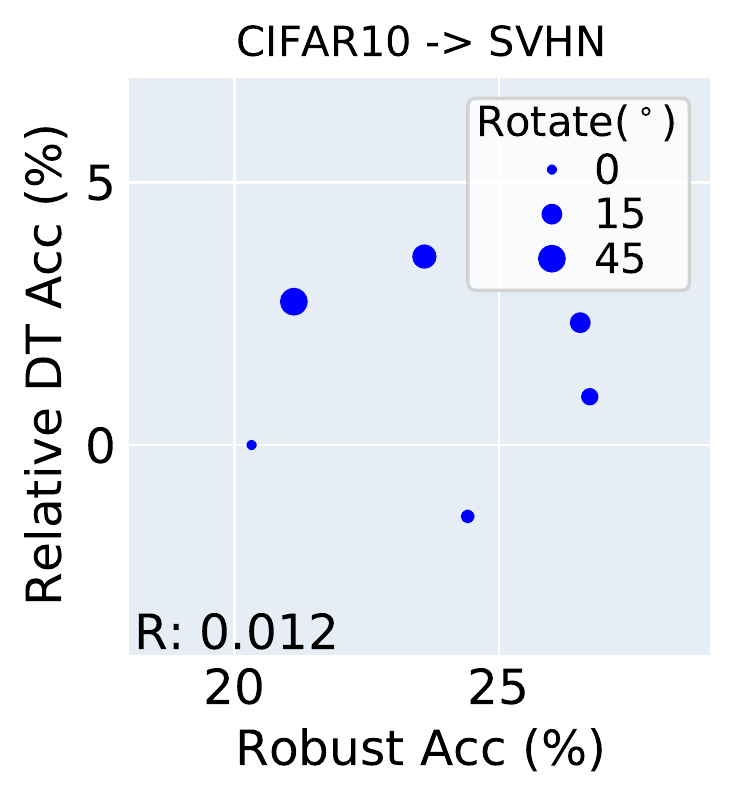}
    \includegraphics[height=1.4in]{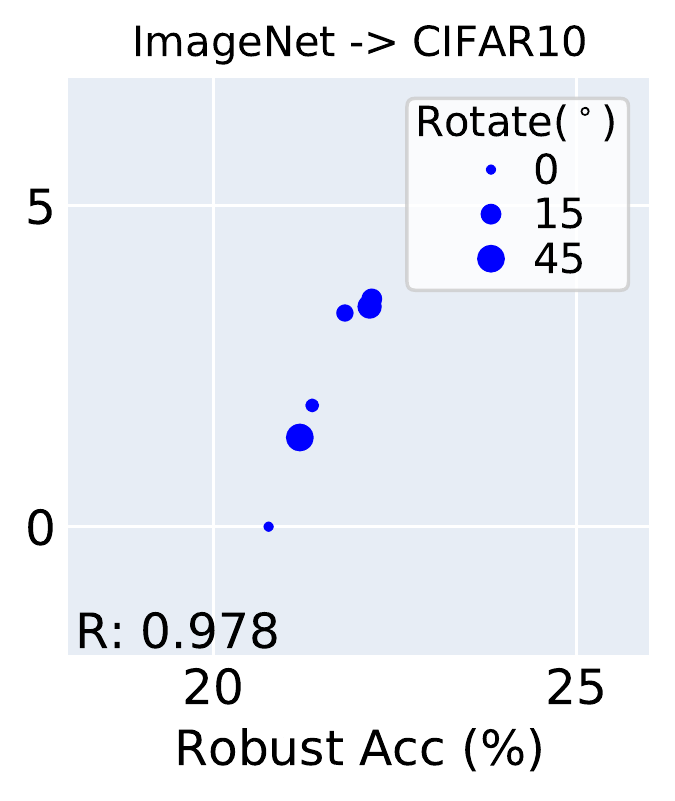}
    
    \includegraphics[height=1.4in]{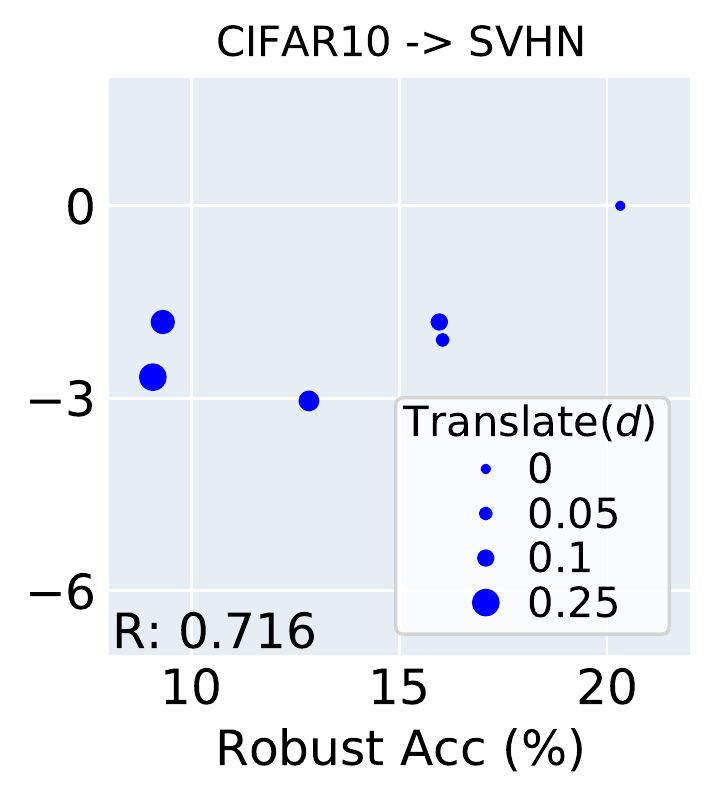}
    \includegraphics[height=1.4in]{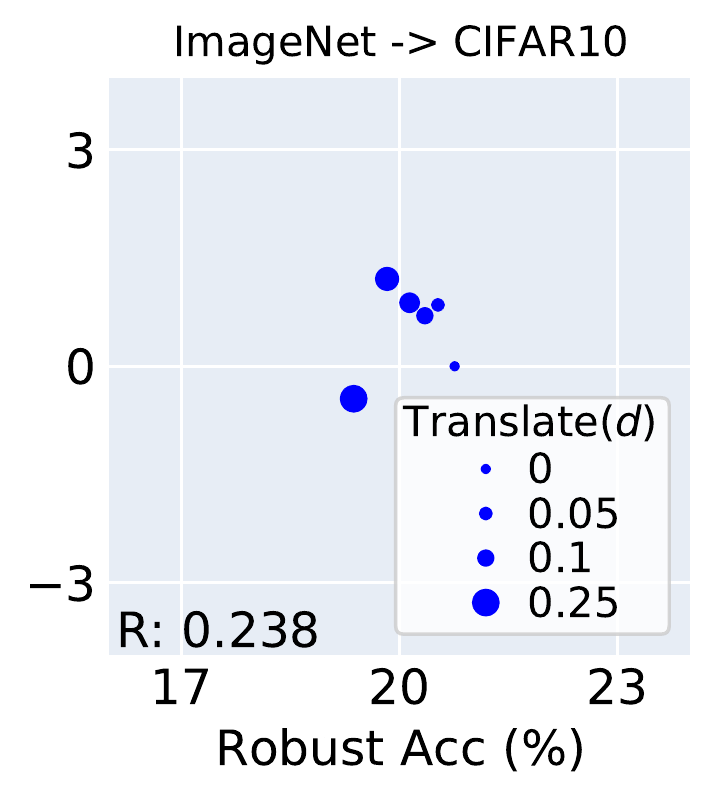}
    \caption{\cready{Relationship between robustness and transferability when we use rotation and translation as data augmentations, which violate the sufficient condition for regularization. We cannot see any obvious trend for such augmentations.}}
    \label{fig:result-rotate}
\end{figure}

\section{Conclusions} %
In this work, we theoretically analyze the sufficient conditions for (relative) domain transferability based on the view of function class regularization. We also conduct experiments to verify our claims and observe some counterexamples that show negative correlations between robustness and domain transferability. These results would contribute to a better understanding of the domain generalization.

%% file: appendix.tex
\begin{center}
\LARGE

\textbf{Appendix}
\end{center}

\section{Proofs}\label{sec:proofs}
\begin{proposition}[Proposition~\ref{prop:example} Restated]
Given the problem defined in subsection~\ref{subsec:example}, $f^{\gD_S}_{c}$ is a minimizer of \eqref{eq:ex-2}. Moreover, if $ c\geq c'\geq 0$, then the relative domain transfer loss $\ell_{\gD_T}(f^{\gD_S}_{c})-\ell_{\gD_S}(f^{\gD_S}_{c})\geq \ell_{\gD_T}(f^{\gD_S}_{c'})-\ell_{\gD_S}(f^{\gD_S}_{c'})$.
\end{proposition}
\begin{proof}
Recall that $\ell_{\gD_S}(f)=\|f-y_S\|_{\gD}$, $\ell_{\gD_T}(f)=\|f-y_T\|_{\gD}$ and $f^{\gD_S}_{c}:=y_S\min\{1, \tfrac{c}{\|y_S\|_\gD}\}$. First, let's verify that 
\begin{align}
    f^{\gD_S}_{c}\in \argmin_{f\in \gF} \ \ell_{\gD_S}(f), \quad \text{s.t.}\ \ \|f\|_\gD\leq c.
\end{align}
If $c\geq\|y_S\|_\gD$, then $f^{\gD_S}_{c}=y_S$ minimizes $\ell_{\gD_S}(f)=\|f-y_S\|_\gD$ to achieve $0$. 

If $c<\|y_S\|_\gD$, then we can show that $f^{\gD_S}_{c}=\tfrac{c}{\|y_S\|_\gD}y_S$ achieves the minimum. For any $f\in \gF:\|f\|_\gD\leq c$, we have 
\begin{align}
    \ell_{\gD_S}(f)&=\|f-y_S\|_{\gD}\geq \|y_S\|_{\gD}-\|f\|_{\gD}\geq \|y_S\|_\gD-c\\
    &=\|(1-\tfrac{c}{\|y_S\|_\gD})y_S\|_\gD=\|y_S-\tfrac{c}{\|y_S\|_\gD}y_S\|_\gD=\ell_{\gD_S}(f^{\gD_S}_{c}).
\end{align}
Therefore, $ f^{\gD_S}_{c}$ indeed achieves the minimum. 

Now, let's prove the proposition. For any $c\geq \|y_S\|_\gD$, we have $\ell_{\gD_S}(f^{\gD_S}_{c})=0$ and $\ell_{\gD_T}(f^{\gD_S}_{c})$ is a constant. Therefore, there is no difference for all $c\geq \|y_S\|_\gD$, and the proposition holds for $c\geq c'\geq \|y_S\|_\gD$. Then, We only need to verify the case for $\|y_S\|\geq c\geq c'$:
\begin{align}
    \ell_{\gD_S}(f^{\gD_S}_{c'})-\ell_{\gD_S}(f^{\gD_S}_{c})&=c-c'=\|\tfrac{c}{\|y_S\|_\gD}y_S-\tfrac{c'}{\|y_S\|_\gD}y_S\|_\gD\\
    &=\|f^{\gD_S}_{c'}-f^{\gD_S}_{c}\|_\gD=\|f^{\gD_S}_{c'}-y_T+y_T-f^{\gD_S}_{c}\|_\gD\\
    &\geq |\|f^{\gD_S}_{c'}-y_T\|_\gD-\|y_T-f^{\gD_S}_{c}\|_\gD|\\
    &\geq \|f^{\gD_S}_{c'}-y_T\|_\gD-\|y_T-f^{\gD_S}_{c}\|_\gD\\
    &= \ell_{\gD_T}(f^{\gD_S}_{c'})-\ell_{\gD_T}(f^{\gD_S}_{c}).
\end{align}
Rearranging the above inequality gives the proposition. 
\end{proof}

\begin{proposition}[Proposition~\ref{prop:pseudometric} Restated]
$d_{\gG, \gF}(\cdot,\cdot):\gP_{\gX\times \gY}\times \gP_{\gX\times \gY}\to \sR_+$ satisfies the following three properties. 
\begin{enumerate}
    \item (Symmetry) $d_{\gG, \gF}(\gD_S,\gD_T)=d_{\gG, \gF}(\gD_T,\gD_S)$.
    \item (Triangle Inequality) For $\forall \gD'\in \gP_{\gX\times \gY}$:  $d_{\gG, \gF}(\gD_S,\gD_T)\leq d_{\gG, \gF}(\gD_S,\gD')+d_{\gG, \gF}(\gD',\gD_T)$.
    \item (Weak Zero Property) For $\forall \gD\in \gP_{\gX\times \gY}$: $d_{\gG, \gF}(\gD,\gD)=0$.
\end{enumerate}
\end{proposition}
\begin{proof}
Recall that 
\begin{align}
    d_{\gG,\gF}(\gD_S,\gD_T):=\sup_{f\in \gF}|\inf_{g\in \gG} \ell_{\gD_S}(g\circ f)-\inf_{g\in \gG} \ell_{\gD_T}(g\circ f)|.
\end{align}
We can see that the symmetry and weak zero property are obvious. For triangle inequality, given $\forall \gD'\in \gP_{\gX\times \gY}$:
\begin{align}
    d_{\gG,\gF}(\gD_S,\gD_T)&=\sup_{f\in \gF}|\inf_{g\in \gG} \ell_{\gD_S}(g\circ f)-\inf_{g\in \gG} \ell_{\gD_T}(g\circ f)|\\
    &=\sup_{f\in \gF}|\inf_{g\in \gG} \ell_{\gD_S}(g\circ f)-\inf_{g\in \gG} \ell_{\gD'}(g\circ f)+\inf_{g\in \gG} \ell_{\gD'}(g\circ f)-\inf_{g\in \gG} \ell_{\gD_T}(g\circ f)|\\
    &\leq \sup_{f\in \gF} ( |\inf_{g\in \gG} \ell_{\gD_S}(g\circ f)-\inf_{g\in \gG} \ell_{\gD'}(g\circ f)|+|\inf_{g\in \gG} \ell_{\gD'}(g\circ f)-\inf_{g\in \gG} \ell_{\gD_T}(g\circ f)|)\\
    &\leq \sup_{f\in \gF}  |\inf_{g\in \gG} \ell_{\gD_S}(g\circ f)-\inf_{g\in \gG} \ell_{\gD'}(g\circ f)|+\sup_{f\in \gF}  |\inf_{g\in \gG} \ell_{\gD'}(g\circ f)-\inf_{g\in \gG} \ell_{\gD_T}(g\circ f)|\\
    &= d_{\gG, \gF}(\gD_S,\gD')+d_{\gG, \gF}(\gD',\gD_T).
\end{align}
\end{proof}

{

\begin{proposition}\label{prop:wasserstein}
Denote the function class 
\begin{align}
    \gL_{\gG, \gF}:=\{h_{g, f}:\gX\times \gY\to \sR_+ \mid g\in \gG, f\in \gF\}, \quad \text{where} \ \ h_{g, f}(x, y):=\ell(g\circ f(x), y).
\end{align}
Let $d:\gX\times \gY\to \sR_+$ be a metric on $\gX\times \gY$, and assume $\forall h\in \gL_{\gG, \gF}$ is $L$-Lipschitz continuous with respect to the metric $d$. Then, we have 
\begin{align}
    d_{\gF_A}(\gD_S,\gD_T)\leq L \cdot W(\gD_S, \gD_T),
\end{align}
where $W(\gD_S, \gD_T)$ is the Wasserstein distance:
\begin{align}
    W(\gD_S, \gD_T)=\sup_{\phi:\gX\times \gY \to \sR}\ \E_{(x, y)\sim \gD_S}[\phi(x, y)]-\E_{(x, y)\sim \gD_T}[\phi(x, y)]\quad \text{s.t. $\phi$ is 1-Lipschitz}.  \label{def:wasserstein}
\end{align}
\end{proposition}
} %

\begin{proof}
Recall that
\begin{align}
    d_{\gG,\gF}(\gD_S,\gD_T):=\sup_{f\in \gF}|\inf_{g\in \gG} \ell_{\gD_S}(g\circ f)-\inf_{g\in \gG} \ell_{\gD_T}(g\circ f)|.
\end{align}
By the definition of $\inf$,  for $\forall \epsilon>0$ there exist $g_{S, \epsilon}, g_{T,\epsilon}\in \gG$ such that
\begin{align}
    \inf_{g\in \gG} \ell_{\gD_S}(g\circ f_\epsilon)&\geq \ell_{\gD_S}(g_{S,\epsilon}\circ f_\epsilon)-\epsilon\\
    \inf_{g\in \gG} \ell_{\gD_T}(g\circ f_\epsilon)&\geq \ell_{\gD_T}(g_{T,\epsilon}\circ f_\epsilon)-\epsilon.
\end{align}
By the definition of $\sup$, there exists $f_\epsilon\in \gF$ such that
\begin{align}
    d_{\gG,\gF}(\gD_S,\gD_T) &\leq |\inf_{g\in \gG} \ell_{\gD_S}(g\circ f_\epsilon)-\inf_{g\in \gG} \ell_{\gD_T}(g\circ f_\epsilon)| +\epsilon\\
    &=\max\{\inf_{g\in \gG} \ell_{\gD_S}(g\circ f_\epsilon)-\inf_{g\in \gG} \ell_{\gD_T}(g\circ f_\epsilon),\inf_{g\in \gG} \ell_{\gD_T}(g\circ f_\epsilon)-\inf_{g\in \gG} \ell_{\gD_S}(g\circ f_\epsilon)\}+\epsilon\\
    &\leq \max\{\ell_{\gD_S}(g_{T, \epsilon}\circ f_\epsilon)-\inf_{g\in \gG} \ell_{\gD_T}(g\circ f_\epsilon),\ell_{\gD_T}(g_{S, \epsilon}\circ f_\epsilon)-\inf_{g\in \gG} \ell_{\gD_S}(g\circ f_\epsilon)\}+\epsilon\\
    &\leq \max\{\ell_{\gD_S}(g_{T, \epsilon}\circ f_\epsilon)- \ell_{\gD_T}(g_{T, \epsilon}\circ f_\epsilon),\ell_{\gD_T}(g_{S, \epsilon}\circ f_\epsilon)- \ell_{\gD_S}(g_{S, \epsilon}\circ f_\epsilon)\}+2\epsilon. \label{eq:prop-w-0}
\end{align}
Let's first consider the first term in the $\max\{\cdot, \cdot\}$ above. 
\begin{align}
    \ell_{\gD_S}&(g_{T,\epsilon}\circ f_\epsilon)-\ell_{\gD_T}(g_{T,\epsilon}\circ f_\epsilon) \\
    &=L\cdot(\tfrac{1}{L}\ell_{\gD_S}(g_{T,\epsilon}\circ f_\epsilon)-\tfrac{1}{L}\ell_{\gD_T}(g_{T,\epsilon}\circ f_\epsilon)) \\
    &=L\cdot(\E_{(x,y)\sim \gD_S}[\tfrac{1}{L}\ell(g_{T,\epsilon}\circ f_\epsilon(x), y)]-\E_{(x,y)\sim \gD_T}[\tfrac{1}{L}\ell(g_{T,\epsilon}\circ f_\epsilon(x), y)])\\
    &\leq L \cdot W(\gD_S, \gD_T),
\end{align}
where the inequality is due to that both $\tfrac{1}{L}\ell(g_{S,\epsilon}\circ f_\epsilon(x), y)$ and $\tfrac{1}{L}\ell(g_{T,\epsilon}\circ f_\epsilon(x), y)$ are 1-Lipschitz w.r.t. $(x, y)$ and the metric $d$. 

Similarly, we also have
\begin{align}
    \ell_{\gD_T}&(g_{S,\epsilon}\circ f_\epsilon)-\ell_{\gD_S}(g_{S,\epsilon}\circ f_\epsilon)\leq L \cdot W(\gD_S, \gD_T).
\end{align}
Therefore, \eqref{eq:prop-w-0} implies
\begin{align}
    d_{\gG,\gF}(\gD_S,\gD_T)\leq L\cdot W(\gD_S, \gD_T)+2\epsilon.
\end{align}
Letting $\epsilon\to 0$ completes the proof.

\end{proof}

{

\begin{proposition}\label{prop:d-TV}
Consider multi-class classification where $\gY = [k]$ for some $k \geq 2$. Define the loss function $\ell$ as
\begin{align}
    \ell (g\circ f(x), y) = \1\{\arg\max_{j \in [k]} (g\circ f(x))_j \neq y \}
\end{align}
Let $\delta_{TV}(D_S,D_T)$ denote the total variation distance.
Then we have
\begin{align}
    d_{\gF_A}(\gD_S,\gD_T)\leq \delta_{TV}(\gD_S,\gD_T)
\end{align}
\end{proposition}

} %

\begin{proof}
Fix $f \in \gF$. By the definition of inf, there exists $g_{T,\epsilon}$ such that 
\begin{align}
    &|\inf_{g\in \gG} \ell_{\gD_S}(g\circ f)-\inf_{g\in \gG} \ell_{\gD_T}(g\circ f)| \\
    &\leq |\inf_{g\in \gG} \ell_{\gD_S}(g\circ f)-\ell_{\gD_T}(g_{T,\epsilon} \circ f)| + \eps\\
    &\leq |\ell_{\gD_S}( g_{T,\epsilon} \circ f)-\ell_{\gD_T}(g_{T,\epsilon} \circ f)| + \eps\\
    &= |\E_{(x,y)\sim \gD_S}[\ell(g_{T,\epsilon}\circ f(x), y)]-\E_{(x,y)\sim \gD_T}[\ell(g_{T,\epsilon}\circ f(x), y)]| + \eps\\
    &= |\sP_{(x,y)\sim \gD_S}[\1\{\arg\max_{j \in [k]} (g_{T,\epsilon}\circ f(x))_j \neq y \}]-\sP_{(x,y)\sim \gD_T}[\1\{\arg\max_{j \in [k]} (g_{T,\epsilon}\circ f(x))_j \neq y \}]| + \eps\\ \label{eq:event}
\end{align}
Let $A$ be the event such that $A = \{(x,y): \arg\max_{j \in [k]} (g_{T,\epsilon}\circ f(x))_j \neq y \}$. Then we can write  \eqref{eq:event} as 
\begin{align}
    (\ref{eq:event}) &=  |\sP_{(x,y)\sim \gD_S}[A]-\sP_{(x,y)\sim \gD_T}[A]| + \eps \\
    &\leq \sup_{B}  |\sP_{(x,y)\sim \gD_S}[B]-\sP_{(x,y)\sim \gD_T}[B]| + \eps \\
    &= \delta_{TV}(\gD_S,\gD_T) + \eps 
\end{align}
Send $\eps \rightarrow 0$. Noting that $f \in \gF$ was arbitrary, apply $\sup$ to both sides gives us the desired inequality.
\end{proof}

Theorem~\ref{thm-bound} can be proved easily by definition.
\begin{theorem}[Theorem~\ref{thm-bound} Restated]
Given a training algorithm $A$, for $\forall \gD_S,\gD_T\in \gP_{\gX\times \gY}$ we have
\begin{align}
    \tau(A;\gD_S,\gD_T) &\leq d_{\gF_A}(\gD_S,\gD_T),\\
    \text{or equivalently}, \qquad \inf_{g\in \gG} \ell_{\gD_T}(g\circ f_A^{\gD_S})&\leq \ell_{\gD_S}(g_A^{\gD_S} \circ f_A^{\gD_S}) + d_{\gF_A}(\gD_S,\gD_T).
\end{align}
\end{theorem}
\begin{proof}
By definition, 
\begin{align}
    \tau(A;\gD_S,\gD_T)&=\inf_{g\in \gG} \ell_{\gD_T}(g\circ f_A^{\gD_S})- \ell_{\gD_S}(g_A^{\gD_S} \circ f_A^{\gD_S})\\
    &\leq \inf_{g\in \gG} \ell_{\gD_T}(g\circ f_A^{\gD_S})- \inf_{g\in \gG} \ell_{\gD_S}(g\circ f_A^{\gD_S})\\
    &\leq |\inf_{g\in \gG} \ell_{\gD_T}(g\circ f_A^{\gD_S})- \inf_{g\in \gG} \ell_{\gD_S}(g\circ f_A^{\gD_S})|\\
    &\leq \sup_{f\in \gF_A}|\inf_{g\in \gG} \ell_{\gD_T}(g\circ f)- \inf_{g\in \gG} \ell_{\gD_S}(g\circ f)|\\
    &=d_{\gF_A}(\gD_S,\gD_T).
\end{align}
\end{proof}

To prove Theorem~\ref{thm-tightness}, we first prove the following interesting lemma. 
\begin{lemma}\label{lemma-convex}
Let $\gS_r^{d-1}:=\{y\in \sR^d\mid \|y\|_2=r\}$ denotes the $(d-1)$-dimensional sphere in $\sR^d$ with radius $r>0$. If a function $h:\gS_r^{d-1}\to \sR^d$ satisfies
\begin{align}
     \forall y\in \gS_r^{d-1}:\quad \langle h(y),y \rangle <0, \label{eq:convex-2}
\end{align}
then we have 
\begin{align}
    \vec{0}\in \conv(h(\gS_r^{d-1})),
\end{align}
i.e., $\vec{0}$ is in the convex hull of $\{h(y)\mid y\in \gS_r^{d-1}\}$.
\end{lemma}
\begin{proof}
We assume that $\vec{0}\notin \conv(h(\gS_r^{d-1}))$ and prove by contradiction. 
Since $\vec{0}\notin \conv(h(\gS_r^{d-1}))$, we can find a hyperplane that separates $\vec{0}$ and the convex set $ \conv(h(\gS_r^{d-1}))$. By the separating hyperplane theorem there exists a nonzero vector $\vv\in \sR^d$ and $c\geq 0$ such that
\begin{align}
    \forall \vz\in \conv(h(\gS_r^{d-1})):\quad \langle \vz, \vv\rangle \geq c\geq 0.\label{eq:convex-1}
\end{align}
We choose $y=r\vv/\|\vv\|_2$ and observe that $h(y)\in \conv(h(\gS_r^{d-1}))$. Hence, by \eqref{eq:convex-1} we have
\begin{align}
    \langle h(y), y\rangle \geq 0,
\end{align}
which contradicts to the condition of \eqref{eq:convex-2}.  Therefore, it must be that $\vec{0}\in \conv(h(\gS_r^{d-1}))$.

\end{proof}

We first prove a generalized version of Theorem~\ref{thm-tightness} as shown below, and then we can see that Theorem~\ref{thm-tightness} is exactly the following theorem but with $\epsilon=0$.

\begin{theorem}[Generalized Version of Theorem~\ref{thm-tightness}]\label{thm-tightness-g}
Given any source distribution $\gD_S\in \gP_{\gX\times \sR^d}$, any fine-tuning function class $\gG$ where $\gG$ includes the zero function, and any training algorithm $A$, denote
\begin{align}
    \epsilon:=\ell_{\gD_S}(g_A^{\gD_S} \circ f_A^{\gD_S})- \inf_{g\in \gG, f\in \gF_{A}}\ell_{\gD_S}(g \circ f).
\end{align} 
We assume some properties of the sample individual loss function $\ell:\sR^d\times \sR^d\to \sR_+$: it is differentiable and strictly convex w.r.t. its first argument; $\ell(y,y)=0$ for any $y\in  \sR^d$; and $\lim_{r\to \infty} \inf_{y:\|y\|_2=r} \ell({\vec{0}}, y)=\infty$.
Then, for any distribution $\gD^\gX$ on $\gX$, there exist some distributions $\gD_T\in \gP_{\gX\times \gY}$ with its marginal on $\gX$ being $\gD^\gX$ such that
\begin{align}
    \tau(A;\gD_S,\gD_T)\leq d_{\gF_A}(\gD_S,\gD_T)\leq \tau(A;\gD_S,\gD_T)+\epsilon.
\end{align}
\end{theorem}
\begin{proof}
The $\tau(A;\gD_S,\gD_T)\leq d_{\gF_A}(\gD_S,\gD_T)$ is proved by Theorem~\ref{thm-bound}, we only need to prove that there exists some  $\gD_T\in \gP_{\gX\times \gY}$ with its marginal on $\gX$ being $\gD^\gX$ such that
\begin{align}
    d_{\gF_A}(\gD_S,\gD_T)\leq \tau(A;\gD_S,\gD_T)+\epsilon=\inf_{g\in \gG} \ell_{\gD_T}(g\circ f_A^{\gD_S})-\inf_{g\in \gG, f\in \gF_{A}}\ell_{\gD_S}(g \circ f).
\end{align}
We begin by observing that $\lim_{r\to \infty} \inf_{y:\|y\|_2=r} \ell({\vec{0}}, y)=\infty$, and thus there exists $r>0$ such that
\begin{align}
    \forall y \in \gS_r^{d-1}:\quad\ell(\vec{0},y)\geq \ell_{\gD_S}(\vec{0})=\E_{(\vx,y)\sim \gD_S}[\ell(\vec{0}, y)] , \label{eq:thm-2-3}
\end{align}
where $\gS_r^{d-1}:=\{y\in \sR^d\mid \|y\|_2=r\}$ denotes the $(d-1)$-dimensional sphere with radius $r$. Note the we abuse the notion a bit to let $\vec{0}$ also denotes the zero function (i.e., maps all input to zero).
Now, let us define at the following set
\begin{align}
    \gV:= \{ \nabla_1 \ell(\vec{0}, y) \mid y\in \gS_r^{d-1}\} ,
\end{align}
where $\nabla_1$ is taking the gradient w.r.t. the first argument of $\ell(\cdot, \cdot)$. By the strict convexity of $\ell(\cdot, y)$, we have
\begin{align}
    \ell(y, y)-\ell(\vec{0},y)>\langle  \nabla_1 \ell(\vec{0}, y) , y \rangle.
\end{align}
Noting that $\ell(y, y)=0$ is the unique minimum of $\ell(\cdot, y)$, we have $\ell(\vec{0},y)>0$. Accordingly, 
\begin{align}
   \forall y\in \gS_r^{d-1}:\quad 0 >-\ell(\vec{0},y)>\langle  \nabla_1 \ell(\vec{0}, y) , y \rangle.
\end{align}
Having the above property, and noting that $\nabla_1\ell(\vec{0}, \cdot):\gS_r^{d-1}\to \sR^d$, we can invoke Lemma~\ref{lemma-convex} to see that 
\begin{align}
    \vec{0}\in \conv(\gV).
\end{align}
Therefore, there exists $n$ points $\{y_i\}_{i=1}^n\subset \gS^{d-1}_r$ such that
\begin{align}
    \vec{0}=\sum_{i=1}^n c_i\nabla_1 \ell(\vec{0}, y_i), \label{eq:thm-2-1}
\end{align}
where $c_i>0$ and $\sum_{i=1}^n c_i=1$. 

Therefore, we can define the target distribution $\gD_T$ as the following. Given any $\vx\sim \gD^{\gX}$, the distribution of $y$ conditioned on $\vx$ is: $y=y_i$ with probability $c_i$. Now we verify the distribution $\gD_T$ indeed makes the bound $\epsilon$-tight. Denote a strictly convex function $h:\sR^d\to \sR_+$ as the following
\begin{align}
    h(\cdot):=\sum_{i=1}^n c_i \ell(\cdot, y_i).
\end{align}
Since $h$ is strictly convex and $\nabla h(\vec{0})=\vec{0}$ (\eqref{eq:thm-2-1}), we can see that $h(\vec{0})$ achieves the unique global minimum of $h$ on $\sR^d$.

Therefore, given the $\gD_T$, for any $\forall f\in \gF_A$ we have
\begin{align}
    \inf_{g\in \gG}\ell_{\gD_T}(g\circ f)&=\inf_{g\in \gG} \E_{(\vx,y)\sim \gD_T} [\ell(g\circ f(\vx),y)]\\
    &=\inf_{g\in \gG} \E_{\vx\sim \gD^{\gX}} \left[\sum_{i=1}^n c_i \ell(g\circ f(\vx),y_i) \right]\\
    &=\inf_{g\in \gG} \E_{\vx\sim \gD^{\gX}} \left[h(g\circ f(\vx)) \right]\\
    &=h(\vec{0})\tag{$\gG$ contains the zero function}\\
    &=\sum_{i=1}^n c_i \ell(\vec{0}, y_i). \label{eq:thm-2-4}
\end{align}
Recall that $d_{\gF_A}(\gD_S, \gD_T)=\sup_{f\in \gF_A}|\inf_{g\in \gG} \ell_{\gD_T}(g\circ f)- \inf_{g\in \gG} \ell_{\gD_S}(g\circ f)|$, we can see that
\begin{align}
    d_{\gF_A}(\gD_S, \gD_T)=\sup_{f\in \gF_A}|\sum_{i=1}^n c_i \ell(\vec{0}, y_i)- \inf_{g\in \gG} \ell_{\gD_S}(g\circ f)| \label{eq:thm-2-2}
\end{align}
By \eqref{eq:thm-2-3}, for $\forall f\in \gF_A$, we have
\begin{align}
    \sum_{i=1}^n c_i \ell(\vec{0}, y_i)\geq \ell_{\gD_S}(\vec{0})=\ell_{\gD_S}(\vec{0}\circ f)\geq \inf_{g\in \gG} \ell_{\gD_S}(g\circ f).
\end{align}
Hence, we can continue as
\begin{align}
    (\ref{eq:thm-2-2})&=\sup_{f\in \gF_A}\left(\sum_{i=1}^n c_i \ell(\vec{0}, y_i)- \inf_{g\in \gG} \ell_{\gD_S}(g\circ f)\right)=\sum_{i=1}^n c_i \ell(\vec{0}, y_i)- \inf_{g\in \gG,f\in \gF_A} \ell_{\gD_S}(g\circ f)\\
    &=\inf_{g\in \gG} \ell_{\gD_T}(g\circ f_A^{\gD_S})-\inf_{g\in \gG, f\in \gF_{A}}\ell_{\gD_S}(g \circ f)\tag{by \eqref{eq:thm-2-4}}\\
    &=\inf_{g\in \gG} \ell_{\gD_T}(g\circ f_A^{\gD_S})-\ell_{\gD_S}(g_A^{\gD_S} \circ f_A^{\gD_S})+\ell_{\gD_S}(g_A^{\gD_S} \circ f_A^{\gD_S})-\inf_{g\in \gG, f\in \gF_{A}}\ell_{\gD_S}(g \circ f)\\
    &=\tau(A;\gD_S,\gD_T)+\epsilon.
\end{align}
Therefore, it holds that $d_{\gF_A}(\gD_S,\gD_T)\leq \tau(A;\gD_S,\gD_T)+\epsilon$, and thus the theorem.

\end{proof}

An alternative form of the tightness bound may be derived from the above theorem, and the alternative forms present the tightness result from a different perspective. Noting that the upper bound in Theorem~\ref{thm-bound} states that the following inequality always hold:
$$ \tau(A; \gD_S, \gD_T) - d_{\gF_A} (\gD_S,\gD_T)\leq 0,$$ 
The tightness of this upper bound depends on whether we can derive a lower bound for $\tau(A; \gD_S, \gD_T) - d_{\gF_A} (\gD_S,\gD_T)$. 

We show in the following that, under the same settings as Theorem~\ref{thm-tightness-g},  
\begin{align}
    \inf_{\gD_S, \gD^\gX} \sup_{\gD_T\in \Gamma(\gD^\gX)} (\tau(A; \gD_S, \gD_T) - d_{\gF_A} (\gD_S,\gD_T)) \geq -\epsilon, \label{eq:tightness-equiv}
\end{align}
where $\Gamma(\gD^\gX)$ denotes the distributions whose marginal distribution on $\gX$ is $\gD^\gX$. Theorem~\ref{thm-tightness-g} implies that there exists a $\gD_T$ whose marginal on $\gX$ is $\gD^\gX$ such that $\tau(A; \gD_S, \gD_T) - d_{\gF_A} (\gD_S,\gD_T) \geq -\epsilon$. Therefore, $\sup_{\gD_T\in \Gamma(\gD^\gX)} (\tau(A; \gD_S, \gD_T) - d_{\gF_A} (\gD_S,\gD_T)) \geq -\epsilon$. Then, taking the infimum over all pairs of source distributions and target marginal distribution gives the min-max type of lower bound as shown above. \Eqref{eq:tightness-equiv} is essentially an equivalent statement as the statement of Theorem~\ref{thm-tightness-g}, but its perspective may be more clear and more interesting to some of the readers.  %

\begin{lemma}[Lemma \ref{lemma:emp-finite-dist} Restated]
Assuming the individual loss function $\ell:\gY\times\gY\to [0, c]$, given any distribution $\gD\in \gP_{\gX\times \gY}$ and $\forall \delta>0$, with probability $\geq 1-\delta$ we have
\begin{align}
    d_{\gG, \gF}(\gD, \widehat\gD^n)\leq 2\Rad_{\widehat\gD^n}(\gL_{\gG, \gF})+3c\sqrt{\frac{\ln(4/\delta)}{2n}}.
\end{align}
\end{lemma}
\begin{proof}
Given any $\delta>0$, $f\in \gF, g\in \gG$, $\gD\in \gP_{\gX\times \gY}$, and taking any $h_{g, f}\in \gL_{\gG, \gF}$ (Definition~\ref{def:rad}), with probability $\geq 1-\delta$ we have
\begin{align}
\ell_{\gD}(g\circ f)-\ell_{\widehat\gD^n}(g\circ f) &= \E_{(x, y)\sim \gD}[h_{g, f}(x, y)]-\frac{1}{n}\sum_{i=1}^n h_{g,f}(x_i, y_i)\\
&\leq 2\Rad_{\widehat\gD^n}(\gL_{\gG, \gF})+3c\sqrt{\frac{\ln(2/\delta)}{2n}},\label{eq:emp-finite-dist-1}
\end{align}
where the inequality is by the well-known Rademacher complexity uniform bound. Similarly,
\begin{align}
    \ell_{\widehat\gD^n}(g\circ f) - \ell_{\gD}(g\circ f)&= \E_{(x, y)\sim \gD}[-h_{g, f}(x, y)]-\frac{1}{n}\sum_{i=1}^n -h_{g,f}(x_i, y_i)\\
&\leq 2\Rad_{\widehat\gD^n}(-\gL_{\gG, \gF})+3c\sqrt{\frac{\ln(2/\delta)}{2n}}\\
&=2\Rad_{\widehat\gD^n}(\gL_{\gG, \gF})+3c\sqrt{\frac{\ln(2/\delta)}{2n}}. \label{eq:emp-finite-dist-2}
\end{align}
The probability that both events \eqref{eq:emp-finite-dist-1} and \eqref{eq:emp-finite-dist-2} happen can be upper bounded by union bound, i.e.,
\begin{align}
    \Pr((\ref{eq:emp-finite-dist-1})\wedge(\ref{eq:emp-finite-dist-2}))=1-\Pr((\ref{eq:emp-finite-dist-1})^c\vee(\ref{eq:emp-finite-dist-2})^c)\geq 1-(\Pr((\ref{eq:emp-finite-dist-1})^c)+\Pr((\ref{eq:emp-finite-dist-2})^c))\geq 1-2\delta.
\end{align}
Therefore, combining the above with probability $\geq 1-\delta$ we have
\begin{align}
    |\ell_{\gD}(g\circ f)-\ell_{\widehat\gD^n}(g\circ f)|\leq 2\Rad_{\widehat\gD^n}(\gL_{\gG, \gF})+3c\sqrt{\frac{\ln(4/\delta)}{2n}}. \label{eq:emp-finite-dist-3}
\end{align}
With \eqref{eq:emp-finite-dist-3}, we can prove the lemma as the following. Given $\forall \epsilon>0$, by the definition of infimum there exists a $g_\epsilon\in \gG$ such that
\begin{align}
    \ell_{\gD}(g_\epsilon\circ f)<\inf_{g\in\gG}\ell_{\gD}(g\circ f)+\epsilon.
\end{align}
By \eqref{eq:emp-finite-dist-3},  with probability $\geq 1-\delta$ we have
\begin{align}
    \ell_{\widehat\gD^n}(g_\epsilon\circ f)\leq \ell_{\gD}(g_\epsilon\circ f)+2\Rad_{\widehat\gD^n}(\gL_{\gG, \gF})+3c\sqrt{\frac{\ln(4/\delta)}{2n}}.
\end{align}
Moreover, by definition
\begin{align}
    \inf_{g\in \gG}\ell_{\widehat\gD^n}(g\circ f)\leq \ell_{\widehat\gD^n}(g_\epsilon\circ f).
\end{align} 
Combining the above three inequalities we have
\begin{align}
    \inf_{g\in \gG}\ell_{\widehat\gD^n}(g\circ f)< \inf_{g\in\gG}\ell_{\gD}(g\circ f)+\epsilon+2\Rad_{\widehat\gD^n}(\gL_{\gG, \gF})+3c\sqrt{\frac{\ln(4/\delta)}{2n}}.
\end{align}
Letting $\epsilon\to 0$, we can see that
\begin{align}
    \inf_{g\in \gG}\ell_{\widehat\gD^n}(g\circ f)\leq \inf_{g\in\gG}\ell_{\gD}(g\circ f)+2\Rad_{\widehat\gD^n}(\gL_{\gG, \gF})+3c\sqrt{\frac{\ln(4/\delta)}{2n}}.
\end{align}
Similarly, we can derive the above inequality again but with $\gD$ and $\widehat\gD^n$ switched. Therefore, 
\begin{align}
    |\inf_{g\in \gG}\ell_{\widehat\gD^n}(g\circ f)- \inf_{g\in\gG}\ell_{\gD}(g\circ f)|\leq 2\Rad_{\widehat\gD^n}(\gL_{\gG, \gF})+3c\sqrt{\frac{\ln(4/\delta)}{2n}}.
\end{align}
Since the above inequality holds for $\forall f\in \gF$, taking the supremum over $f\in \gF$ gives the lemma.
\end{proof}

\begin{lemma}\label{lemma:emp-dist}
Assuming the individual loss function $\ell:\gY\times\gY\to [0, c]$, given any distributions $\gD_S, \gD_T\in \gP_{\gX\times \gY}$ and $\forall \delta>0$, with probability $\geq 1-\delta$ we have
\begin{align}
    d_{\gF_A}(\gD_S, \gD_T)\leq d_{\gF_A}(\widehat\gD^n_S, \widehat\gD^n_T) + 2(\Rad_{\widehat\gD_S^n}(\gL_{\gG, \gF_A})+\Rad_{\widehat\gD_T^n}(\gL_{\gG, \gF_A}))+6c\sqrt{\frac{\ln(8/\delta)}{2n}}.
\end{align}
\end{lemma}
\begin{proof}
By Proposition~\ref{prop:pseudometric}, we apply the triangle inequality to derive
\begin{align}
    d_{\gF_A}(\gD_S, \gD_T)&\leq d_{\gF_A}(\gD_S, \widehat\gD^n_T)+d_{\gF_A}(\widehat\gD^n_T, \gD_T)\\
    &\leq d_{\gF_A}( \widehat\gD^n_S, \widehat\gD^n_T)+d_{\gF_A}(\widehat\gD^n_T, \gD_T)+d_{\gF_A}(\widehat\gD^n_S, \gD_S).
\end{align}
By Lemma~\ref{lemma:emp-finite-dist}, we can apply the union bound argument (e.g., see the proof of Lemma~\ref{lemma:emp-finite-dist}) to bound $d_{\gF_A}(\widehat\gD^n_T, \gD_T)$ and $d_{\gF_A}(\widehat\gD^n_S, \gD_S)$. That being said, $\forall \delta'>0$ with probability $\geq 1-2\delta'$ we have
\begin{align}
    d_{\gF_A}(\widehat\gD^n_S, \gD_S)\leq 2\Rad_{\widehat\gD_S^n}(\gL_{\gG, \gF_A})+3c\sqrt{\frac{\ln(4/\delta')}{2n}}\\
    d_{\gF_A}(\widehat\gD^n_T, \gD_T)\leq 2\Rad_{\widehat\gD_T^n}(\gL_{\gG, \gF_A})+3c\sqrt{\frac{\ln(4/\delta')}{2n}}.
\end{align}
Therefore,
\begin{align}
    d_{\gF_A}(\widehat\gD^n_T, \gD_T)+d_{\gF_A}(\widehat\gD^n_S, \gD_S)\leq 2(\Rad_{\widehat\gD_S^n}(\gL_{\gG, \gF_A})+\Rad_{\widehat\gD_T^n}(\gL_{\gG, \gF_A}))+6c\sqrt{\frac{\ln(4/\delta')}{2n}}.
\end{align}
Denoting $\delta=2\delta'$ gives the lemma.
\end{proof}

\begin{theorem}[Theorem~\ref{thm:emp-bound} Restated]
Given $\forall \gD_S, \gD_T\in \gP_{\gX\times \gY}$, for $\forall \delta>0$ with probability $\geq 1-\delta$ we have
\begin{align}
    \tau(A; \widehat\gD^n_S, \gD_T)\leq d_{\gF_A}(\widehat\gD^n_S, \widehat\gD^n_T) +2\Rad_{\widehat\gD_T^n}(\gL_{\gG, \gF_A})+ 4\Rad_{\widehat\gD_S^n}(\gL_{\gG, \gF_A})+9c\sqrt{\frac{\ln(8/\delta)}{2n}}.
\end{align}
\end{theorem}
\begin{proof}
For $\forall \delta>0$, from the proof of Lemma~\ref{lemma:emp-dist} we can see that with probability $\geq 1-\delta$:
\begin{align}
    d_{\gF_A}(\widehat\gD^n_S, \gD_S)\leq 2\Rad_{\widehat\gD_S^n}(\gL_{\gG, \gF_A})+3c\sqrt{\frac{\ln(8/\delta)}{2n}}\label{eq:thm-emp-bound-1}\\
    d_{\gF_A}(\widehat\gD^n_T, \gD_T)\leq 2\Rad_{\widehat\gD_T^n}(\gL_{\gG, \gF_A})+3c\sqrt{\frac{\ln(8/\delta)}{2n}},
\end{align}
and Lemma~\ref{lemma:emp-dist} holds. Therefore
\begin{align}
    \tau(A; \widehat\gD^n_S, \gD_T)&=\inf_{g\in \gG} \ell_{\gD_T}(g\circ f_A^{\widehat\gD^n_S})- \ell_{\widehat\gD^n_S}(g_A^{\widehat\gD^n_S} \circ f_A^{\widehat\gD^n_S})\\
    &\leq \inf_{g\in \gG} \ell_{\gD_T}(g\circ f_A^{\widehat\gD^n_S})- \inf_{g\in \gG}\ell_{\widehat\gD^n_S}(g \circ f_A^{\widehat\gD^n_S})\\
    &=\inf_{g\in \gG} \ell_{\gD_T}(g\circ f_A^{\widehat\gD^n_S})- \inf_{g\in \gG}\ell_{\gD_S}(g \circ f_A^{\widehat\gD^n_S})+\inf_{g\in \gG}\ell_{\gD_S}(g \circ f_A^{\widehat\gD^n_S})-\inf_{g\in \gG}\ell_{\widehat\gD^n_S}(g \circ f_A^{\widehat\gD^n_S})\\
    &\leq d_{\gF_A}(\gD_S, \gD_T)+d_{\gF_A}(\widehat\gD^n_S, \gD_S)\\
    &\leq d_{\gF_A}(\gD_S, \gD_T)+ 2\Rad_{\widehat\gD_S^n}(\gL_{\gG, \gF_A})+3c\sqrt{\frac{\ln(8/\delta)}{2n}}\\
    &\leq d_{\gF_A}(\widehat\gD^n_S, \widehat\gD^n_T) +2\Rad_{\widehat\gD_T^n}(\gL_{\gG, \gF_A})+ 4\Rad_{\widehat\gD_S^n}(\gL_{\gG, \gF_A})+9c\sqrt{\frac{\ln(8/\delta)}{2n}},
\end{align}
where the first inequality is by definition of infimum, the second inequality is by the Definition~\ref{def:pseudometric}, the third inequality is by \eqref{eq:thm-emp-bound-1} and the last inequality is by Lemma~\ref{lemma:emp-dist}.
\end{proof}

\section{Data augmentation (DA) as Regularization}\label{sec:da}

In this section, we discuss data augmentation  (DA) as a concrete example of regularization for training  feature extractor $f_A^{\gD_S}$, and explore its impact on the function class $\gF_A$ discussed in {Section~\ref{sec:analysis}}.

Empirical research has shown evidence of the regularization effect of DA~\citep{hernandez2018data, hernandez2018further}. However, there is a lack of theoretical analysis, and thus we aim to construct a theoretical framework to understand under what sufficient conditions  DA can be viewed as regularization on the feature extractor function class $\gF_A$. We categorize DA into \underline{\textit{feature-level DA}} and \underline{\textit{data-level DA}}, and for each category, we analyze different DA algorithms to characterize the sufficient conditions under which DA regularizes the function class $\gF_A$. Combined with analysis in {Theorem~\ref{thm:emp-bound}}, we also provide concrete sufficient conditions to tighten the upper bound of  relative transferability $\tau(A;\widehat{\gD}^n,\gD_T)$.

\textbf{General Settings.} For the following discussion  we apply a general DA setting of affine transformation~\citep{perez2017effectiveness}, taking the form of $ x_\star= W_\star^\top x + b_\star$, where $(x,x_\star)$ is a pair of the original  and  augmented samples, $(W_\star,b_\star)$ are parameters representing  specific DA policies. 
We set $g:\mathbb{R}^d\rightarrow \mathbb{R}$ as the linear layer corresponding to the weight matrix $W_g$, which will be composed with the feature extractor $f:\mathbb{R}^m\rightarrow \mathbb{R}^d$. We use squared loss for $\ell:\mathbb{R}\times \mathbb{R} \rightarrow \mathbb{R}$, and let $\ell_{\widehat{\gD}^n, A}(g\circ f)$  be the objective function given by  training algorithm $A$ from {Theorem~\ref{thm:emp-bound}}.

\subsection{Feature-level DA ($A^{FL}$)}
Feature-level DA~\citep{wong2016understanding, devries2017dataset} requires the transformation to be performed in the learned feature space, which gives us an augmented feature $W_\star f(x)+b_\star$. 
We use \underline{Loss-Averaging} algorithm where we take an average of the loss
over augmented features for training. Denote the training algorithm based on feature-level DA as $A^{FL}$, the objective function is as below.
\begin{align}
\label{eq:FLLA}
  \ell_{\widehat{\gD}^n, A^{FL}}(g\circ f)  = \frac{1}{n}\sum_{i=1}^{n}\E_{W_\star, b_\star}\ell\Big(g\circ \big(W_\star f(x_i)+ b_\star\big),y_i\Big).
\end{align}

\begin{theorem}\label{apdx:FLDA}
Apply feature-level DA with affine transformation parameters $(W_\star,b_\star)$ s.t. 1) $\E_{W_\star}[W_\star] = \mathbb{I}_m$; 2) $W_\star\not\equiv \mathbb{I}_m$ (i.e., $W_\star$ is not an identity matrix); 3) $\E_{b_\star}[b_\star] = \vec{0}_m$; 4) $W_\star$ and $b_\star$ are independent. Set $\ell:\mathbb{R}\times \mathbb{R} \rightarrow \mathbb{R}$ as squared loss; Define $\Delta_{W_\star} := W_\star -\mathbb{I}_m$, then we have
\begin{align}
    \ell_{\widehat{\gD}^n, A^{FL}}(g\circ f) = \ell_{\widehat{\gD}^n, A}(g\circ f) + \Omega_{A^{FL}},
\end{align}
where $\Omega_{A^{FL}} = \frac{1}{n}\sum_{i=1}^{n} \E_{W_\star}\Big[ \big\|f(x_i)^\top\Delta_{W_\star}^{}W_g\big\|_2^2\Big] +  \E_{b^{}_\star}\Big[ \big\|b_\star^\top W_g\big\|_2^2\Big]$.
\end{theorem}

\begin{proof}
$\ell''(W_g^\top\circ f(x_i)) = 2$ for $\ell$ as squared loss. Apply Taylor expansion to $\ell\Big(g\circ \big(W_\star f(x_i)+ b_\star\big),y_i\Big)$ around $f(x_i)$, all higher-than-two order terms will vanish:
\begin{align}
    & \E_{W_\star, b_\star}\bigg[\ell\Big(g\circ \big(W_\star f(x_i)+ b_\star\big),y_i\Big)\bigg]\\
    = &\E_{W_\star, b_\star}\bigg[\ell\Big(W_g^\top\circ f(x_i),y_i\Big)+ W_g^\top(\Delta_{W_\star}f(x_i)+b_\star)\ell'(W_g^\top\circ f(x_i),y_i) +\\
    &\frac{1}{2} W_g^\top(\Delta_{W_\star}f(x_i)+b_\star)(\Delta_{W_\star}f(x_i)+b_\star)^\top\ell''(W_g^\top\circ f(x_i),y_i) W_g\bigg]\\
    = &\ell\Big(W_g^\top\circ f(x_i),y_i\Big) + \E_{W_\star,b_\star} \Big[W_g^\top(\Delta_{W_\star}f(x_i)+b_\star)(\Delta_{W_\star}f(x_i)+b_\star)^\top W_g\Big]\\
    =& \ell\Big(W_g^\top\circ f(x_i),y_i\Big) + \E_{W_\star}\Big[ \big\|f(x_i)^\top\Delta_{W_\star}^{}W_g\big\|_2^2\Big] +  \E_{b^{}_\star}\Big[ \big\|b_\star^\top W_g\big\|_2^2\Big];
\end{align}
The second equality holds since $\E_{\Delta_{W_\star}} =\E_{W_\star}[W_\star -\mathbb{I}_m] = 0_{(m,m)}$ and $\E_{b_\star} = \vec{0}_{m}$; The third equality holds since $W_i$ and $b_i$ are independent.
Therefore, $\ell_{\widehat{\gD}^n, A^{FL}}(g\circ f)
    := \frac{1}{n}\sum_{i=1}^{n}\bigg[\E_{W_\star, b_\star}\ell\Big(g\circ \big(W_\star f(x_i)+ b_\star\big),y_i\Big)\bigg]
    = \ell_{\widehat{\gD}^n, A}(g\circ f) + \Omega_{A^{FL}}$.
\end{proof}

\textbf{\textit{Interpretation}}. $\Omega_{A^{FL}}$ is composed of two segments: 1) $l_2$ regularization to an $f$-dependent scalar averaged over $W_\star$ and $x_i$; 2) $l_2$ regularization to an $f$-independent scalar averaged over $b_\star$. Due to the regularization effect on $f$ from the first segment of $\Omega_{A^{FL}}$, we can reasonably expect the function class $\gF_{A'}$ enabled by $A^{FL}$ to be a subset of that enabled by a general training algorithm $A$.

\textbf{\textit{Sufficient conditions}}. Combined with {Theorem~\ref{thm:emp-bound}}, the \textit{sufficient conditions} to tighten the upper bound $d_{\gF_A}(\widehat{\gD}_S^n,\widehat{\gD}_T^n)$ for the relative transferability $\tau(A;\widehat{\gD}_S,\gD_T)$ are: feature-level DA ($A^{FL}$) with parameters satisfying: 
1) $\E_{W_\star}[W_\star] = \mathbb{I}_m$; 
2) $W_\star\not\equiv \mathbb{I}_m$; 
3) $\E_{b_\star}[b_\star] = \vec{0}_m$; 
4) $W_\star$ and $b_\star$ are independent. 

\subsection{Data-level DA ($A^{DL}$)}
Data-level DA requires that the transformation to be performed in the input space to generate augmented samples $W_\star x + b_\star$. We cover analysis on two ubiquitous algorithms for data-level DA training: \underline{\textit{Prediction-Averaging ($A_P^{DL}$)}}~\citep{lyle2019analysis} and \underline{\textit{Loss-Averaging ($A_L^{DL}$)}}~\citep{wong2016understanding}. The difference between $A_P^{DL}$ and $A_L^{DL}$ lies in whether we take the average of the prediction or the losses: 
    \begin{align}
    \label{eq:PA_def}
    \ell_{\widehat{\gD}^n,A_P^{DL}}(g\circ f) &:= \frac{1}{n}\sum_{i=1}^n \ell\big(\E_{W_\star,b_\star}\big[g \circ f(W_\star x_i +b_\star)\big],y_i\big); \\
    \label{eq:LA_def}
        \ell_{\widehat{\gD}^n, A_L^{DL}}(g\circ f)
        &:= \frac{1}{n}\sum_{i=1}^n\E_{W_\star,b_\star}\big[\ell\big(g\circ {f}(W_\star x_i +b_\star),y_i\big)\big].
    \end{align}

\begin{theorem}\label{apdx:DLDA} Define the data-level deviation caused by data-level DA $A^{DL}\in\{A_P^{DL},A_L^{DL}\}$ with parameters $(W_\star,b_\star)$ from the original data sample as $\Delta_{x_i}:=(W_\star-\mathbb{I}_m)x_i +b_\star$, and  define $\Delta_{x}^3:=\E_{x_i,W_\star,b_\star}\Big[\big\|\Delta_{x_i}\big\|_2^3\Big]$.
Suppose  we apply data-level DA s.t. 
1) $\E_{W_\star}[W_\star] = \mathbb{I}_m$;
2) $\E_{b_\star}[b_\star] = \vec{0}_m$; 
3) $\mathcal{O}(\Delta_{x}^j)\approx 0,\forall j\in \mathbb{N}_+, j\geq 3$; 
4) $W_\star$ and $b_\star$ are independent.
Define $\Delta_{W_\star} := W_\star -\mathbb{I}_m \in \mathbb{R}^{m\times m}$, $\Delta_{\widehat{y}_i} := W_g^\top f(x_i) - y_i \in\mathbb{R}$. Let $W_g^{(k)}\in \mathbb{R}$ be the $k^{th}$ dimension component  of $W_g$ and then define $w_{i,(k)}:=W_g^{(k)}\Delta_{\widehat{y}_i}\in \mathbb{R}$;
Denote the Hessian matrix of the $k^{th}$ dimension component in $f(x_i)$ as $\mathcal{H}_f^{(k),i}$; 
Let $\nabla f$ be the Jacobian matrix of $f$, then we have
\begin{align}
\ell_{\widehat{\gD}^n, A_{}^{DL}}(g\circ f) = \ell_{\widehat{\gD}^n, A}(g\circ f) + \Omega_{A_{}^{DL}} + \mathcal{O}(\Delta_x^3),
\end{align}
where 
$
\Omega_{A_P^{DL}} 
    = \frac{1}{n}\sum_{i=1}^n\sum_{k=1}^d
    w_{i,(k)}\Big[tr\Big(\E_{W_\star}[\Delta_{x_i}\Delta_{x_i}^\top ] \mathcal{H}_f^{(k),i}\Big) \Big]$, where $\Delta_{x_i} = (W_\star-\mathbb{I})^\top x_i+b_\star
$; $
\Omega_{A_L^{DL}} = \Omega_{A_P^{DL}} + \frac{1}{n}\sum_{i=1}^n \Big[\E_{W_\star} \big\| x_i^\top\Delta_{W_\star}\nabla f(x_i) W_g\big\|_2^2 + \E_{b_\star} \big\| b_\star^\top\nabla f(x_i) W_g\big\|_2^2 \Big]. 
$
\end{theorem}

\begin{proof}
Let $\Delta_{f_i,A_P^{DL}} := \E_{W_\star,b_\star}f(W_\star^\top x_i+b_\star) - f(x_i)$, then
\begin{align}
    \Delta_{f_i,A_P^{DL}} :=& \E_{W_\star,b_\star}f(W_\star^\top x_i+b_\star) - f(x_i)\\
    =& \E_{W_\star,b_\star}\Big[\nabla f(x_i)^\top (\Delta_{x_i})\Big] + 
    \frac{1}{2}\E_{W_\star,b_\star} \Big[\Delta_{x_i}^\top\mathcal{H}_f^{(k),i}(x_i)\Delta_{x_i}\Big]_d + \mathcal{O}(\E_{W_\star,b_\star}\|\Delta_{x_i}\|_2^3)\\
    =&\frac{1}{2}\E_{W_\star,b_\star} \Big[\Delta_{x_i}^\top\mathcal{H}_f^{(k),i}(x_i)\Delta_{x_i}\Big]_d + \mathcal{O}(\E_{W_\star,b_\star}\|\Delta_{x_i}\|_2^3),\label{eq:dfiP}
\end{align}
where $[\cdot^{(k)}]_d$ denotes a d-dimensional vector and $k$ denotes the $k^{th}$ dimension element. Since $\ell$ is squared loss, the third-and-higher derivative are 0, therefore, the third-and-higher order terms in Taylor expansion to $\ell\big(\E_{W_\star,b_\star}\big[g \circ f(W_\star x_i +b_\star)\big],y_i\big)$ around $f(x_i)$ will vanish:
\begin{align}
  &\ell\big(\E_{W_\star,b_\star}\big[g \circ f(W_\star x_i +b_\star)\big],y_i\big) \\
  =& \ell\big(g \circ f( x_i),y_i\big) + W_g^\top(\Delta_{f_i,A_{P}^{DL}})\ell'\big(g \circ f( x_i),y_i\big) + \\
  &\frac{1}{2}W_g^\top(\Delta_{f_i,A_{P}^{DL}})(\Delta_{f_i,A_{P}^{DL}})^\top W_g\ell''\big(g \circ f( x_i),y_i\big)\\
  = & \ell\big(g \circ f( x_i),y_i\big) + W_g^\top(\Delta_{f_i,A_{P}^{DL}})\ell'\big(g \circ f( x_i),y_i\big) + \mathcal{O}(\E_{W_\star,b_\star}\|\Delta_{x_i}\|_2^4)\label{eq:DLP0}
\end{align}

Substitute Eq. (\ref{eq:dfiP}) into the first-order term in Eq. (\ref{eq:DLP0}), we have 
\begin{align}
    W_g^\top(\Delta_{f_i,A_{P}^{DL}})\ell'\big(g \circ f( x_i),y_i\big) 
    = &W_g^\top\E_{W_\star,b_\star} \Big[\Delta_{x_i}^\top\mathcal{H}_f^{(k),i}\Delta_{x_i}\Big]_d\Delta_{\widehat{y}_i} +\mathcal{O}(\E_{W_\star,b_\star}\|\Delta_{x_i}\|_2^3)\\
    =& \Delta_{\widehat{y}_i}\sum_{k=1}^{d}W_g^{(k)}\E_{W_\star,b_\star} \Big[\Delta_{x_i}^\top\mathcal{H}_f^{(k),i}\Delta_{x_i}\Big] +\mathcal{O}(\E_{W_\star,b_\star}\|\Delta_{x_i}\|_2^3)\\
    =&\sum_{k=1}^{d}w_{i,(k)} tr\big(\E_{W_\star,b_\star}[\Delta_{x_i}\Delta_{x_i}^\top]\mathcal{H}_f^{(k),i}\big) +\mathcal{O}(\E_{W_\star,b_\star}\|\Delta_{x_i}\|_2^3). \label{eq:DLP01}
\end{align}

Substitute Eq. (\ref{eq:DLP01}) into  Eq. (\ref{eq:DLP0}), we have
\begin{align}
  \ell\big(\E_{W_\star,b_\star}\big[g \circ f(W_\star x_i +b_\star)\big],y_i\big) 
   =& \ell\big(g \circ f( x_i),y_i\big) + \sum_{k=1}^{d}w_{i,(k)} tr\big(\E_{W_\star,b_\star}[\Delta_{x_i}\Delta_{x_i}^\top]\mathcal{H}_f^{(k),i}\big) +\\ &\mathcal{O}(\E_{W_\star,b_\star}\|\Delta_{x_i}\|_2^3).\label{eq:DLP03}
\end{align}

Substitute Eq. (\ref{eq:DLP03}) into Eq. (\ref{eq:PA_def}) which is the definition of $\ell_{\widehat{\gD}^n, A_{P}^{DL}}(g\circ f)$, and recall that $\Delta_{x}^3:=\E_{x_i,W_\star,b_\star}\Big[\big\|\Delta_{x_i}\big\|_2^3\Big]$, we have
\begin{align}
\ell_{\widehat{\gD}^n, A_{P}^{DL}}(g\circ f) := \frac{1}{n}\sum_{i=1}^n \ell\big(\E_{W_\star,b_\star}\big[g \circ f(W_\star x_i +b_\star)\big],y_i\big) =   \ell_{\widehat{\gD}^n, A}(g\circ f) + \Omega_{A_{P}^{DL}} + \mathcal{O}(\Delta_{x}^3).\label{eq:DLPfin}
\end{align}

Let $\Delta_{f_i,A_L^{DL}} := f(W_\star^\top x_i+b_\star) - f(x_i)
    = \nabla f(x_i)^\top (\Delta_{W_\star}x_i+b_\star) + \mathcal{O}(\|\Delta_{x_i}\|_2^2)$.

Applying Taylor expansion to $\E_{W_\star,b_\star}\big[\ell\big(g \circ f(W_\star x_i +b_\star),y_i\big)\big]$ around $f(x_i)$ will give us
\begin{align}
  \E_{W_\star,b_\star}\big[\ell\big(g \circ f(W_\star x_i +b_\star),y_i\big)\big]
  =& \ell\big(g \circ f( x_i),y_i\big) + W_g^\top\E_{W_\star,b_\star}\big[\Delta_{f_i,A_{L}^{DL}}\big]\ell'\big(g \circ f( x_i),y_i\big) + \\
  &\frac{1}{2}W_g^\top\E_{W_\star,b_\star}\big[(\Delta_{f_i,A_{L}^{DL}})(\Delta_{f_i,A_{L}^{DL}})^\top\big] W_g\ell''\big(g \circ f( x_i),y_i\big)\label{eq:DLL0}
\end{align}

Since $\E_{W_\star,b_\star}\Delta_{f_i,A_{L}^{DL}} = \Delta_{f_i,A_{P}^{DL}}$, the first-order term in Eq. (\ref{eq:DLL0}) is exactly Eq. (\ref{eq:DLP01}):
\begin{align}
    &W_g^\top\E_{W_\star,b_\star}\big[\Delta_{f_i,A_{L}^{DL}}\big]\ell'\big(g \circ f( x_i),y_i\big) \\
    =& W_g^\top\Delta_{f_i,A_{P}^{DL}}\ell'\big(g \circ f( x_i),y_i\big)\\
    =&\sum_{k=1}^{d}w_{i,(k)} tr\big(\E_{W_\star,b_\star}[\Delta_{x_i}\Delta_{x_i}^\top]\mathcal{H}_f^{(k),i}\big)+\mathcal{O}(\E_{W_\star,b_\star}\|\Delta_{x_i}\|_2^3)\label{eq:DLL1}
\end{align}
The second-order term in Eq. (\ref{eq:DLL0}) is
\begin{align}
&\frac{1}{2}W_g^\top\E_{W_\star,b_\star}\big[(\Delta_{f_i,A_{L}^{DL}})(\Delta_{f_i,A_{L}^{DL}})^\top\big] W_g\ell''\big(g \circ f( x_i),y_i\big)\\
    = &W_g^\top \E_{W_\star,b_\star}\big[(\nabla f(x_i)^\top (\Delta_{W_\star}x_i+b_\star) (\Delta_{W_\star}x_i+b_\star)^\top \nabla f(x_i)\big]W_g +\mathcal{O}(\E_{W_\star,b_\star}\|\Delta_{x_i}\|_2^4)\\
    =&\E_{W_\star} \big\| x_i^\top\Delta_{W_\star}^\top\nabla f(x_i) W_g\big\|_2^2 + \E_{b_\star} \big\| b_\star^\top\nabla f(x_i) W_g\big\|_2^2 +\mathcal{O}(\E_{W_\star,b_\star}\|\Delta_{x_i}\|_2^4)\label{eq:DLL2}
\end{align}

Substituting Eq. (\ref{eq:DLL1}) and  Eq. (\ref{eq:DLL2}) into  Eq. (\ref{eq:DLL0}), we have
\begin{align}
  &\E_{W_\star,b_\star}\big[\ell\big(g \circ f(W_\star x_i +b_\star),y_i\big)\big] \\
   =& \ell\big(g \circ f( x_i),y_i\big) + 
   \sum_{k=1}^{d}w_{i,(k)} tr\big(\E_{W_\star,b_\star}[\Delta_{x_i}\Delta_{x_i}^\top]\mathcal{H}_f^{(k),i}\big) + \\
   &\E_{W_\star} \big\| x_i^\top\Delta_{W_\star}\nabla f(x_i) W_g\big\|_2^2 + \E_{b_\star} \big\| b_\star^\top\nabla f(x_i) W_g\big\|_2^2+\mathcal{O}(\E_{W_\star,b_\star}\|\Delta_{x_i}\|_2^4)\label{eq:DLL3}
\end{align}

Substitute Eq. (\ref{eq:DLL3}) into the definition of $\ell_{\widehat{\gD}^n, A_L^{DL}}(g\circ f)$, then
\begin{align}
\ell_{\widehat{\gD}^n, A_L^{DL}}(g\circ f)
:= \frac{1}{n}\sum_{i=1}^n\E_{W_\star,b_\star}\big[\ell\big(g\circ {f}(W_\star x_i +b_\star),y_i\big)\big] =\ell_{\widehat{\gD}^n, A}(g\circ f) + \Omega_{A_{L}^{DL}} +\mathcal{O}(\Delta_{x}^3)  \label{eq:DLLfin}
\end{align}

The proof is complete by  Eq. (\ref{eq:DLPfin}) and Eq. (\ref{eq:DLLfin}).
\end{proof}

\textbf{\textit{Interpretation}}. $\Omega_{A_P^{DL}}$ and $\Omega_{A_L^{DL}}$ turn out to be:
1) $\Omega_{A_P^{DL}}$ is a weighted trace expectation dependent on the Hessian matrix of $f$; 
2) $\Omega_{A_L^{DL}}$ is equivalent to $\Omega_{A_P^{DL}}$ together with the summation of two norm expectations dependent on $\nabla f$. Therefore, the data-level DA algorithms $A_P^{DL}$ and $A_L^{DL}$ are expected to regularize $f$ so that the $f$ function class $\gF_{A}^{DL}$ enabled by $A^{DL}\in\{A_P^{DL},A_L^{DL}\}$ would be reasonably expected as a subset of $\gF_A$ enabled by general training algorithm $A$.

\textbf{\textit{Sufficient conditions}}. 
Combined with {Theorem~\ref{thm:emp-bound}}, the \textit{sufficient conditions} indicated here to tighten the upper bound
$d_{\gF_A}(\widehat{\gD}_S^n,\widehat{\gD}_T^n)$
of the relative transferability
$\tau(A;\widehat{\gD}_S,\gD_T)$ 
are: data-level DA ($A^{DL}$) with DA parameters satisfying that 
1) $\E_{W_\star}[W_\star] = \mathbb{I}_m$;
2) $\E_{b_\star}[b_\star] = \vec{0}_m$; 
3) $\mathcal{O}(\Delta_{x}^j)\approx 0,\forall j\in \mathbb{N}_+, j\geq 3$;
4) $W_\star$ and $b_\star$ are independent.

\textbf{\textit{Empirical verification}}. We further provide empirical verification  in Section~\ref{sec:exp} for the sufficient conditions above, investigating the concrete cases of DA methods:
\textbf{\textit{1) Gaussian noise}} \underline{satisfies} the sufficient conditions, then we empirically show that Gaussian noise improves domain transferability while robustness decreases a bit (Figure~\ref{fig:result-da}); \textbf{\textit{2) Rotation}}, which rotates input image with a predefined fixed angle with predefined fixed probability, \underline{violates} $\E_{W_\star}[W_\star] = \mathbb{I}_m$, and we empirically show that rotation barely affect domain transferability (Figure~\ref{fig:app-result-rotate} in Appendix~\ref{sec:app-badda}); 
\textbf{\textit{Translation}}, which moves the input image for a predefined  distance along a pre-selected axis with  fixed probability, \underline{violates} $\E_{b_\star}[b_\star] = \vec{0}_m$ (Figure~\ref{fig:app-result-rotate} in Appendix~\ref{sec:app-badda}).

\begin{corollary}
\label{apdx:DLDA2}
If the neural network in Theorem \ref{apdx:DLDA} is activated by Relu or Max-pooling, then Theorem \ref{apdx:DLDA} becomes
\begin{align}
\ell_{\widehat{\gD}^n, A_{}^{DL}}(g\circ f) = \ell_{\widehat{\gD}^n, A}(g\circ f) + \Omega_{A_{}^{DL}} + \mathcal{O}(\Delta_x^3),
\end{align}
where 
$\Omega_{A_P^{DL}}  =0$; $ \Omega_{A_L^{DL}} = \frac{1}{n}\sum_{i=1}^n \Big[\E_{W_\star} \big\| x_i^\top\Delta_{W_\star}\nabla f(x_i) W_g\big\|_2^2 + \E_{b_\star} \big\| b_\star^\top\nabla f(x_i) W_g\big\|_2^2 \Big].$
\end{corollary}

\begin{proof}
Denote an $L-$layer NN $g\circ f(x) := W_g^\top \cdot z^{[L-1]}$, where $z^{[l]} := \sigma^{[l-1]}(W_{[l-1]}^\top \cdot z^{[l-1]})$, $l = 1, 2,3,..., L-1$; Define that $\sigma^{[0]}(W_{[0]}^\top \cdot z^{[0]}) := x$, then $\nabla^2\big(g\circ f(x)\big) = 0$ (B.2 of \cite{zhang2020does}). Since $\nabla^2\big(g\circ f(x)\big) = W_g^\top \cdot \nabla^2 f(x) $, we have $\nabla^2 f(x) = 0$. 

Combine this with Theorem \ref{apdx:DLDA}, we have 
\begin{align}
    \Omega_{A_P^{DL}}
    &= \frac{1}{n}\sum_{i=1}^n\sum_{k=1}^d
    w_{i,(k)}\Big[tr\Big(\E_{W_\star}[\Delta_{x_i}\Delta_{x_i}^\top ] \mathcal{H}_f^{(k),i}\Big) \Big] = 0;\\
    \Omega_{A_L^{DL}} &= \Omega_{A_P^{DL}} + \frac{1}{n}\sum_{i=1}^n \Big[\E_{W_\star} \big\| x_i^\top\Delta_{W_\star}\nabla f(x_i) W_g\big\|_2^2 + \E_{b_\star} \big\| b_\star^\top\nabla f(x_i) W_g\big\|_2^2 \Big] \\
    &= \frac{1}{n}\sum_{i=1}^n \Big[\E_{W_\star} \big\| x_i^\top\Delta_{W_\star}\nabla f(x_i) W_g\big\|_2^2 + \E_{b_\star} \big\| b_\star^\top\nabla f(x_i) W_g\big\|_2^2 \Big].
\end{align}
\end{proof}
\textbf{\textit{Remark.}} Corollary \ref{apdx:DLDA2} analyzes special cases (Relu/ Max-pooling activation) of Theorem \ref{apdx:DLDA}, giving notably different regularization effect: in these cases the $A_P^{DL}$ (average the prediction) fails as a regularizer, therefore, doesn't fulfill our sufficient conditions for improving domain transferability (Theorem \ref{thm:emp-bound}); $A_L^{DL}$ (average the loss) only reserves the regularization on $\nabla f$-dependent norms, but no longer regularizes $\mathcal{H}_f(x)$. Since $A_L^{DL}$ still induces regularization, the induced sufficient conditions analyzed after Theorem \ref{apdx:DLDA} for promoting domain transferability won't be affected.

\textbf{Comparison and connections with related work.} 
On the empirical end, recent work uncovers that heuristic DA can replace explicit regularization mechanisms (dropout, weight decay, etc.)~\cite{hernandez2018data, hernandez2018further, zhang2021understanding}. 
On the theoretical end, there has been a line of work on understanding the DA-induced regularization, including the branches of 1) regularization from specific methods such as mixup~\cite{carratino2020mixup,zhang2020does, greenewald2021k}, random noise~\cite{bishop1995training}, adversarial examples~\cite{szegedy2013intriguing}, etc.; 2) DA-induced regularization on the variance at the feature or output level~\cite{leen1995data,van2018learning,dao2019kernel, chen2020group}; 3) regularization on Hessian-based complexity~\cite{lejeune2019implicit}. 
Our analysis in this section contributes uniquely in that: 1) we consider a general DA family of linear transformation~\cite{perez2017effectiveness}, which can be extended to most of the previously analyzed specific DA mechanisms; 2) Besides data-level DA, we also investigate feature-level DA~\cite{devries2017dataset} which shows advantages in empirical performance but lacks theoretical support. 3) Rather than demonstrating the regularization effect abstractly at the feature or output level, our results indicates data-dependent $l_2$ regularization on the DA transformation parameter and Jacobian of the model under concrete applicable sufficient conditions. 

\section{Adversarial Training as a Regularizer}\label{sec:adv-reg}
In this section, we show, under certain conditions, why adversarial training may improve domain generalization by viewing adversarial training as a function class regularizer.

We first provide some notation. Let 
\[
\mathcal{F} = \{f_\theta(\cdot) = W^L\phi^{L-1}(W^{L-1}\phi^{L-2}(\dots) +b^{L-1}) + b^L  \}
\]
where $\phi^j$ are activations, $W^j,b^j$ are weight matrix and bias vector, $\theta$ is the collection of parameters (i.e. $\theta = (W^1,b^1,\dots,W^L,b^L)$. For the rest of the article, assume that $\phi^j$ are just ReLUs.

Now fix $x \in \gX$. Define the preactivation as
\begin{align}
 \widetilde{x}^1 &:= W^1x + b^1 \\
    \widetilde{x}^j &:= W^j\phi^{j-1}(\widetilde{x}^{j-1}) + b^j \, , \, j \geq 2
\end{align}

Define the activation pattern $\phi_x := (\phi^1_x,\dots,\phi_x^{L-1}) \in \{0,1\}^m$ such that
for each $j \in [L-1]$
\begin{equation}
   \phi^j_x =  \1(\widetilde{x}^j > 0)
\end{equation}
where $\1$ is applied elementwise.

Now, given an activation pattern $\phi \in \{0,1\}^m$, we define the preimage $X(\phi) := \{ x \in \sR^{d} : \phi_x = \phi\}$

\begin{theorem}\label{th:equiv}
(In the proof of theorem 1 in \citep{roth2020reg})

Let $\epsilon > 0$ s.t. $B^p_\epsilon(x) \subset X(\phi_x)$ where $B^p_\epsilon(x)$ denotes the $l_p$ ball centered at $x$ with radius $\epsilon$. Let $p = \{1,2,\infty\}$ and $q$ be the Holder conjugate of $p$ (i.e. $\frac{1}{p} + \frac{1}{q} = 1$). Then
\begin{align}
    \E_{(x,y)\sim P}[\l(y,f(x)) + \lambda \max_{x^* \in B^p_\epsilon(x)} \norm{f(x) - f(x^*)}_q] = \E_{(x,y)\sim P}[\l(y,f(x)) + \lambda \cdot \epsilon  \max_{v^* \, : \, \norm{v^*}_p \leq 1} \norm{J_{f(x)}v}_q]
\end{align}
\end{theorem}

\textbf{Interpretation:}
This theorem provides an equivalence between the objective functions for adversarial training (left term) and jacobian regularization (right term). 
We give some intuition on the size of $\epsilon$.
Let us first consider a shallow 2 layer network $f(x) = W^2\phi(W^1x+b^1)$. Suppose $W^2 \in \sR^{m_2 \times m_1}$ and $W^1 \in \sR^{m_1 \times d}$. Given a matrix M, let $M_j$ denote the $j$th row of $M$. We study the activation pattern $\phi_x$ which equals
\begin{align}
    \phi_x = (\phi^1_x) = \begin{pmatrix}
    \1 \{ W^1_1x + b^1_1 \} \\
    \vdots \\
    \1 \{ W^1_{m_1}x + b^1_{m_1} \}
    \end{pmatrix}
\end{align}
We wish to compute the largest radius $\epsilon$ such that the activation pattern $\phi_x$ is constant within $B^2_\epsilon(x)$. This is simply the distance from $x$ to the closest hyperplane of the form $H_{W^1_j, b^1_j} = \{x \in \sR^d \, : \, W^1_{j}x + b^1_{j} = 0\}$ where $j = 0,\dots,m_1$ (i.e. $ \epsilon = \min_j \textrm{dist}(x, H_{W^1_j, b^1_j})$). In particular, if $W^1 = I_{d \times d}$ and $b^1 = \bold{0}$, $\epsilon = \min_{j \in d} \lvert x_j \rvert$.  

Furthermore, we note that $\epsilon$ is nondecreasing as a function of the number of layers. However, it has been observed empirically in \citep{roth2020reg} that approximate correspondence holds in a much larger ball.

\begin{definition} (source and target function class)
Let $\gG^S, \gG^T$ be fine tuning function classes for source and target domains, respectively. We define the class of source models as 
\[
 \gH^S = \gG^S \circ \gF = \{g^S \circ f_\theta \, : \, g^S \in \gG^S, f_\theta \in \gF \}  
\]
and the class of target models as
\[
 \gH^T = \gG^T \circ \gF =  \{g^T \circ f_\theta \, : \, g^T \in \gG^T, f_\theta \in \gF \}  
\]
\end{definition}

\begin{definition} (empirical training objective with jacobian regularization)
Let $\lambda, \epsilon > 0$. Take any hypothesis $h_\theta = g^S \circ f_\theta \in \gH^S$. Let $\hat{R}(h_\theta) = \frac{1}{n}\sum_{i=1}^n\ell(h_{\theta}(x_i),y_i)$ denote the empirical risk where $l(\hat{y},y) = \norm{\hat{y}-y}^2$. We define the empirical training objective with jacobian regularization as
\begin{equation}
    \aobj(h_\theta) = \hat{R}(h_\theta)  + \frac{\lambda \cdot \epsilon}{n} \sum_{i=1}^n \norm{J_{h_\theta}(x_i)}_2 
\end{equation}
\end{definition}

\begin{theorem}\label{th:adv}
Fix regularization strength $\lambda > 0$.
Define 
\[
\mathcal{F}^{A}_\lambda = \{f^A_\theta \in \mathcal{F} : \exists g^S \in \gG^S \, \mathrm{s.t.} \,~ \aobj(g^S \circ f^A_\theta) \leq \aobj(\bold{0}) \}
\]
where $\bm{0}$ denotes the zero function (i.e. the class of feature extractors that outperform the zero function). Suppose $(x,y) \in \gX \times \gY$ is bounded such that  $\max{(\norm{x}_{\infty} ,\norm{y}_{2} )} \leq R$. Fix $\delta > 0$.
Suppose we additionally restrict our fine tuning class models to linear models where 
\[
G^S = \{ W \, : \, W \in \R^{d \times n}, n \geq 1, \min_j \norm{W_j}_{2} \geq  \delta \}
\]
(where $W_j$ is the jth column of $W$) and 
\[
G^T = \{ W \, : \, W \in \R^{d \times n}, n \geq 1 \}
\]
(Here we are abusing notation to let $g^S \in G^S$ to denote the last linear layer as well as the fine tuning function).

Then for $0 \leq \lambda_1 < \lambda_2$
 \begin{equation}
     \gF^A_{\lambda_2} \subsetneq \gF^A_{\lambda_1}  \subsetneq \gF
 \end{equation}
 (where $\subsetneq$ denotes proper subset). In particular, if $H^{A,T}_\lambda = \gG^T \circ \gF^A_\lambda$, we have
  \begin{equation}
     \gH^{A,T}_{\lambda_2} \subsetneq \gH^{A,T}_{\lambda_1} \subsetneq \gH^T
 \end{equation}
\end{theorem}

\textbf{Interpretation:}

At the high level, this theorem captures the idea that minimizing the empirical risk with jacobian regularization puts a constraint on the set of feature extractors. In particular, $\gF^A_{\lambda_1}$ represents the potential class of feature extractors we select after training with jacobian regularization. Therefore, the class of fine tuned models $H^{A,T}_{\lambda_1}$ with feature extractors trained with jacobian regularization for the target domain is smaller than the class of fine tuned models $\gH$ with feature extractors trained without any regularization. Furthermore, we show that the space of feature extractors shrinks as we increase the regularization stength $\lambda$. Since we showed in section \ref{sec:up} that smaller function classes have smaller $d_{\gF_A}$, this theorem shows that jacobian regularization reduces  $d_{\gF_A}$. To connect back to adversarial training, if $\epsilon$ satisfies the hypothesis in theorem $\ref{th:equiv}$, we have that 
\begin{align}
    \E_{(x,y)\sim P}[\l(y,f(x)) + \lambda \max_{x^* \in B^p_\epsilon(x)} \norm{f(x) - f(x^*)}_q] = \E_{(x,y)\sim P}[\l(y,f(x)) + \lambda \cdot \epsilon  \max_{v^* \, : \, \norm{v^*}_p \leq 1} \norm{J_{f(x)}v}_q]
\end{align}
Therefore, minimizing the training objective with jacobian regularization is equivalent to minimizing the adversarial training objective. Using this connection, this theorem essentially shows that, given sufficient number of samples, adversarial training reduces the class of feature extractors which in turn reduces $d_{\gF_A}$.

Finally, we comment on the assumption that $\norm{g^S} > \delta$. Since $\delta > 0$ is arbitrary, we can make it as small as we like and thus we are essentially excluding the $\bold{0}$ last layer which is hardly a constraint on the function class. This assumption is necessary as we are considering regularization on the whole model $g \circ f$ as opposed to regularization on just the feature extractor. Thus, this assumption prevents the scenario where only the last linear layer is regularized.

\begin{proof}
We first show that if $0 \leq \lambda_1 < \lambda_2$, we have that
 \begin{equation}
     \gF^A_{\lambda_2} \subsetneq \gF^A_{\lambda_1}  \subsetneq \gF
 \end{equation}
 
 We first prove the following lemma
 \begin{lemma}\label{lem:eq}
 Suppose the conditions of theorem $\ref{th:adv}$ are satisfied. Suppose additionally we have that $\overline{y} := \frac{1}{n} \sum_{i=1}^d y_i \neq 0$ (note this occurs with probability 1 if marginal distribution  over $\gY$ is continuous). Then for every $\lambda \geq 0$, there exists a function $f_{\theta} \in \gF^A_{\lambda}$ and a fine tuning layer $g^* \in \gG^S$ such that 
 \[
 \aobj(g^* \circ f_\theta) = \inf_{g \in G^S} \aobj(g \circ f_\theta) = \aobj(\bold{0})
 \]
 
Choose another $\lambda' \geq 0$ (can equal $\lambda$). Then there exists a ${g^*}' \in \gG^S$ be the fine tuning layer  such that $\inf_{g \in \gG^S} \aobj[\lambda'](g \circ f_\theta) =  \aobj[\lambda']({g^*}' \circ f_\theta) $ and
 \[
\frac{1}{n} \sum_{i=1}^n\norm{J_{{g^*}' \circ f_\theta}(x_i)}_2 > 0 
 \]
 \end{lemma}

 \begin{proof}
 
 Fix $\alpha \geq 0$ and $c > \alpha \cdot R$. Set biases
 \begin{align}
     b^1 = \begin{pmatrix}
     c \\
     0\\
     \vdots \\
     0
     \end{pmatrix}
\, b^j = \bold{0}
 \, , \, j \geq 2
 \end{align}
 and weights
  \begin{align}
     W^1 =\begin{pmatrix}
    \alpha & 0 & \dots & 0 \\
    0 & 0 & \dots & 0 \\
    \vdots & \vdots & \vdots & \vdots \\
    0 & \dots & \dots & 0 \\
    \end{pmatrix}
\, W^j =\begin{pmatrix}
    1 & 0 & \dots & 0 \\
    0 & 0 & \dots & 0 \\
    \vdots & \vdots & \vdots & \vdots \\
    0 & \dots & \dots & 0 \\
    \end{pmatrix}
 \, , \, j \geq 2
 \end{align}
 
 Define $x_{i,j}$ be the $j$th entry of the data point $x_i$. Define 
 \begin{align}
     \alpha_i &:=  \alpha \cdot x_{i,1}  \\
     \overline{\alpha} &:= \frac{1}{n} \sum_{i=1}^d\alpha_i \\
    \overline{y} &:= \frac{1}{n} \sum_{i=1}^d y_i
 \end{align}
 
 Now we observe that for a fixed $\lambda \geq 0$ and any $g \in \gG^S$, we have that 
 \begin{align}
      \aobj(g\circ f_{\theta}) &= \hat{R}(g \circ f_{\theta}) + \lambda \cdot \epsilon \frac{1}{n} \sum_{i=1}^n\norm{J_{g \circ f_{\theta}}(x_i)}_2 \\
       &= \frac{1}{n}\sum_{i=1}^n \norm{(\alpha x_{i,1} + c)\begin{pmatrix}
       g_{11} \\
       \vdots \\
       g_{d1}
       \end{pmatrix} - y_i}^2_2 + \lambda \cdot \epsilon \frac{1}{n} \sum_{i=1}^n\norm{g \begin{pmatrix}
    \alpha & 0 & \dots & 0 \\
    0 & 0 & \dots & 0 \\
    \vdots & \vdots & \vdots & \vdots \\
    0 & \dots & \dots & 0 \\
    \end{pmatrix}}_2 \\
     &= \frac{1}{n}\sum_{i=1}^n  \norm{( \alpha_i + c) g_1 - y_i}^2_2 + \lambda \cdot \epsilon \frac{1}{n} \sum_{i=1}^n \alpha \norm{g_1}_2 \\
      &= \frac{1}{n}\sum_{i=1}^n  \norm{( \alpha_i + c) g_1 - y_i}^2_2 + \lambda \cdot \epsilon \alpha \norm{g_1}_2 
 \end{align}
 Therefore, 
 \begin{align}
     \inf_{g \in \gG^S}   \aobj(g\circ f_{\theta})
 \end{align}
 is equivalent to solving
  \begin{align}
     \inf_{w \in \R^d \, : \, \norm{w}_2 \geq \delta }  \frac{1}{n}\sum_{i=1}^n  \norm{( \alpha_i + c) w - y_i}^2_2 + \lambda \cdot \epsilon \alpha \norm{w}_2 \label{minprob}
 \end{align}
 Utilizing lagrange multipliers, we find the minimizer is  
 \begin{align}
     w = \delta \cdot \frac{\overline{y}}{\norm{\overline{y}}} \label{minimizer}
 \end{align}
 when $c \geq \frac{\norm{\overline{y}}_2}{\delta}$. 
 
 Now consider the function 
 \[
 S(c,\alpha) =  \frac{1}{n}\sum_{i=1}^n  \norm{( \alpha \cdot x_{i,1} + c) g_1 - y_i}^2_2 + \lambda \cdot \epsilon \alpha \norm{g_1}_2 
 \]
 Note that this function is continuous with respect to the input $(c,\alpha)$. Now fix $\alpha = 0, c = \frac{\norm{\overline{y}}_2}{\delta}$. Set $w = \delta \cdot \frac{\overline{y}}{\norm{\overline{y}}}$. Then we have that 
\begin{align}
    S(\frac{\norm{\overline{y}}_2}{\delta},0) =  \frac{1}{n}\sum_{i=1}^n  \norm{\overline{y} - y_i}^2_2 < \frac{1}{n}\sum_{i=1}^n  \norm{y_i}^2_2 &= \aobj(\bold{0})
\end{align}
The inequality comes from the fact that we assumed $\overline{y} \neq 0$ and noting that $\overline{y}$ is the minimizer of the function $p(z) = \frac{1}{n}\norm{z - y_i}^2_2 $. Continuity of $S$ ensures that there exists $\alpha_0 > 0$ such that 
\begin{align}
    S(\frac{\norm{\overline{y}}_2}{\delta},\alpha_0) < \frac{1}{n}\sum_{i=1}^n  \norm{y_i}^2_2 &= \aobj(\bold{0})
\end{align}
 Now consider $U(t) = S((1+t)\frac{\norm{\overline{y}}_2}{\delta},(1+t)\alpha_0 )$ for $t \geq 0$. Note that $U$ is continuous with respect to $t$. Furthermore, we note that $t \rightarrow \infty$ implies $U(t) \rightarrow \infty$ which implies there exists some time $t = T_f$ such that $U(T_f) > \aobj(\bold{0})$. Therefore, by the intermediate value theorem, there exists a time $t = T$ such that $U(T) =  \aobj(\bold{0})$. Finally, set $c = (T+1)\frac{\norm{\overline{y}}_2}{\delta}$, $\alpha = (T+1)\alpha_0$, and $g^*$ as the matrix where $g^*_1 = \delta \cdot \frac{\overline{y}}{\norm{\overline{y}}}$ and $\bold{0}$ for the other columns. By equation $\ref{minprob}$ and equation $\ref{minimizer}$  we have
 \begin{align}
      \aobj(g^* \circ f_{\theta}) = \inf_{g \in \gG^S}   \aobj(g\circ f_{\theta}) = U(T) = \aobj(\bold{0}) 
 \end{align}
 Furthermore, if we choose another $\lambda' \geq 0$, since $c = (T+1)\frac{\norm{\overline{y}}_2}{\delta} > \frac{\norm{\overline{y}}_2}{\delta}$ by equation $\ref{minimizer}$, we have that 
  \begin{align}
      \aobj[\lambda']({g^*}' \circ f_{\theta}) = \inf_{g \in \gG^S}   \aobj[\lambda'](g\circ f_{\theta}) 
 \end{align}
 and 
   \begin{align}
    \frac{1}{n} \sum_{i=1}^n\norm{J_{{g^*}' \circ f_\theta}(x_i)}_2 = \alpha \norm{{g^*}'}_2 = \alpha \delta 
 \end{align}
 which is nonzero as $\delta > 0$ and $\alpha = (T+1)\alpha_0 > 0$.
 
 \end{proof}

\textbf{Proof of \ref{th:adv}} \\

Clearly, we have $\gF^A_{\lambda_2} \subset \gF^A_{\lambda_1}$. If we can show that $f_{\theta_1} \not\in \gF^A_{\lambda_2}$ then we have 
 $\gF^A_{\lambda_2} \subsetneq \gF^A_{\lambda_1}$. 
 
 Using lemma $\ref{lem:eq}$ we can find $f_{\theta_1} \in \gF^A_{\lambda_1}$ such that 
 \[
 \inf_{g \in G^S} \aobj[\lambda_1](g \circ f_{\theta_1}) = \aobj[\lambda_1](\bold{0})
 \]
 In addition lemma $\ref{lem:eq}$ guarantees  minimizers $g^*_1$ and $g^*_2$ such that 
 \begin{align}
     \aobj[\lambda_1](g^*_1 \circ f_{\theta_1}) &= \inf_{g \in G^S} \aobj[\lambda_1](g \circ f_{\theta_1}) \, \textrm{and} \, \frac{1}{n} \sum_{i=1}^n\norm{J_{g^*_1 \circ f_\theta}(x_i)}_2 > 0  \\
      \aobj[\lambda_2](g^*_2 \circ f_{\theta_1}) &= \inf_{g \in G^S} \aobj[\lambda_2](g \circ f_{\theta_1}) \, \textrm{and} \, \frac{1}{n} \sum_{i=1}^n\norm{J_{g^*_2 \circ f_{\theta_1}}(x_i)}_2 > 0  
 \end{align}
 Thus, we have that
 \begin{align}
     \aobj[\lambda_1](g^*_2 \circ f_{\theta_1}) &= \hat{R}(g^*_2 \circ f_{\theta_1}) + \lambda_2 \cdot \epsilon \frac{1}{n} \sum_{i=1}^n\norm{J_{g^*_2 \circ f_{\theta_1}}(x_i)}_2 \\
     &> \hat{R}(g^*_2  \circ f_{\theta_1}) + \lambda_1 \cdot \epsilon \frac{1}{n} \sum_{i=1}^n\norm{J_{g^*_2 \circ f_{\theta_1}}(x_i)}_2  &\textrm{since}~ \lambda_2 > \lambda_1\\
     &\geq \hat{R}(g^*_1  \circ f_{\theta_1}) + \lambda_1 \cdot \epsilon \frac{1}{n} \sum_{i=1}^n\norm{J_{g^*_1 \circ f_{\theta_1}}(x_i)}_2  &\textrm{def of }~ g^*_1\\
     &= \aobj[\lambda_1](\bold{0}) &\textrm{lemma $\ref{lem:eq}$}\\
 \end{align}
 
Thus $f_{\theta_1} \not\in \gF^A_{\lambda_2}$ which implies $\gF^A_{\lambda_2} \subsetneq \gF^A_{\lambda_1}$. It remains to show for $\lambda_1 \geq 0$, we have that $\gF^A_{\lambda_1}  \subsetneq \gF$.

Consider any $g \in \gG^S$. For $j \in [L]$, define $W^j$ as the weight matrix where $W^j = I_{d \times d}$ (identity matrix) for $j \in [L-1]$ and let the final weight $W^L = B\cdot I_{d \times d}$ for some constant $B > 0$. Set the bias vectors $b^j = \bold{0}$ for $j \geq 2$. Let the first bias equal $b^1 = R \cdot \1$ where $\1$ is the vector of all 1's and $R$ is the upper bound such that $\norm{x}_\infty \leq R$. Set $\theta = (W^1,b^1,\dots,W^L,b^L)$ and let $h_\theta = g \circ f_\theta$

We compute 
\begin{align}
    \aobj[\lambda_1](h _{\theta}) &= \hat{R}(h_\theta) + \lambda_1 \cdot \epsilon \frac{1}{n} \sum_{i=1}^n\norm{J_{h_{\theta}}(x_i)}_2 \\
    &= \frac{1}{n} \sum_{i=1}^n \norm{B(x_i+R\1) - y_i}^2  + \frac{1}{n} \sum_{i=1}^n\norm{J_{h_{\theta}}(x_i)}_2  \\
    &= \frac{1}{n} \sum_{i=1}^n \norm{B(x_i+R\1) - y_i}^2 +   B \norm{g}_2  \\
    &\geq \frac{1}{n} \sum_{i=1}^n \norm{B(x_i+R\1) - y_i}^2 +   B \delta  
\end{align}

We note that sending $B \rightarrow \infty$ we get $\aobj[\lambda_1](h _{\theta}) \rightarrow \infty$ which implies that there exists a $B = B'$ such that $\aobj[\lambda_1](h _{\theta}) > \aobj[\lambda_1](\bold{0})$. Setting $B = B'$ implies $f_\theta \not\in \gF^A_{\lambda_1}$.

\end{proof}

\section{Extra Experiment Results}
\subsection{Absolute Transferability vs. Model Robustness}
\label{sec:app-abs-result}
We show the absolute transferability versus robustness of different models in Figure~\ref{fig:app-result-ll},\ref{fig:app-result-jr},\ref{fig:app-result-da} and \ref{fig:app-result-resolution} respectively.

\begin{figure}[htb]
    \centering
    \includegraphics[height=1.4in]{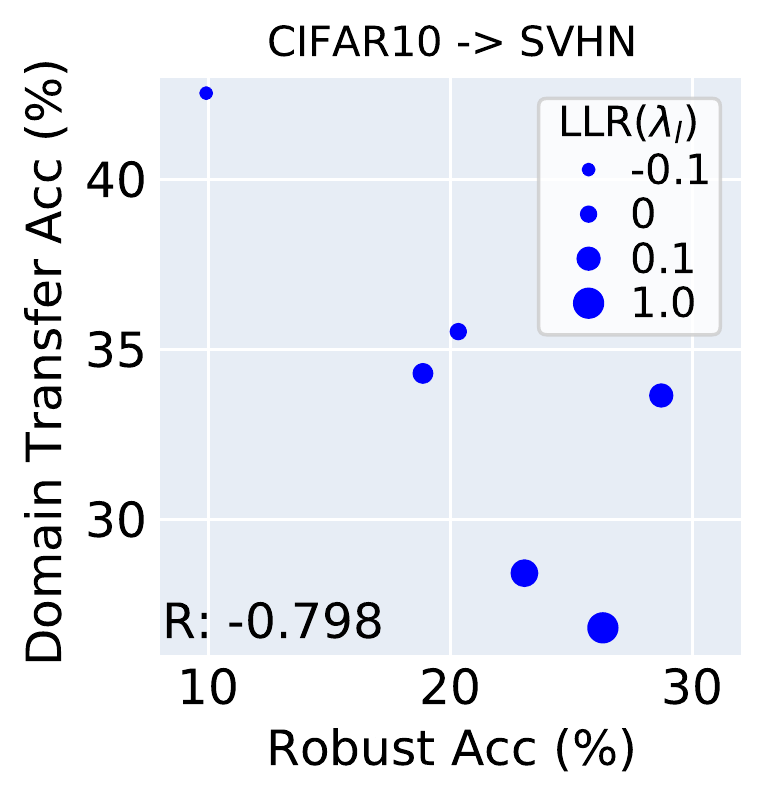}
    \includegraphics[height=1.4in]{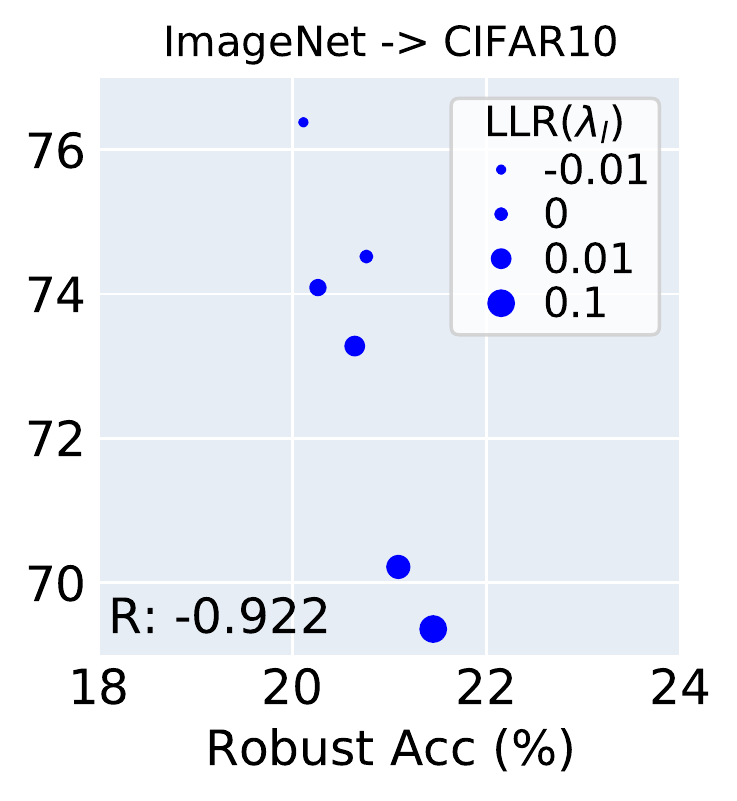}
    \includegraphics[height=1.4in]{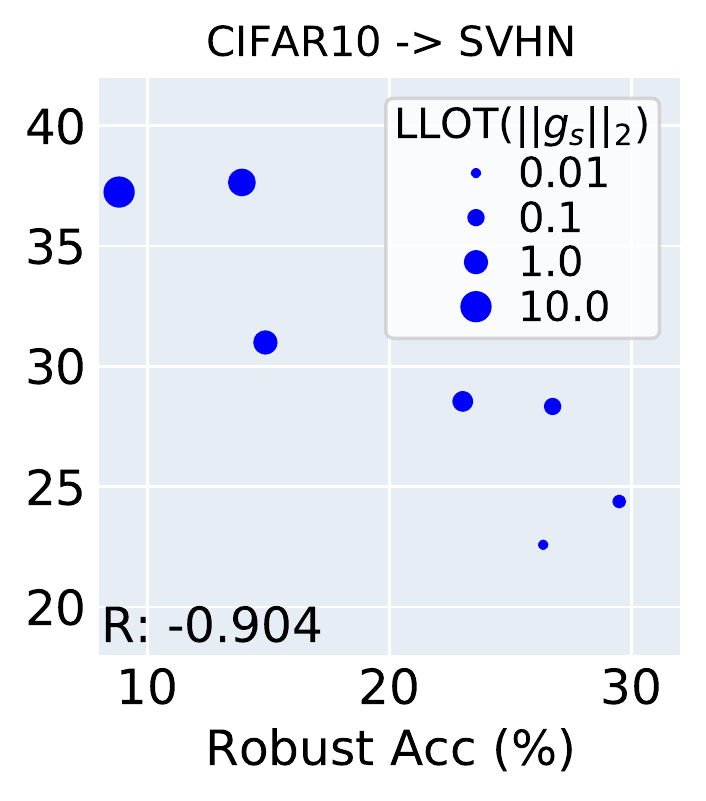}
    \includegraphics[height=1.4in]{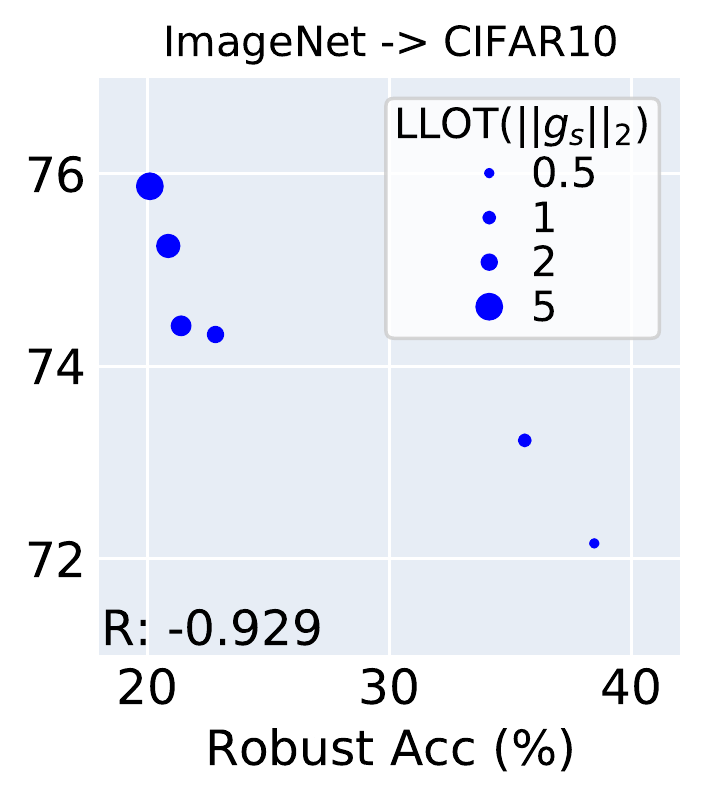}
    \caption{Robustness and absolute transferability when we control the norm of last layer with last-layer regularization (LLR) and last-layer orthogonal training (LLOT) with different parameters.}
    \label{fig:app-result-ll}
\end{figure}

\begin{figure}[htb]
    \centering
    \includegraphics[height=1.4in]{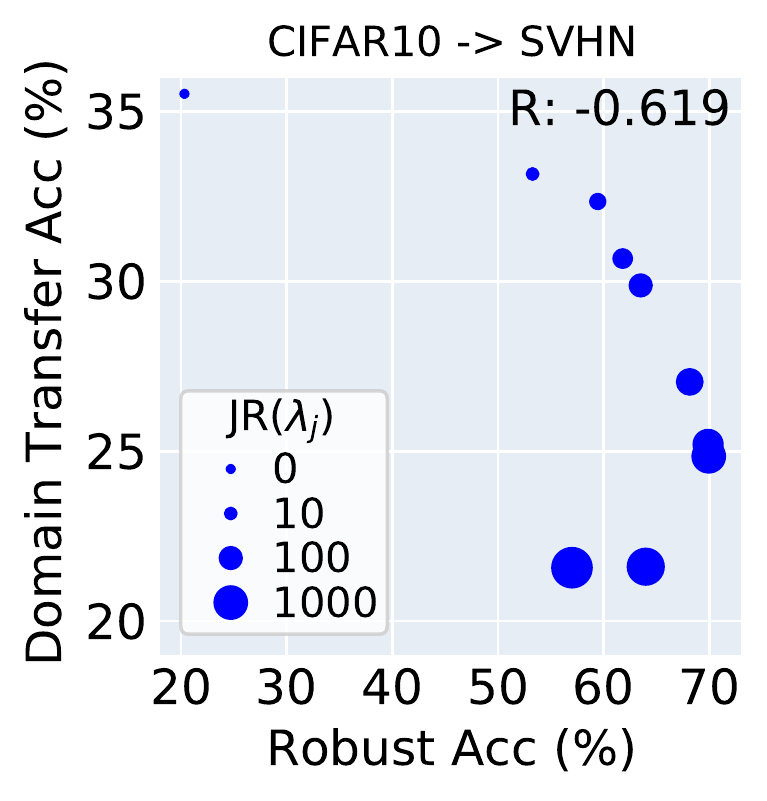}
    \includegraphics[height=1.4in]{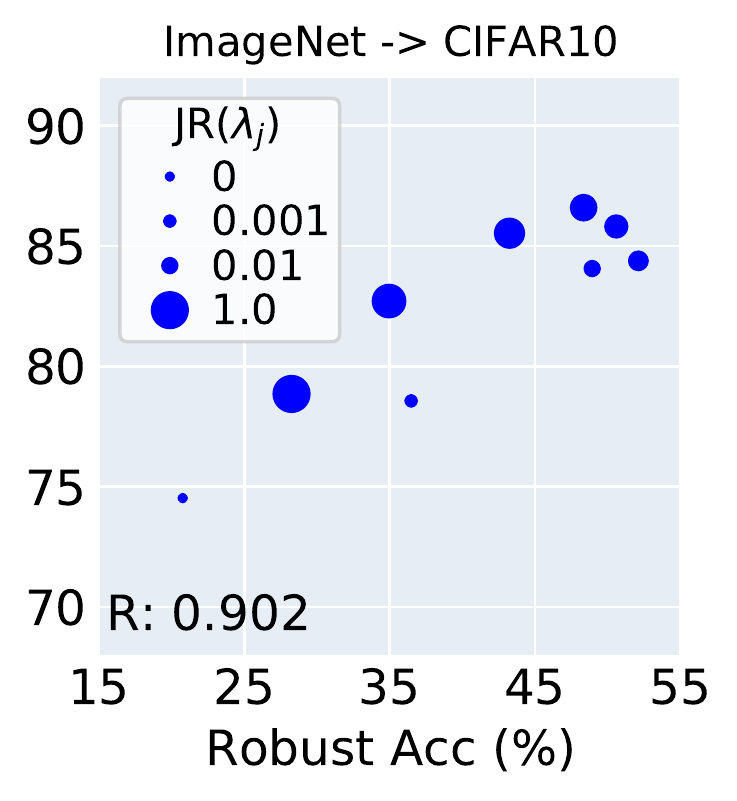}
    \includegraphics[height=1.4in]{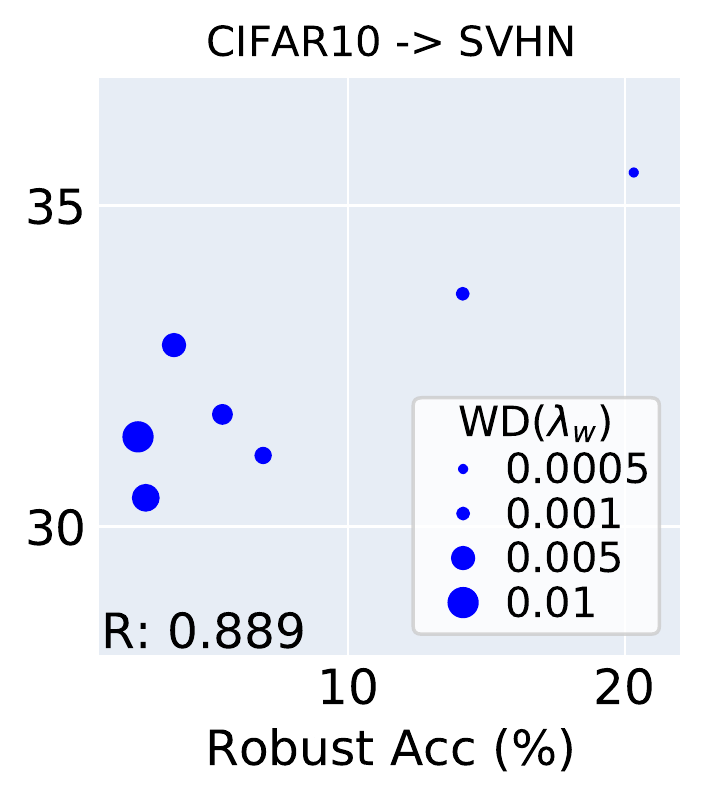}
    \includegraphics[height=1.4in]{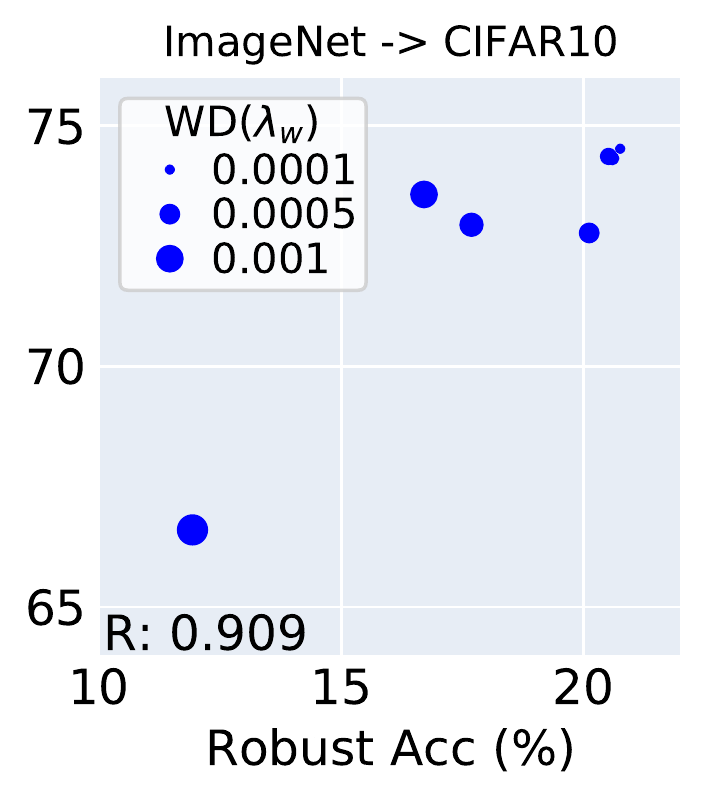}
    \caption{Robustness and absolute transferability when we regularize the feature extractor with Jacobian Regularization (JR) and weight decay (WD) with different parameters.}
    \label{fig:app-result-jr}
\end{figure}

\begin{figure}[htb]
    \centering
    \includegraphics[height=1.4in]{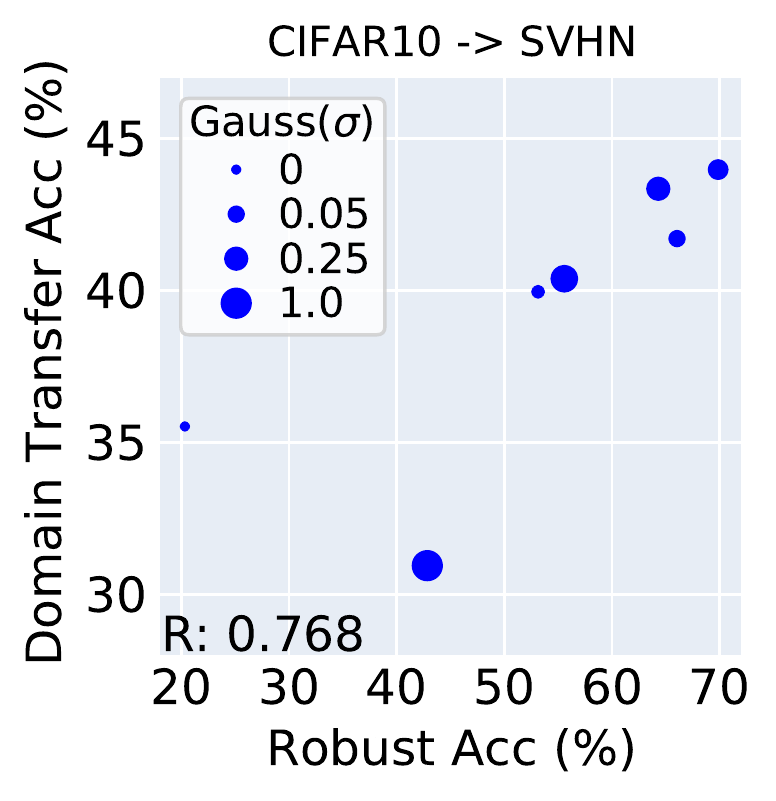}
    \includegraphics[height=1.4in]{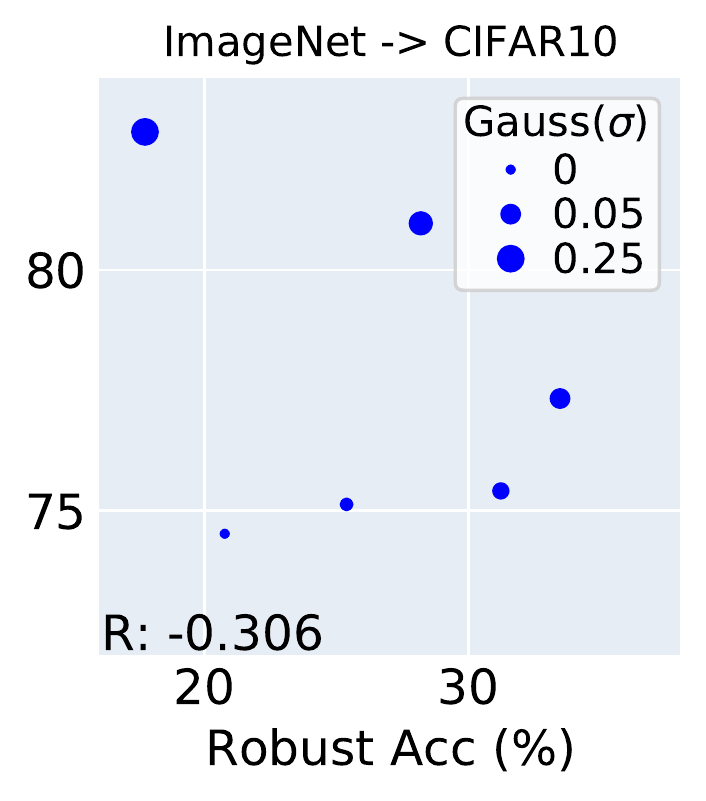}
    \includegraphics[height=1.4in]{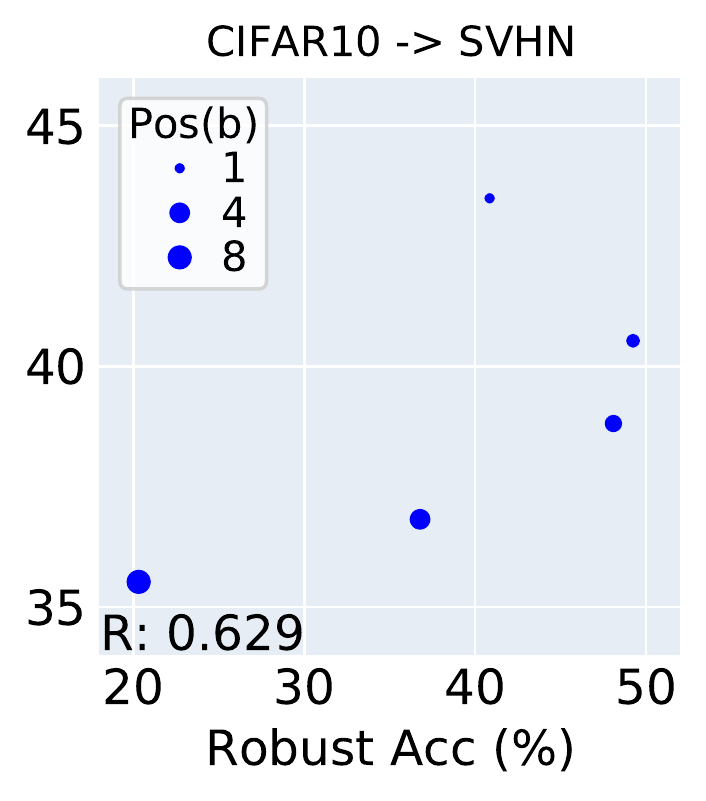}
    \includegraphics[height=1.4in]{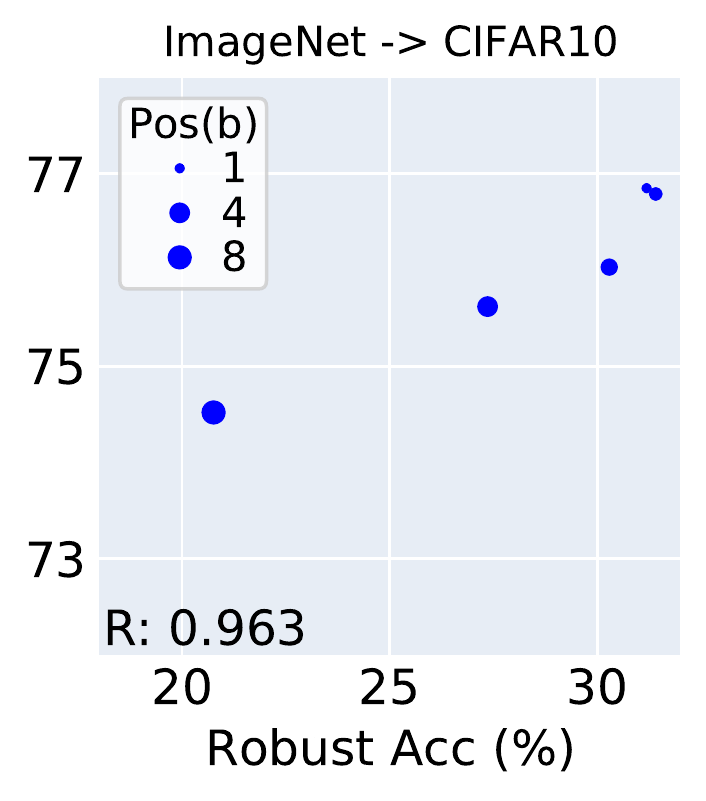}
    \caption{Robustness and absolute transferability when we use Gaussian noise (Gauss) and posterize (Pos) as data augmentations with different parameters.}
    \label{fig:app-result-da}
\end{figure}

\begin{figure}[htb]
    \centering
    \includegraphics[height=1.4in]{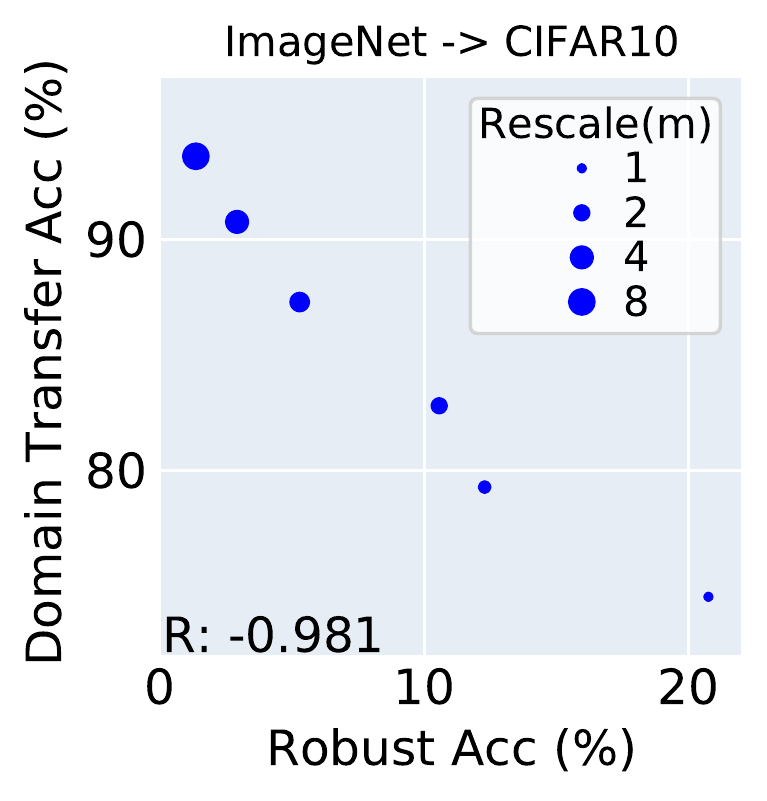}
    \includegraphics[height=1.4in]{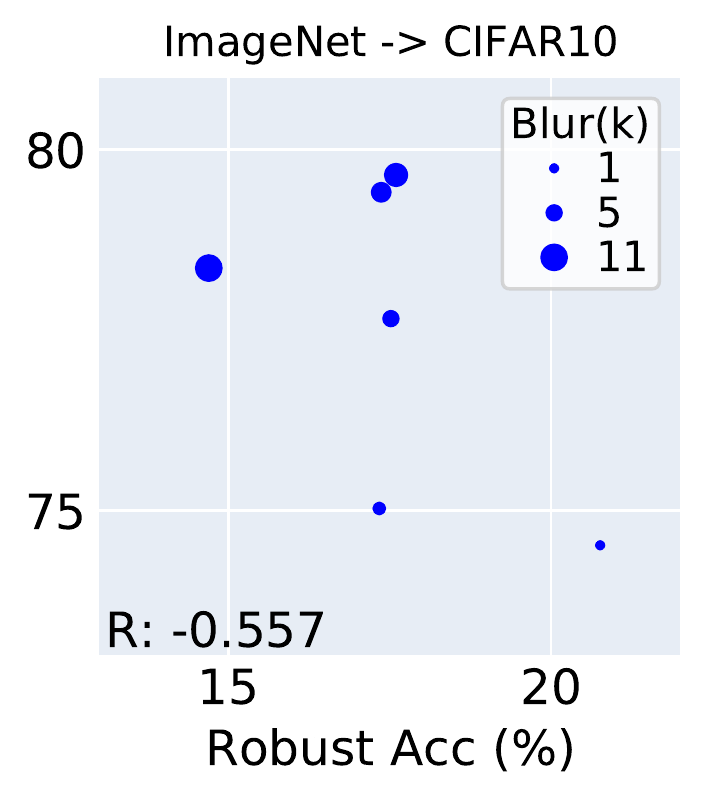}
    \caption{Robustness and absolute transferability when we use rescale and blur as data augmentations with different parameters.}
    \label{fig:app-result-resolution}
\end{figure}

\subsection{Absolute Transferability vs. Regularization Magnitude}
\label{sec:app-cmp-result}
We show the absolute transferability w.r.t. different regularization magnitude in Fig~\ref{fig:app-result-coef-ll},\ref{fig:app-result-coef-jr},\ref{fig:app-result-coef-da} and \ref{fig:app-result-coef-resolution} respectively. The green dashed line is the transferability of vanilla trained model and the red dashed line is that of adversarially trained model. We can observe that with most single regularization or augmentation, the model transferability can be improved compared with vanilla trained model and sometimes even outperform the adversarially trained ones. In some cases (e.g. jacobian regularize), the performance drops because the larger regularization hurts benign accuracy and therefore the absolute transferability drops.

\begin{figure}[htb]
    \centering
    \includegraphics[height=1.4in]{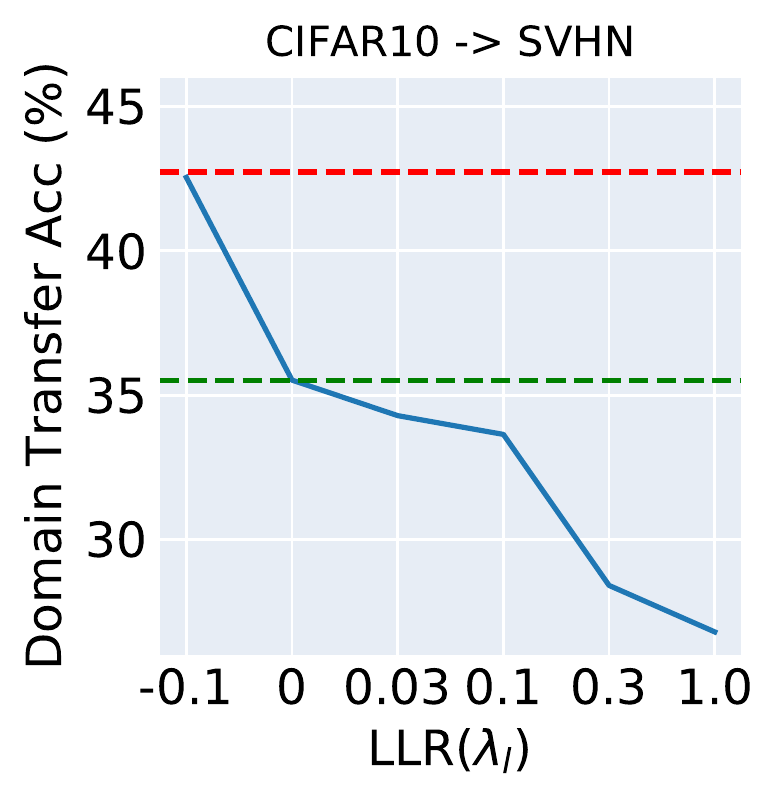}
    \includegraphics[height=1.4in]{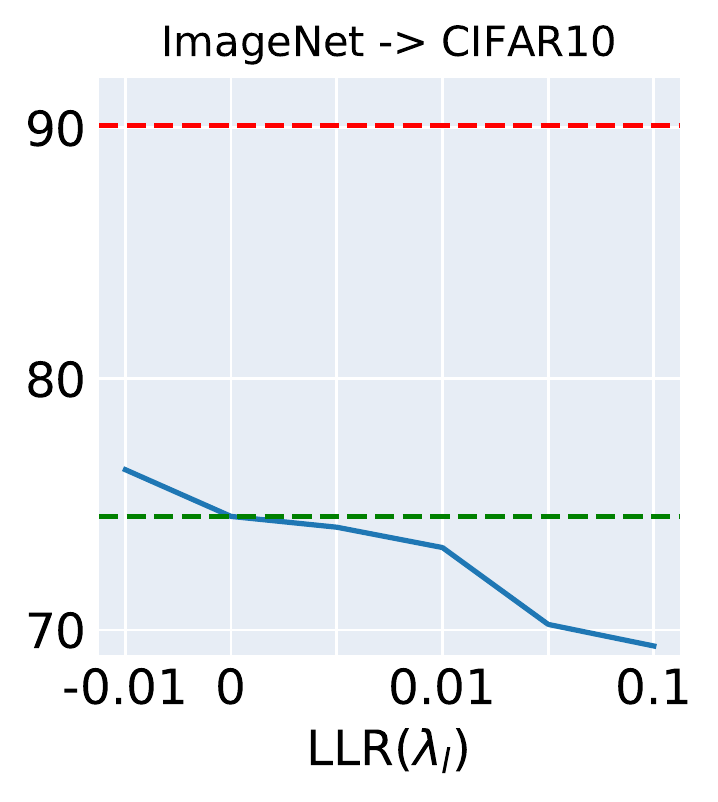}
    \includegraphics[height=1.4in]{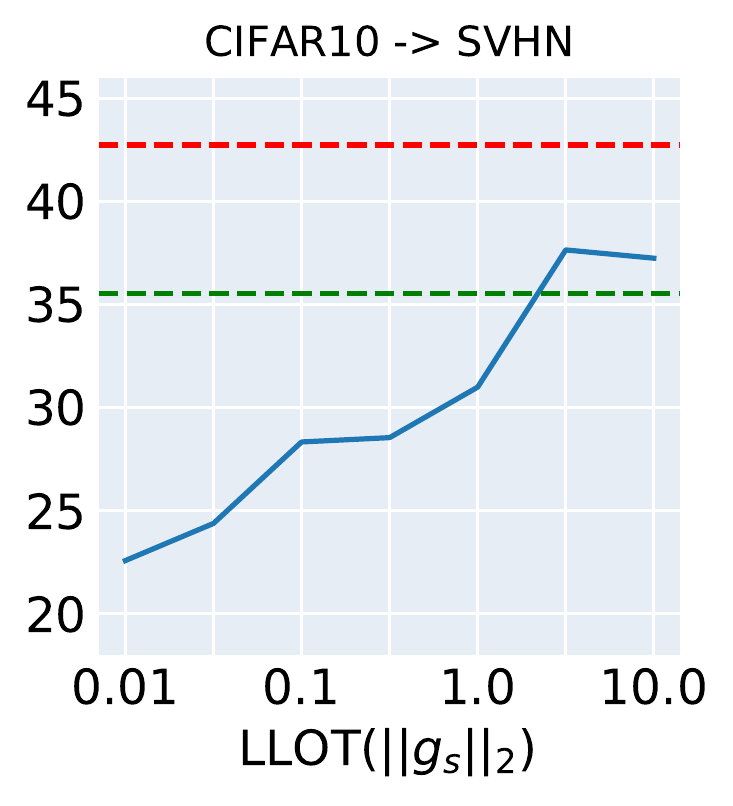}
    \includegraphics[height=1.4in]{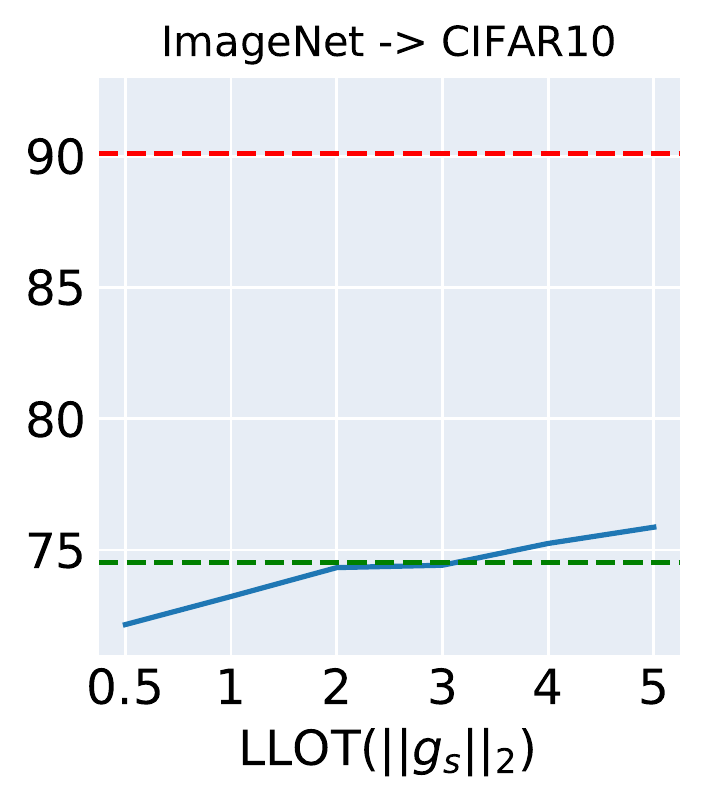}
    \caption{Absolute transferability when we control the norm of last layer with last-layer regularization (LLR) and last-layer orthogonal training (LLOT) with different parameters. Green dashed line is the transferability of vanilla trained model and red dashed line is that of adversarially trained model.}
    \label{fig:app-result-coef-ll}
\end{figure}

\begin{figure}[htb]
    \centering
    \includegraphics[height=1.4in]{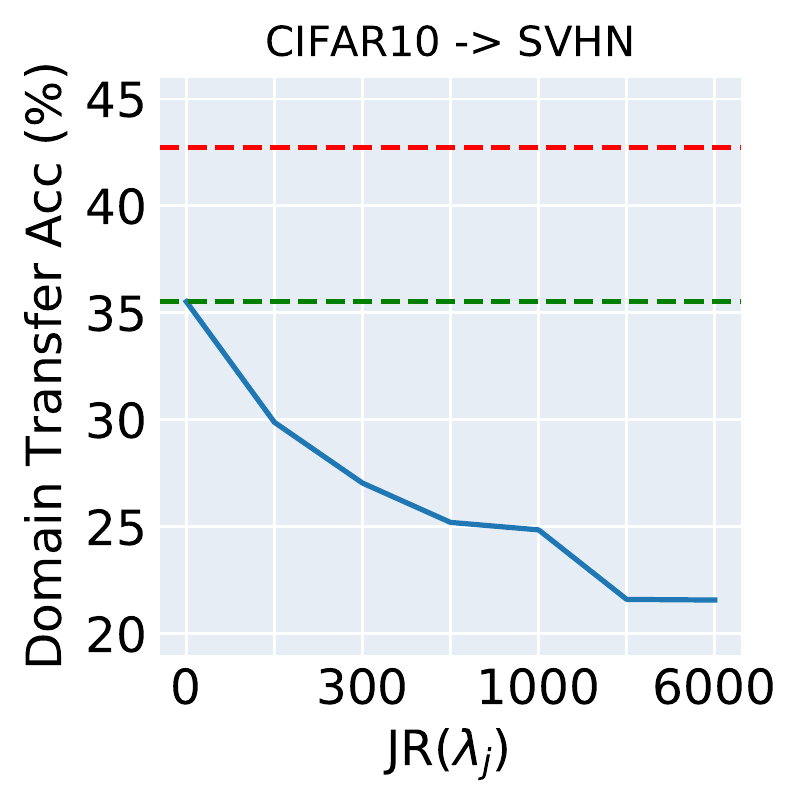}
    \includegraphics[height=1.4in]{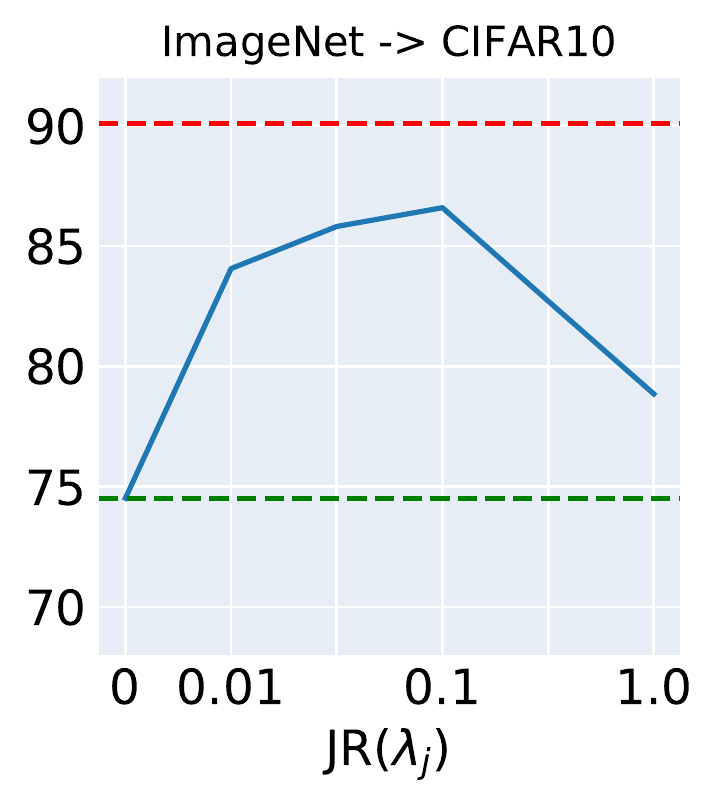}
    \includegraphics[height=1.4in]{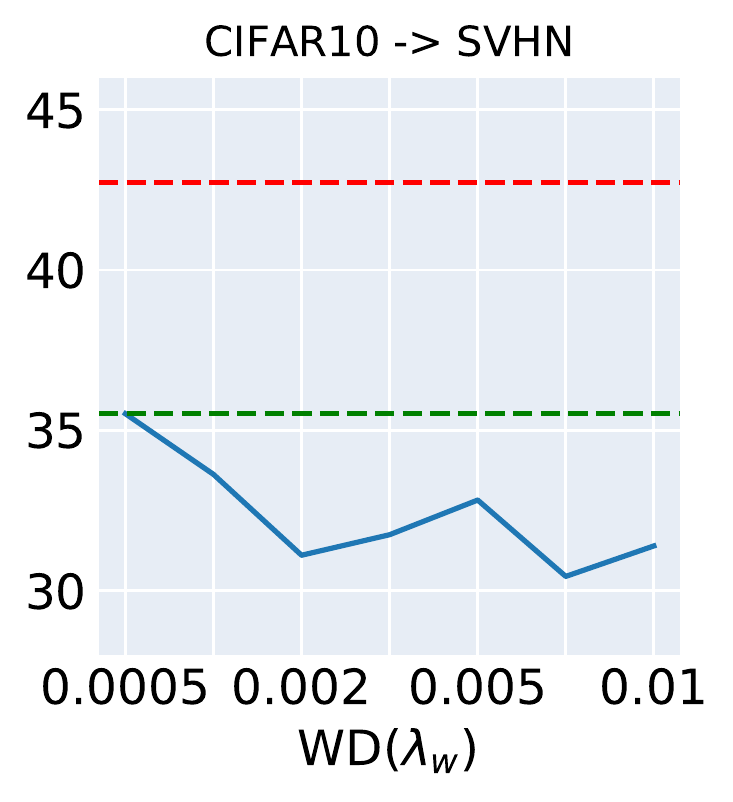}
    \includegraphics[height=1.4in]{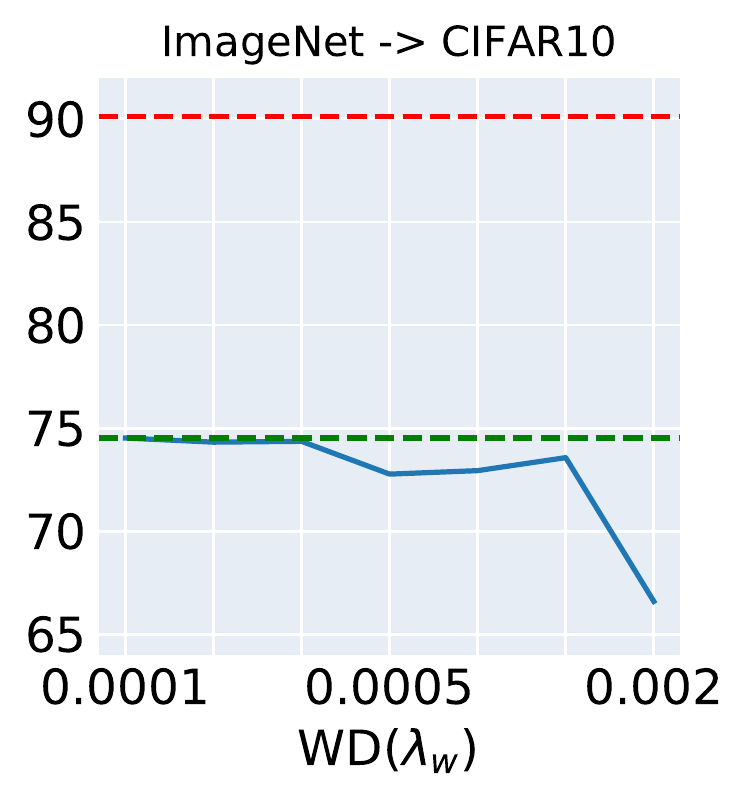}
    \caption{Absolute transferability when we regularize the feature extractor with Jacobian Regularization (JR) and weight decay (WD) with different parameters. Green dashed line is the transferability of vanilla trained model and red dashed line is that of adversarially trained model.}
    \label{fig:app-result-coef-jr}
\end{figure}

\begin{figure}[htb]
    \centering
    \includegraphics[height=1.4in]{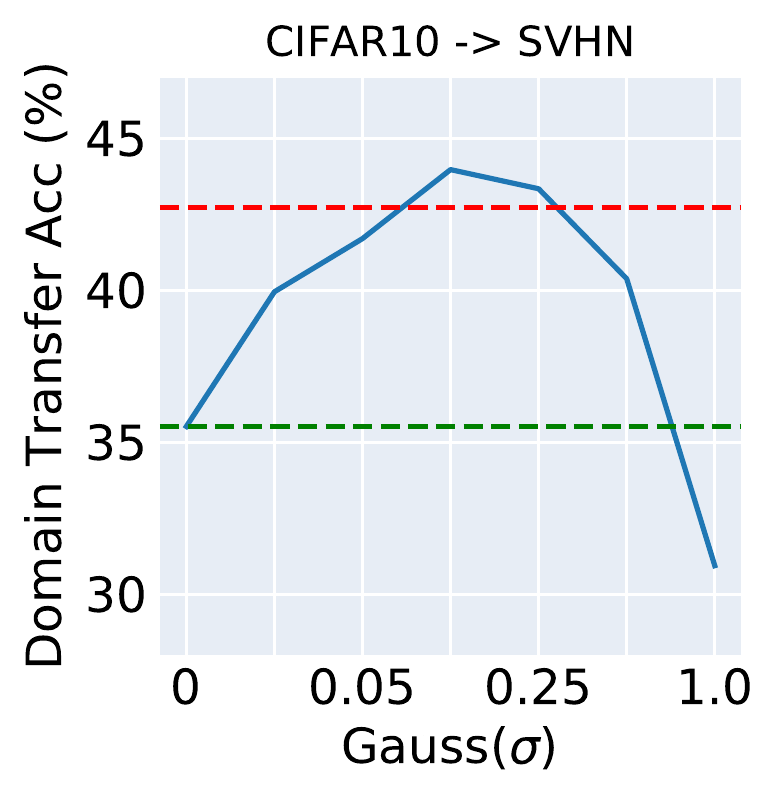}
    \includegraphics[height=1.4in]{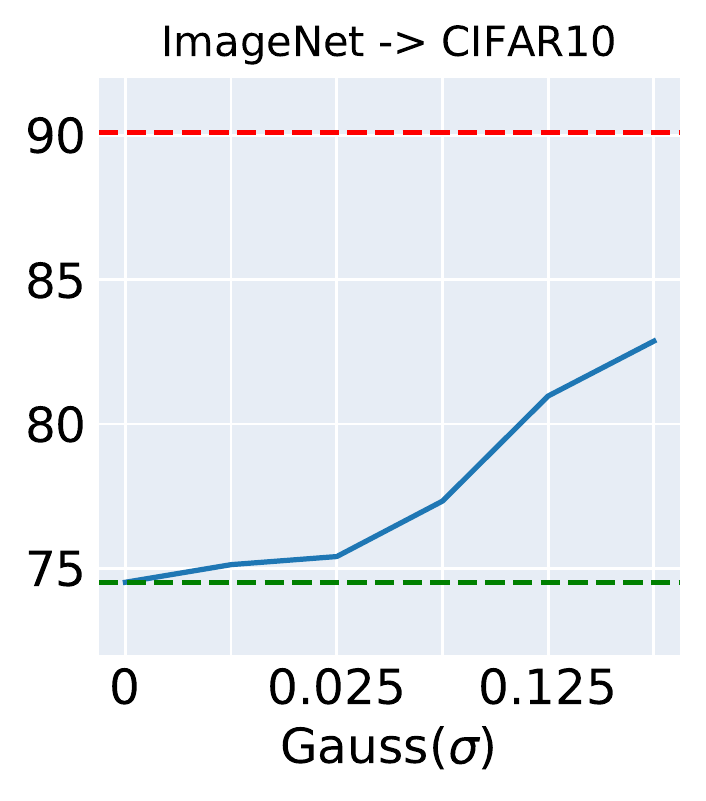}
    \includegraphics[height=1.4in]{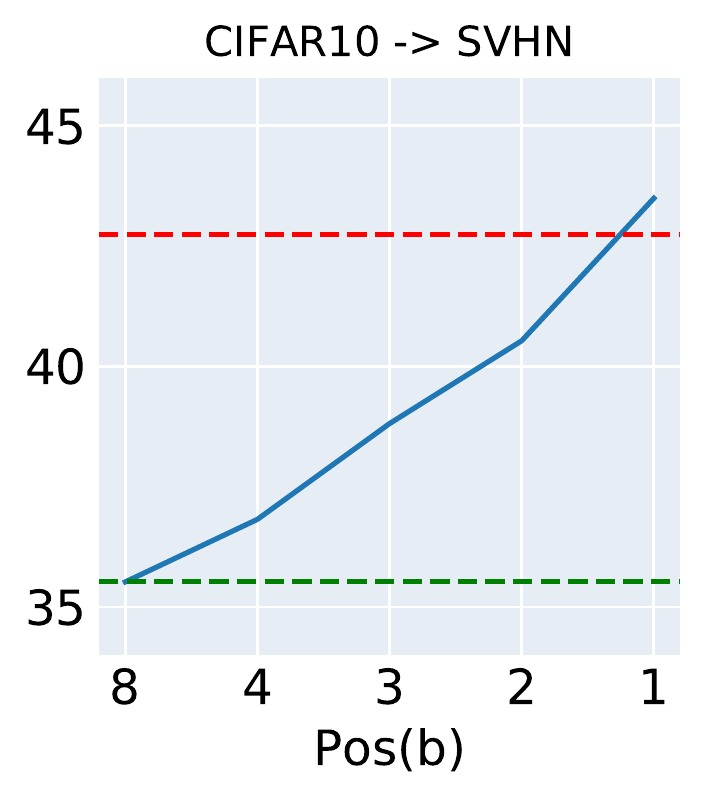}
    \includegraphics[height=1.4in]{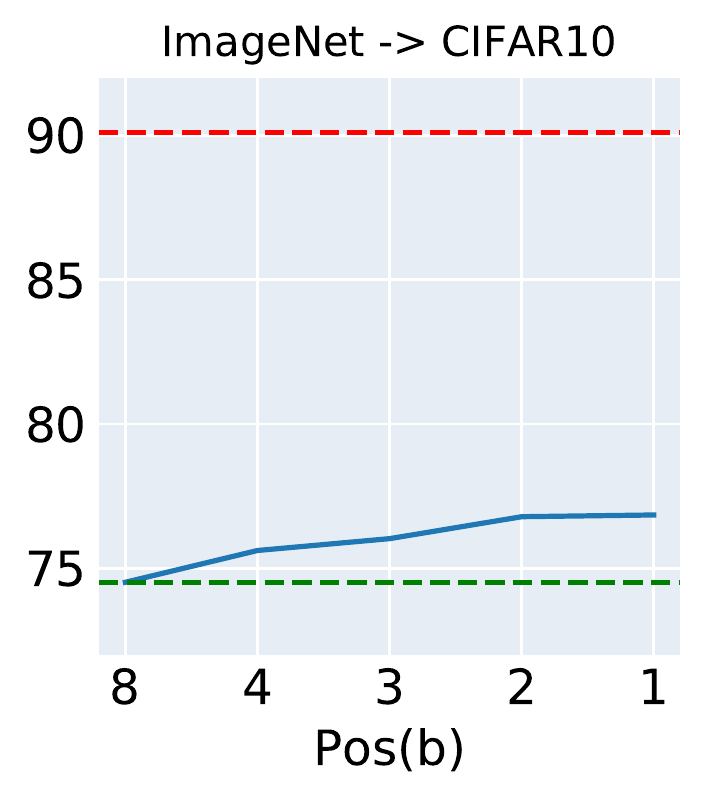}
    \caption{Absolute transferability when we use Gaussian noise (Gauss) and posterize (Pos) as data augmentations with different parameters. Green dashed line is the transferability of vanilla trained model and red dashed line is that of adversarially trained model.}
    \label{fig:app-result-coef-da}
\end{figure}

\begin{figure}[htb]
    \centering
    \includegraphics[height=1.4in]{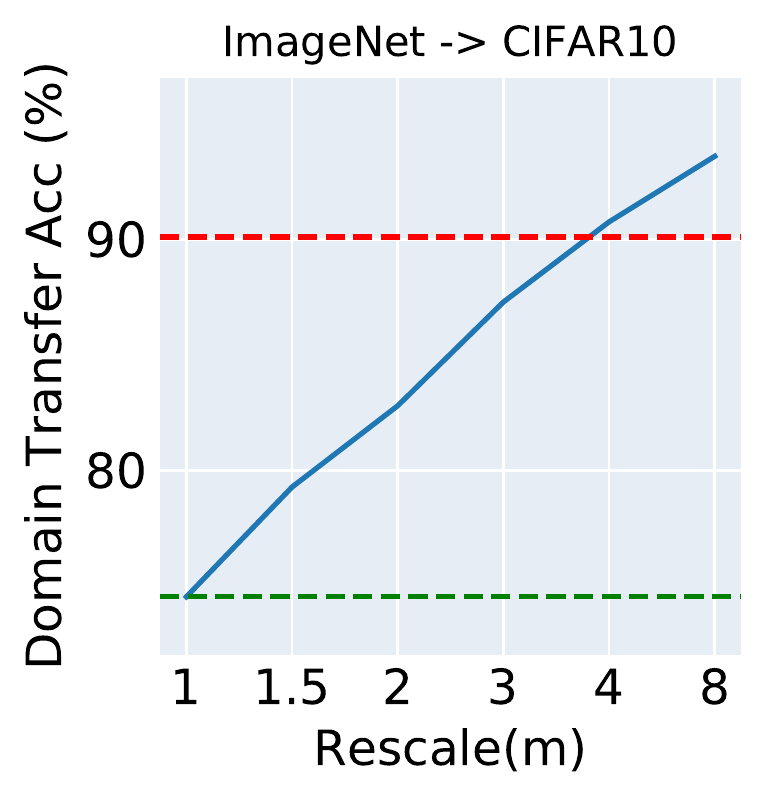}
    \includegraphics[height=1.4in]{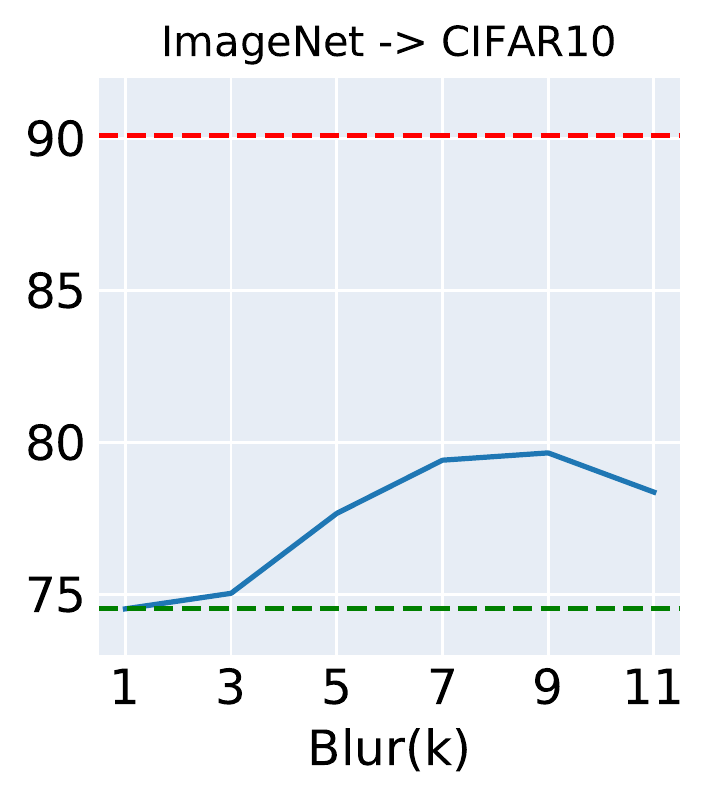}
    \caption{Absolute transferability when we use rescale and blur as data augmentations with different parameters for ImageNet. Green dashed line is the transferability of vanilla trained model and red dashed line is that of adversarially trained model.}
    \label{fig:app-result-coef-resolution}
\end{figure}

\subsection{Results of Other Model Structures}
\label{sec:app-wrncnn-result}
To further validate our evaluation results, we evaluate the experiments on another model structure. We use a simpler CNN model for CIFAR-10 to SVHN and a more complicated WideResNet-50 for ImageNet to CIFAR-10. The CNN model consists of four convolutional layer with $3 \times 3$ kernels and 32,32,64,64 channels respectively, followed by two hidden layer with size 256. A $2\times 2$ max pooling is calculated after the second and fourth layer. Other settings are the same as in the main text. Note that in some settings the new model cannot converge, and therefore we will omit the result. In addition, Jacobian regularization cannot be applied on WideResNet-50 because of the large memory cost, so we do not include it in the figures. The results are shown in Figure~\ref{fig:app-wrncnn-result-ll}, \ref{fig:app-wrncnn-result-jr} and \ref{fig:app-wrncnn-result-da}.

\begin{figure}[htb]
    \centering
    \includegraphics[height=1.4in]{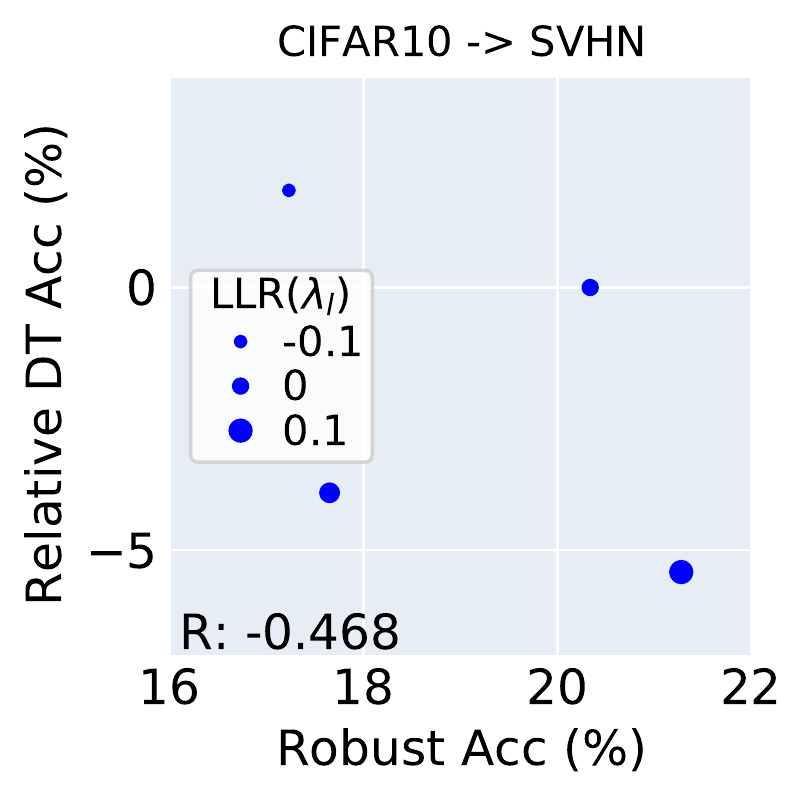}
    \includegraphics[height=1.4in]{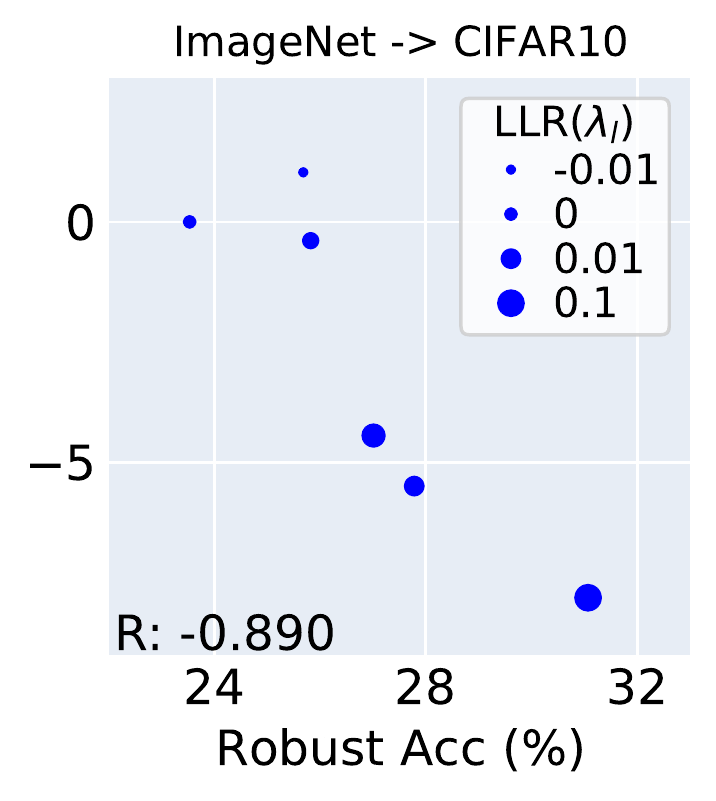}
    \includegraphics[height=1.4in]{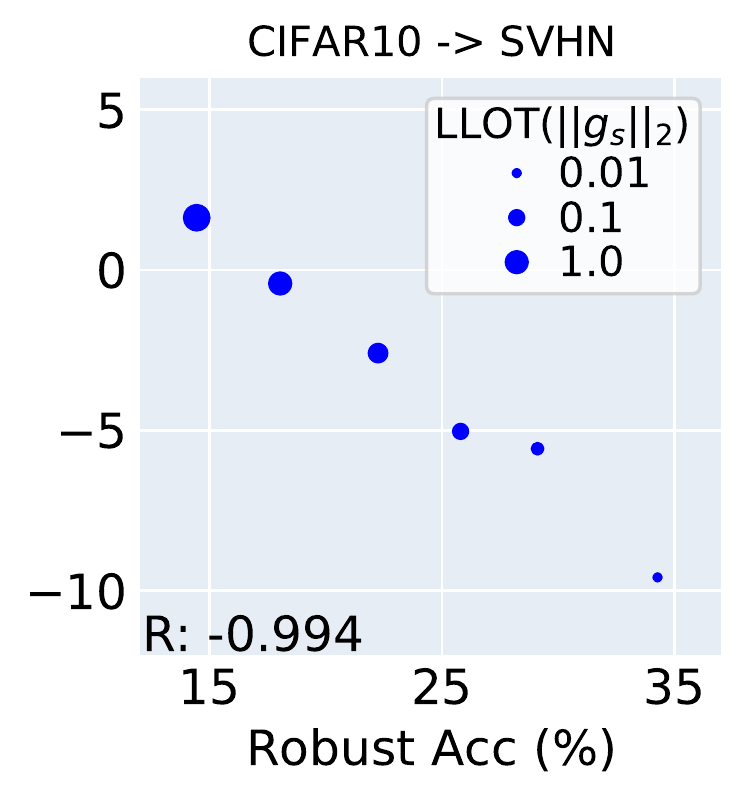}
    \includegraphics[height=1.4in]{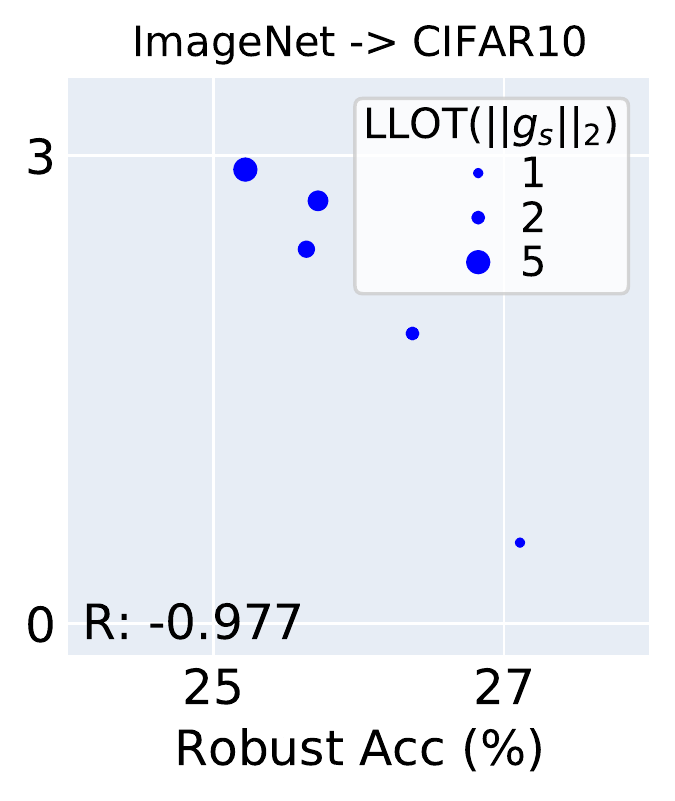}
    \caption{Robustness and transferability for the other model structure when we control the norm of last layer with last-layer regularization (LLR) and last-layer orthogonal training (LLOT) with different parameters.}
    \label{fig:app-wrncnn-result-ll}
\end{figure}
\begin{figure}[htb]
    \centering
    \includegraphics[height=1.4in]{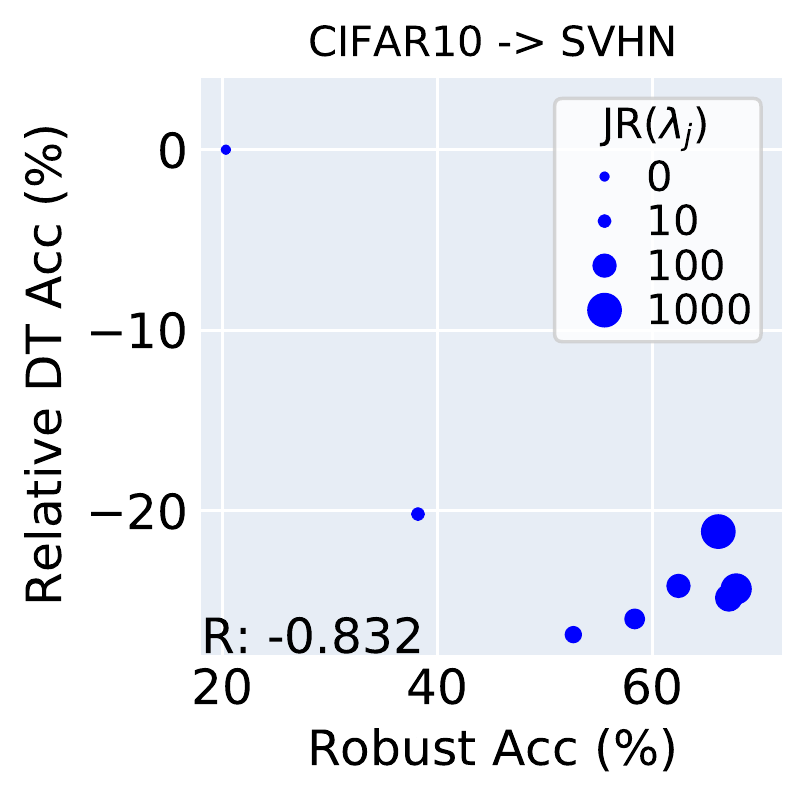}
    \includegraphics[height=1.4in]{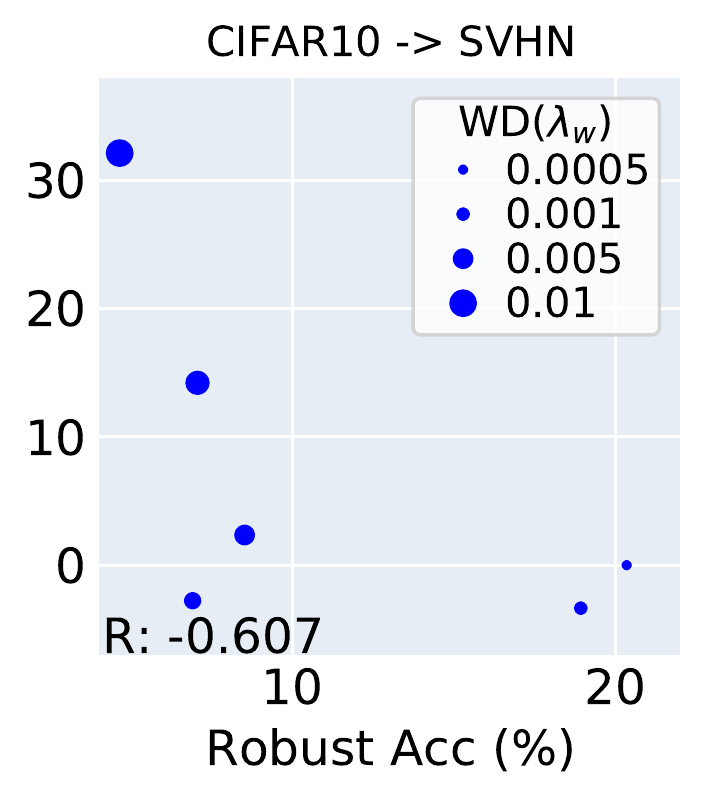}
    \includegraphics[height=1.4in]{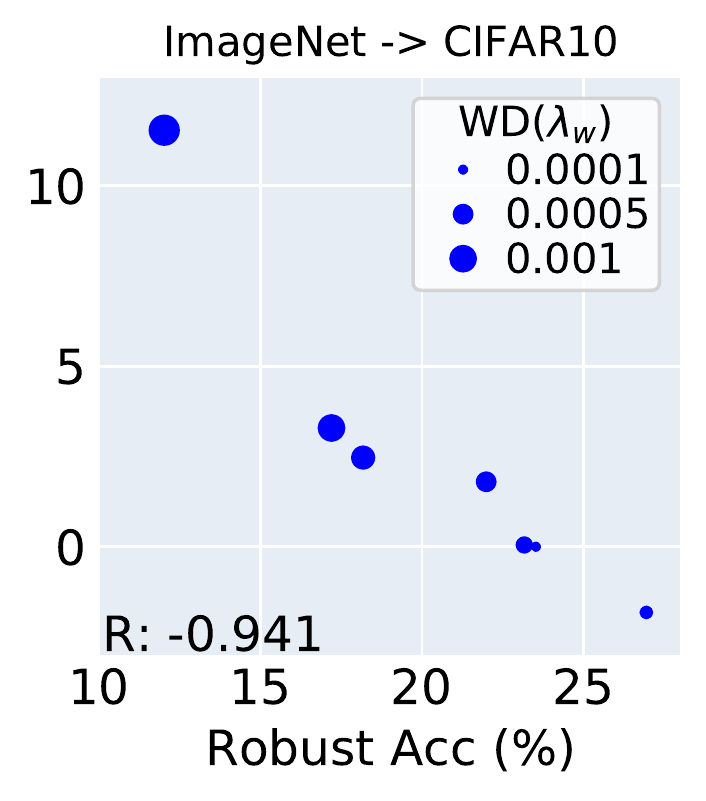}
    \caption{Robustness and transferability for the other model structure when we regularize the feature extractor with Jacobian Regularization (JR) and weight decay (WD) with different parameters.}
    \label{fig:app-wrncnn-result-jr}
\end{figure}
\begin{figure}[htb]
    \centering
    \includegraphics[height=1.4in]{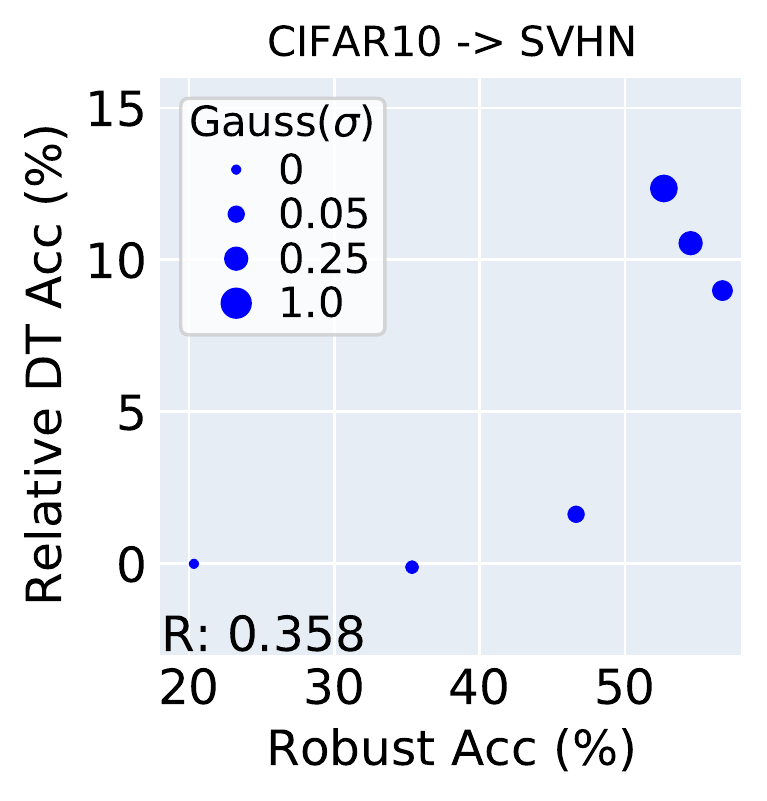}
    \includegraphics[height=1.4in]{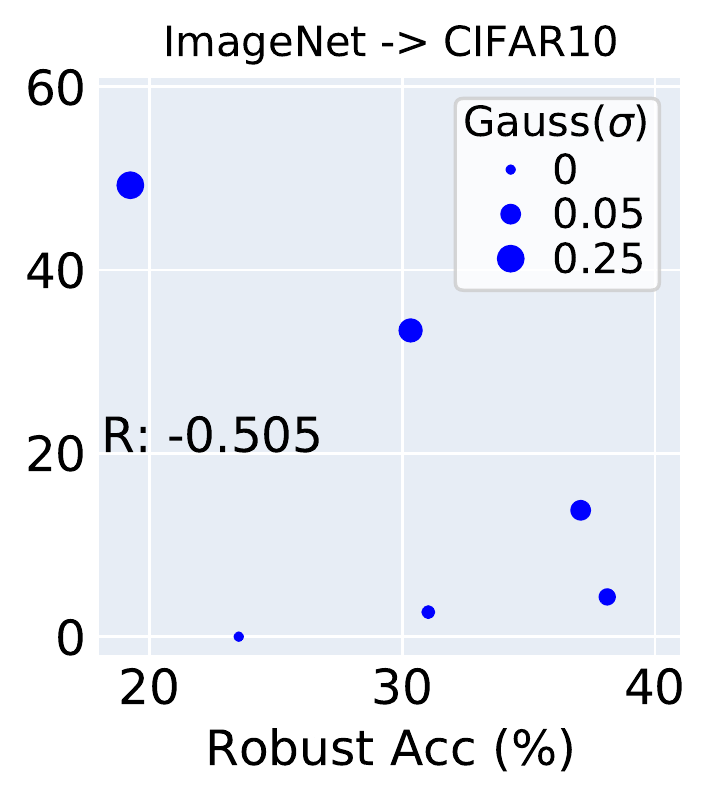}
    \includegraphics[height=1.4in]{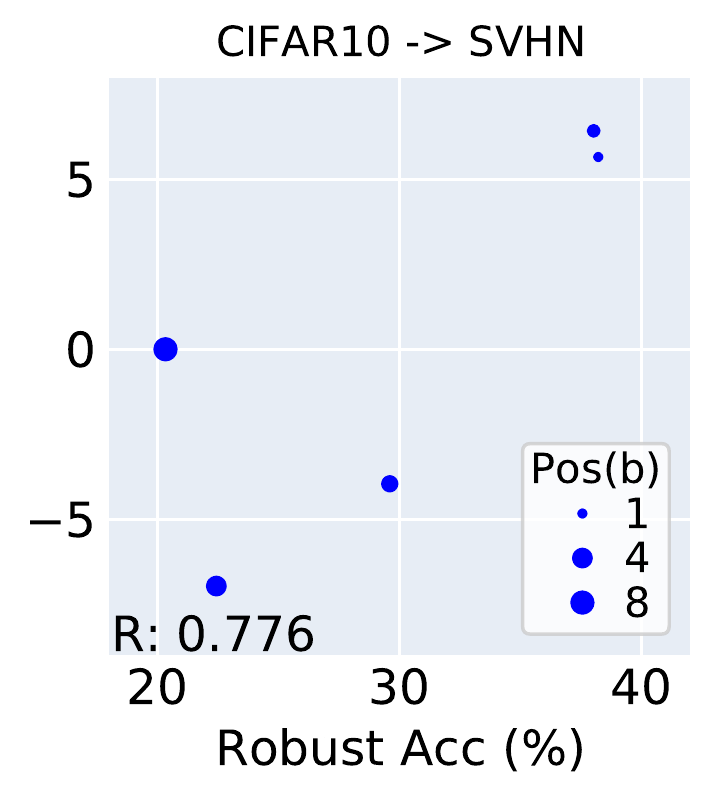}
    \includegraphics[height=1.4in]{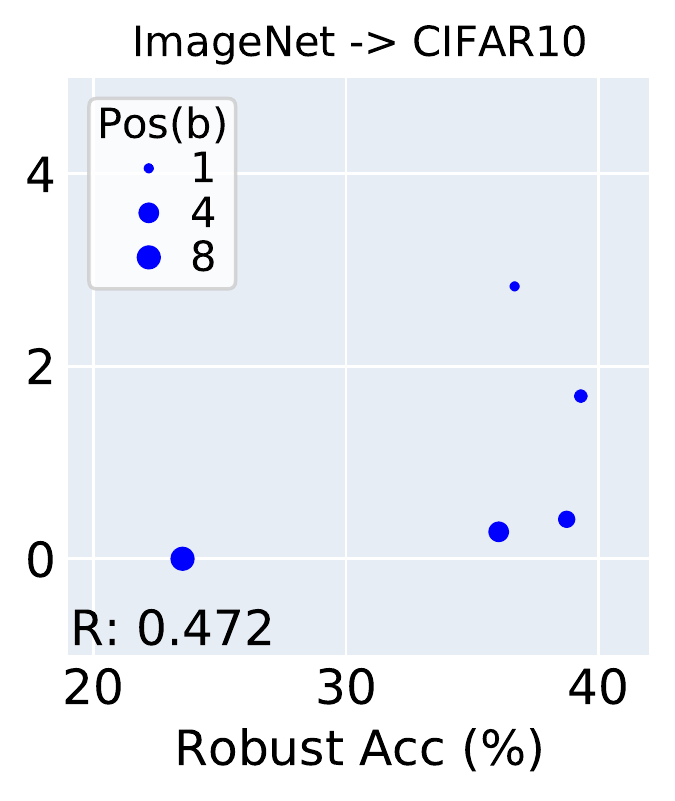}
    \caption{Robustness and transferability for the other model structure when we use Gaussian noise (Gauss) and posterize (Pos) as data augmentations with different parameters.}
    \label{fig:app-wrncnn-result-da}
\end{figure}

{
\subsection{Robustness Evaluation with AutoAttack}
\label{sec:app-autoatk}
Besides PGD attack, we also evaluate the model robustness using the stronger AutoAttack. We use APGD-CE, APGD-T and FAB-T as the sub-attacks in AutoAttack with 100 steps. Since the accuracy will decrease after the stronger attack, we use a slightly smaller $\eps=0.2$ to better visualize the trend. The results are shown in Fig.~\ref{fig:app-cifar-autoatk}. We can observe that the trend is similar with what we observed before when we used the PGD attack - domain generalization is an effect of regularization and data augmentation, and it is sometimes negatively correlated with model robustness. Also, augmentations like rotation and translation, which violates the sufficient condition, do not improve the domain generalization.}
\begin{figure}
    \centering
    \includegraphics[height=1.4in]{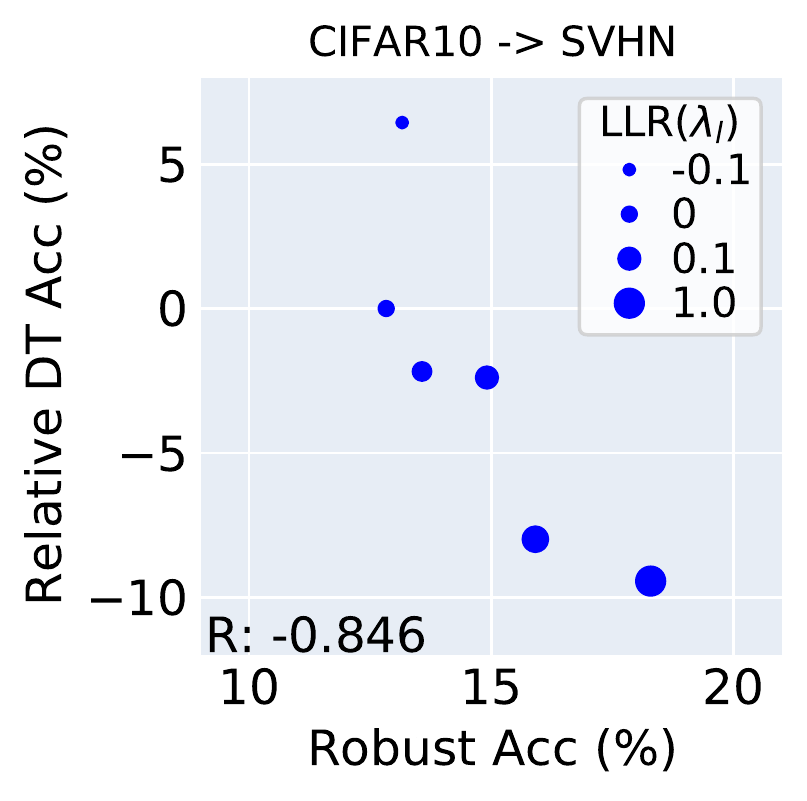}
    \includegraphics[height=1.4in]{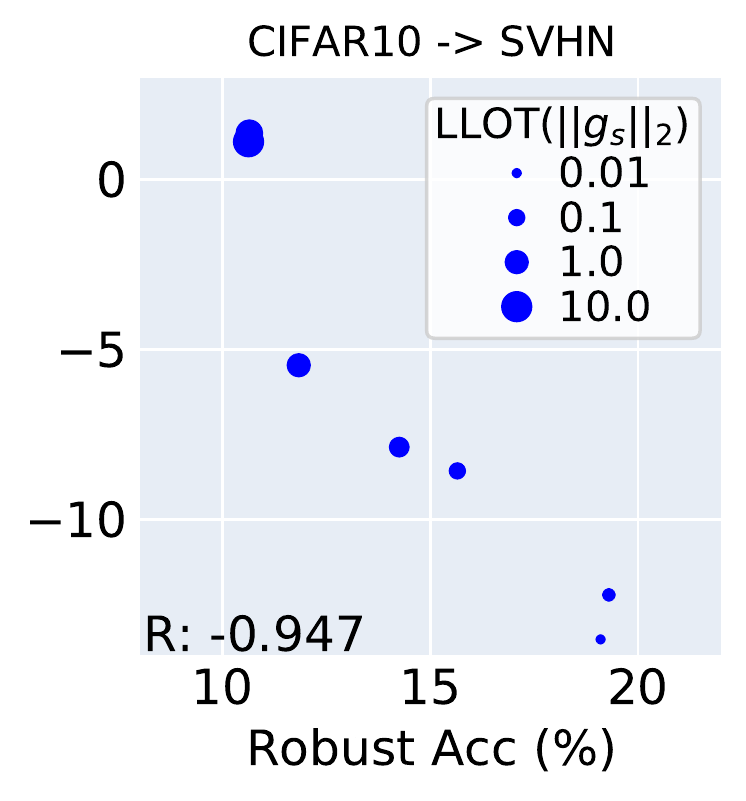}
    \includegraphics[height=1.4in]{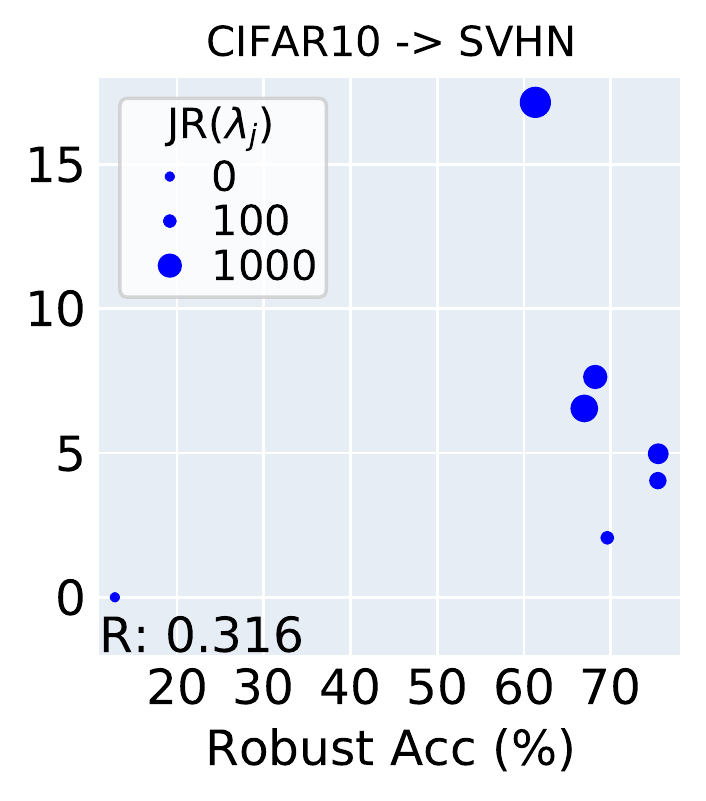}
    \includegraphics[height=1.4in]{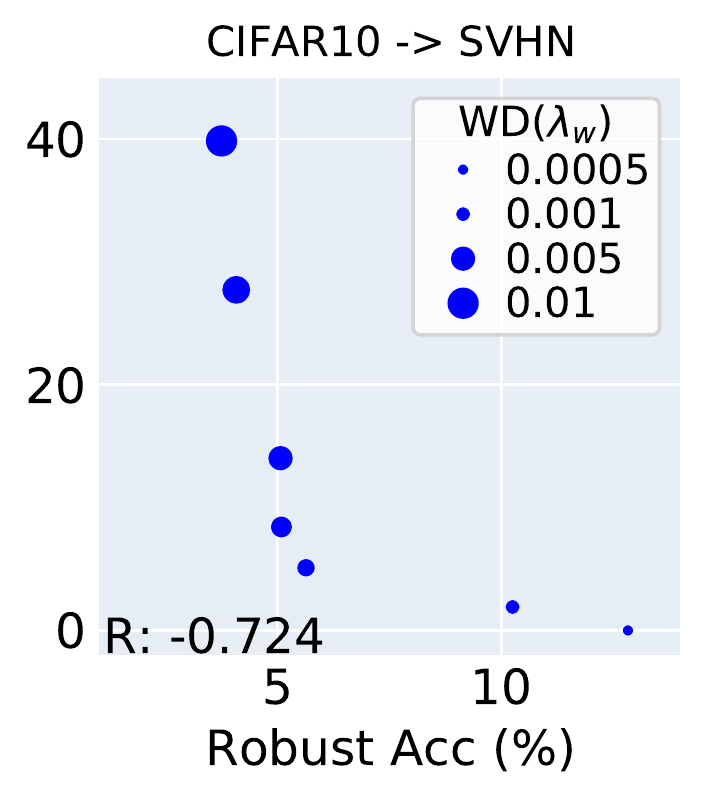}
    
    \includegraphics[height=1.4in]{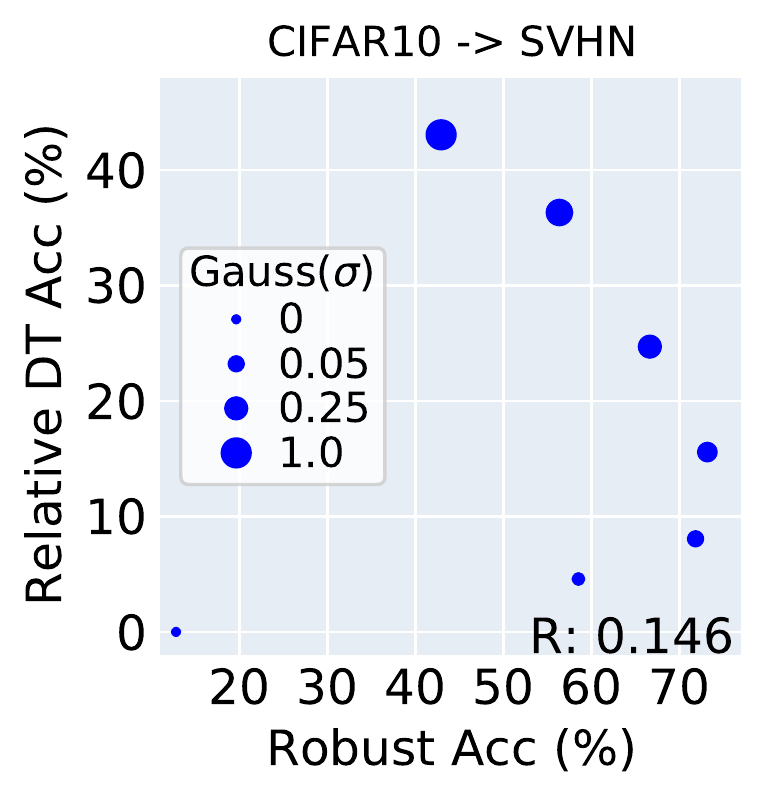}
    \includegraphics[height=1.4in]{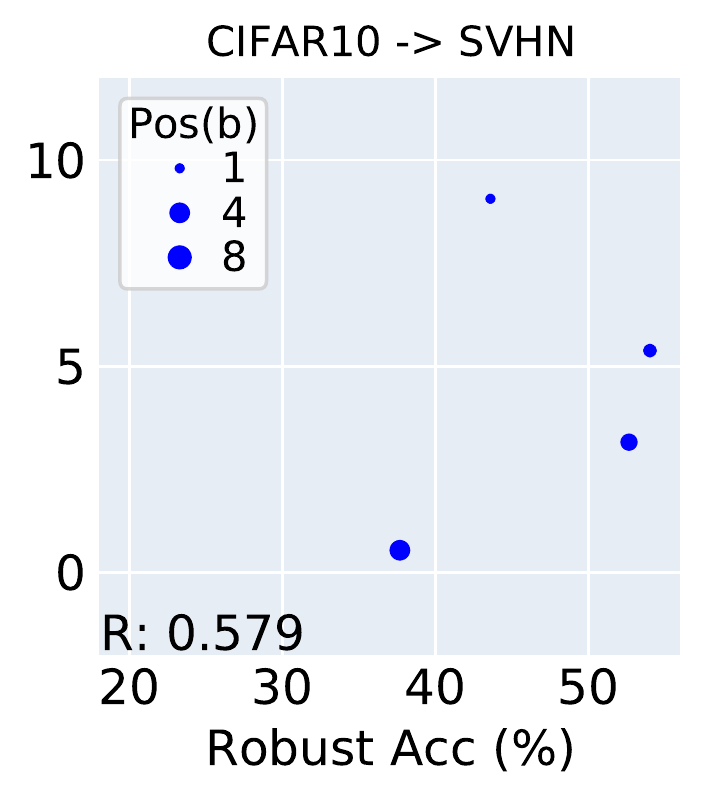}
    \includegraphics[height=1.4in]{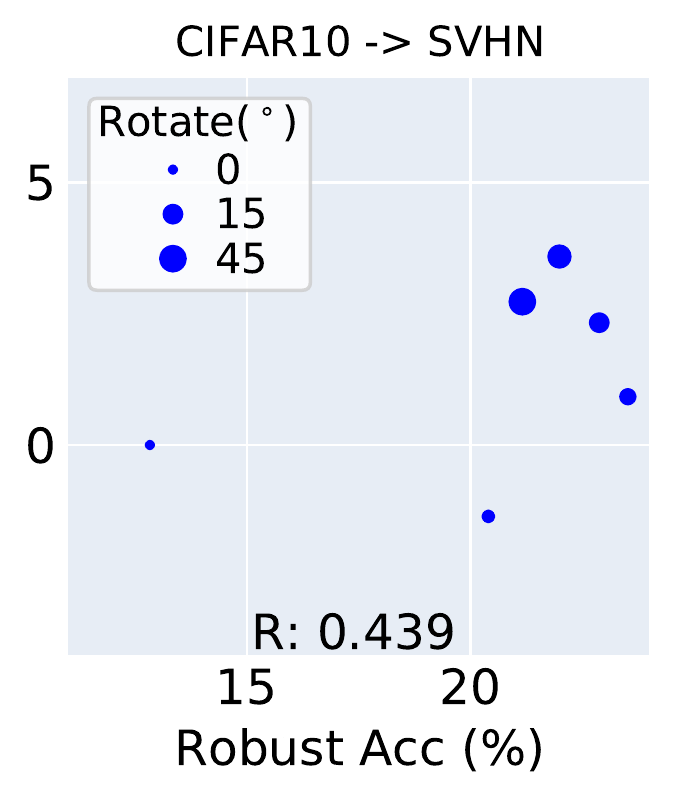}
    \includegraphics[height=1.4in]{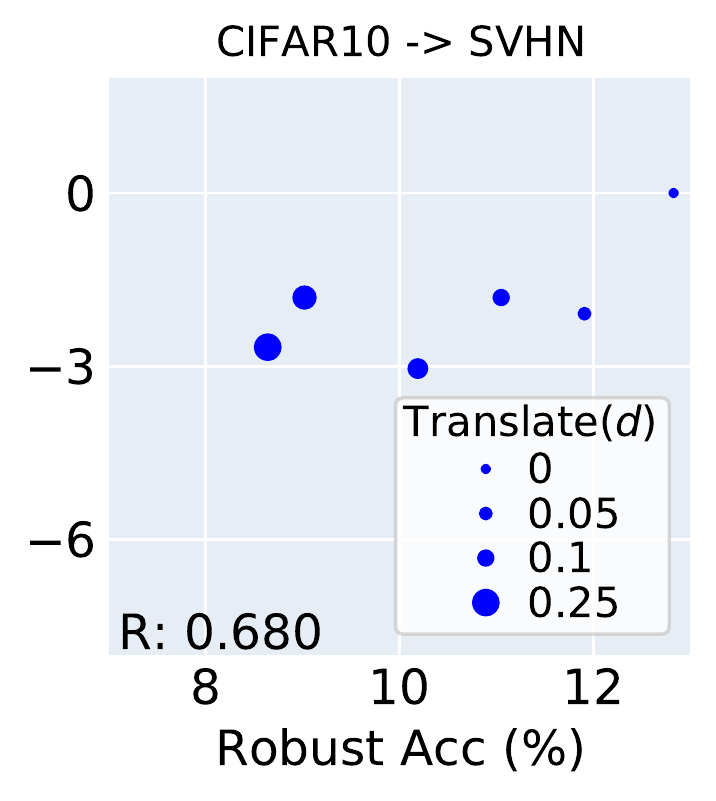}
    \caption{{Relationship between robustness and transferability on CIFAR-10 when we use AutoAttack to evaluate model robustness.}}
    \label{fig:app-cifar-autoatk}
\end{figure}